\pdfoutput=1
\documentclass[10pt]{article}
\usepackage[letterpaper,margin=0.75in]{geometry}
\usepackage[T1]{fontenc}
\usepackage{lmodern}
\usepackage{amsmath,amssymb,amsthm,mathtools,bm}
\usepackage{microtype,graphicx,booktabs,array,enumitem,caption}
\usepackage[round]{natbib}
\usepackage[hidelinks]{hyperref}
\hypersetup{pdftitle={Junction Geometry Survives an Unidentifiable Blur},
  pdfauthor={Sohail (Neel) Sarkar and Shakuntala Baichoo}}
\setlist{itemsep=2pt,topsep=4pt}
\makeatletter
\renewcommand\section{\@startsection{section}{1}{\z@}%
  {-3.2ex \@plus -1ex \@minus -.2ex}{1.4ex \@plus .2ex}%
  {\normalfont\large\bfseries}}
\renewcommand\subsection{\@startsection{subsection}{2}{\z@}%
  {-2.6ex\@plus -1ex \@minus -.2ex}{1ex \@plus .2ex}%
  {\normalfont\normalsize\bfseries}}
\renewcommand\paragraph{\@startsection{paragraph}{4}{\z@}%
  {1.6ex \@plus .5ex \@minus .2ex}{-1em}%
  {\normalfont\normalsize\bfseries}}
\makeatother

\theoremstyle{plain}
\newtheorem{theorem}{Theorem}
\newtheorem{proposition}[theorem]{Proposition}
\newtheorem{lemma}[theorem]{Lemma}

\theoremstyle{remark}

\newcommand{\R}{\mathbb R}
\newcommand{\E}{\mathbb E}
\newcommand{\dd}{\,\mathrm d}
\newcommand{\ph}{\varphi}
\newcommand{\Ccal}{\mathcal C}
\newcommand{\norm}[1]{\left\lVert#1\right\rVert}
\DeclareMathOperator{\KL}{KL}
\DeclareMathOperator{\TV}{TV}
\DeclareMathOperator{\ord}{ord}
\DeclareMathOperator{\clip}{clip}

\title{\bfseries Junction Geometry Survives an Unidentifiable Blur}
\author{Sohail (Neel) Sarkar \qquad Shakuntala Baichoo\\[.4ex]
\normalsize PMCC AI Lab, Peter Munk Cardiac Centre\\
\normalsize University Health Network, Toronto, Ontario, Canada\\[.3ex]
\normalsize\href{mailto:sohail.sarkar@uhn.ca}{\texttt{sohail.sarkar@uhn.ca}} \qquad
\href{mailto:shakuntala.baichoo@uhn.ca}{\texttt{shakuntala.baichoo@uhn.ca}}}
\date{}

\begin{document}
\twocolumn[{%
\begin{@twocolumnfalse}
\maketitle

\begin{center}
\begin{minipage}{0.88\textwidth}
\begin{abstract}
\normalsize
Thin filaments cross in microscopy, and dispersed spectra cross on
X-ray detectors. At a junction, where these structures cross or branch,
how many straight arms meet, and which way do they go?
Unknown blur obscures the answer. At a balanced crossing, distinct
Gaussian blurs can even produce the same observation law. We show that
the arms can still be recovered. For separated rays of nonvanishing
weight and an unknown centered Gaussian-mixture blur, a distance
combining arm count, directions, relative weights and supporting lines
is bounded by a constant times the Euclidean difference in cell
probabilities on a sufficiently fine fixed grid. The bound is uniform
over a bounded model class, including unidentifiable blurs. From $n$
positions or photon counts, the minimax error is of order $n^{-1/2}$.
This rate is proved for a global estimator. In practice we fit rays and
blur by constrained local likelihood, reducing the BIC count when two
detectors agree on fewer arms in the microscopy studies. Across three
studies with protocols fixed before evaluation, on 360 crossings and
branches assembled from measured single-molecule localizations, the fit
gets the arm count and every direction within $10^\circ$ 24 to 37
percentage points more often than a tuned fixed-scale Hessian ridge
detector, averaged over 25 to 400 localizations. On acute angles with correct
counts, mean absolute error is $3.5^\circ$ to $4.3^\circ$, versus
$10.4^\circ$ to $15.5^\circ$. On Chandra spectra, 38 of 90 fits at 800
photons find all four arms within $10^\circ$; among these, median error
is $0.023^\circ$.
\end{abstract}

\noindent\textbf{Keywords:} identifiability; Lipschitz stability; minimax
rates; Gaussian mixtures; ridge detection; single-molecule localization
microscopy; X-ray grating spectroscopy.
\end{minipage}
\end{center}
\end{@twocolumnfalse}
\vspace{1.5em}
}]

\section{Introduction}\label{sec:intro}

Filaments are everywhere in the life and physical sciences:
microtubules and actin in cells, vessels and neurites in tissue, roads
in a city, the dispersed spectra that a grating throws across an X-ray
detector. Where
two filaments cross or one of them branches, the first measurement is
local and geometric. How many straight arms leave the point, and in which
directions? Counts and angles at junctions are the raw material of
network reconstructions, branching statistics and instrument
calibrations, so an error made here travels.

\begin{figure*}[t]
\centering\includegraphics[width=\textwidth]{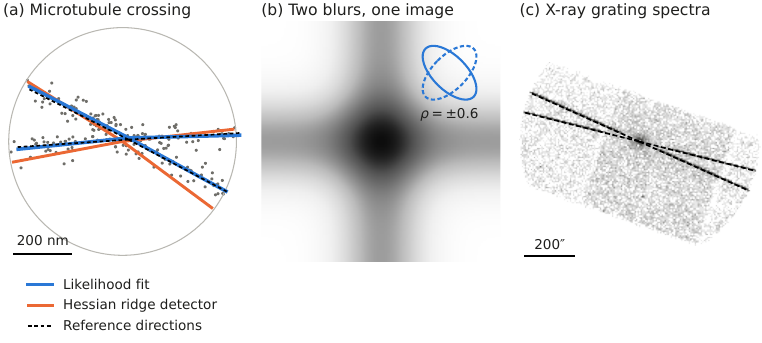}
\caption{Junctions seen through an unknown blur. (a) A crossing
assembled from measured SMLM localizations: 200 positions in a
400\,nm window. Dashed: reference directions from localizations of the
same filaments that the methods do not see. The Hessian ridge detector
finds four arms but opens the two $31^\circ$ angles to $43^\circ$; the
likelihood fit of Section~\ref{sec:method} reads $30^\circ$ and
$33^\circ$. (b) Two full lines blurred with correlation $\rho=-0.6$
(solid ellipse) and $\rho=+0.6$ (dashed) give the same image to machine
precision. (c) 40{,}000 of the 219{,}787 photons in the analysis window
of the High Energy Transmission Grating on the Chandra X-ray Observatory,
for the galaxy NGC~3783. The two dispersed spectra cross at the source image and
form a four-arm junction whose directions (dashed) are known from the
instrument calibration.}
\label{fig:hero}
\end{figure*}

Many modern instruments record junctions as clouds of points.
Single-molecule localization microscopy (SMLM) lights up a few
fluorophores at a time and records one position per detected emission; a
fluorophore that blinks contributes several positions
\citep{betzig06,rust06}. A photon-counting X-ray detector records one
position per photon, and a GPS receiver one position per fix. Every
recorded position is blurred. In SMLM the blur
mixes localization errors that differ from molecule to molecule; in a
grating spectrometer it is the line-spread function of the instrument,
which has wings; GPS error has heavy tails. None of these blurs is known
in advance, and none is a single Gaussian.

There are two standard ways around the blur. The first smooths the
points into an image and reads the arms off its ridges, as Hessian-based
ridge and vessel filters do \citep{lindeberg98,frangi98,steger98}.
Smoothing adds its own blur, and at a fixed scale it opens acute angles
(Proposition~\ref{prop:hessian}) and merges close arms. The second route
fits the arms and the blur together, and it meets a sharper obstacle:
the blur can be exactly unidentifiable. Put equal weight on the four
half-axes of the plane and blur them with the Gaussian covariance
\begin{equation}\label{eq:example}
 \Sigma_\rho=\begin{pmatrix}1&\rho\\ \rho&1\end{pmatrix},
 \qquad -1<\rho<1 .
\end{equation}
For every $\rho$ the blurred intensity is $\{\ph(x)+\ph(y)\}/4$, where
$\ph$ is the standard normal density (Figure~\ref{fig:hero}b). No amount
of data determines $\rho$. Rotating an arm, or moving weight from one arm
to another, does change the intensity. A fit that estimates the blur
must therefore work near a family of models in which part of its
parameter is lost, and near such models mixture parameters converge
slowly \citep{ho16} and are only weakly identified \citep{kaji21}.

Here they do not. The arms depend on the data through a Lipschitz map
with one constant for the whole model class, including every model whose
blur is lost, and they are estimated at the parametric rate $n^{-1/2}$
exactly where the blur cannot be estimated at all. The reason is
geometric and short. Far along an arm the data see the blur only through
its profile across the arm; the mass of that profile is the weight of the
arm and its mean is the line of the arm, whatever the blur. A change of
direction slides the profile by an amount that grows with the distance
from the vertex, and no integrable profile can absorb such a slide. We
turn this observation into a uniform inequality with a general tool that
compares the orders at which a target and an observation change along
analytic curves (Proposition~\ref{prop:ordercriterion}), and the same
tool applies to other targets observed through a nuisance that cannot be
identified.

\paragraph{Contributions.}
\begin{itemize}[leftmargin=*]
\item \emph{The blur the data cannot see.} A covariance change $H$ of a
Gaussian blur leaves a source $\mu$ invisible exactly when
$H:\nabla^2\mu=0$ (Lemma~\ref{lem:invisible}). For junctions this happens
in two cases only: one balanced line, and two balanced lines with
$H\propto ab^\top+ba^\top$ (Proposition~\ref{prop:classification}). At
91 microtubule crossings in measured camera images, a blur change along
$ab^\top+ba^\top$ alters the image eight to nine times less than an
ordinary change of the same size (Section~\ref{sec:invisible}).
\item \emph{Stability through the ambiguity.} For rays blurred by an
unknown centered finite Gaussian mixture, the arm count, directions,
relative weights and supporting lines are Lipschitz in the spatial score
and in the cell probabilities of one fixed partition, uniformly over the
class (Theorem~\ref{thm:inverse}). No identifiability margin of the blur
enters.
\item \emph{Rates.} The minimax rate from $n$ points or counts is
$n^{-1/2}$ (Theorem~\ref{thm:statistics}), and it holds at balanced
crossings, where the covariance risk stays bounded away from zero
(Proposition~\ref{prop:crossing}). An explicit four-cell estimator shows
both rates in closed form (Theorem~\ref{thm:fourcells}). Counts are
selected consistently by a penalized likelihood, but some admissible
pairs of counts need more than $10^8$ points to be told apart
(Section~\ref{sec:counting}).
\item \emph{Beyond straight rays.} The guarantees extend to an unknown
background, to misfit and contamination with a sharp
$n^{-1/2}+\varepsilon$ rate, and to curved arms, whose tangents are
recovered up to a bias of order $M\sigma^\beta$ set by the curvature at
the scale $\sigma$ of the blur; at a fixed blur no number of points
removes this bias (Section~\ref{sec:beyond}). A fixed-scale Hessian
ridge detector widens acute angles for a reason we compute exactly
(Proposition~\ref{prop:hessian}).
\item \emph{Evidence.} In three prospective studies with frozen
protocols, on 360 junctions assembled from measured localizations in 77
fields, a constrained likelihood fit recovers the junction exactly 24.1 to
37.1 points more often than a tuned Hessian ridge detector and 27.5 to
38.4 points more often than the same detector with automatic scale
selection,
without its acute-angle bias. On natural microtubule junctions, the
choice between the two methods moves the reported share of junctions
with an angle below $60^\circ$ from 44.5\% to 35.7\%. On Chandra X-ray
grating photons, where the directions are fixed by calibration, the
error of fits that find all four arms falls as $n^{-1/2}$. At 96 road
intersections traced by taxis, the fit beats the tuned detector by 14.2
points against map geometry, with error falling as $n^{-1/2}$
(Section~\ref{sec:experiments}).
\end{itemize}

\paragraph{Outline.}
Section~\ref{sec:model} sets up the model and the target.
Sections~\ref{sec:invisible}--\ref{sec:beyond} develop the theory, from
the ambiguity to the rates and the extensions. Section~\ref{sec:method}
turns it into an estimator, and Section~\ref{sec:experiments} tests that
estimator on simulated, assembled, natural, astronomical and road
junctions.
Proofs are in Appendix~\ref{app:proofs}, experimental details in
Appendix~\ref{app:design} and every prespecified outcome in
Appendix~\ref{app:outcomes}.

\section{Model and target}\label{sec:model}

Fix a dimension $D\ge2$. A junction with $K$ arms has a vertex
$p\in\R^D$, distinct unit directions $a_1,\dots,a_K$ and weights $w_j>0$
with $\sum_jw_j=1$. Write $\vartheta=(K,p,(a_j,w_j)_{j\le K})$. The
junction emits uniformly along the rays $\{p+ra_j:r\ge0\}$:
\begin{equation}\label{eq:model}
 \mu_\vartheta=\sum_{j=1}^K w_j\int_0^\infty\delta_{p+ra_j}\dd r ,
\end{equation}
so $w_j$ is the intensity per unit length of arm $j$ relative to the
others. The instrument blurs $\mu_\vartheta$ with an unknown kernel
\begin{equation}\label{eq:mixture}
 \begin{aligned}
 k_\eta(x)&=\sum_{\ell=1}^{J}v_\ell\,\ph_{\Sigma_\ell}(x-b_\ell),\\
 \sum_\ell v_\ell&=1,\qquad \sum_\ell v_\ell b_\ell=0,
 \end{aligned}
\end{equation}
a centered mixture of at most $J$ Gaussians with $v_\ell\ge0$,
$|b_\ell|\le B_0$ and $\lambda I\preceq\Sigma_\ell\preceq\Lambda I$; here
$\ph_\Sigma$ is the centered normal density with covariance $\Sigma$.
Components may have weight zero or coincide, so neither the number of
components nor their labels are part of what can be learned. The
intensity is $q_{\vartheta,\eta}=k_\eta*\mu_\vartheta$. In SMLM the
mixture describes localization errors whose size varies across
molecules; in grating spectra it describes a line-spread function with
wings. The rays are infinite; Section~\ref{sec:beyond} treats finite and
curved arms.

\paragraph{Observations.}
We observe the window\hfill\break $U=[-R_0,R_0]^D$. Points form a Poisson process
with intensity $Aq$ on $U$ and unknown exposure $A>0$. Given $N=n$
points, the positions are independent with density
\begin{equation}\label{eq:observation}
 \pi=q/Z,\qquad Z=\int_Uq(x)\dd x,
\end{equation}
and the spatial score is $g=\nabla_x\log q=\nabla_x\log\pi$. Because $A$
is free, the total $N$ carries no information about the shape, and
conditioning on it loses nothing (Appendix~\ref{app:minimax}). A
partition of $U$ into cells $U_1,\dots,U_M$ gives the cell probabilities
$P_i=\int_{U_i}\pi$, and the counts of a detector with $M$ pixels are
$\mathrm{Multinomial}(n,P)$.

\paragraph{Target.}
The blur is a nuisance; the target is the geometry. Let
$u_j=(I-a_ja_j^\top)p$ be the offset from the origin of the line that
carries arm $j$, fix a length $L>0$, and put
$s_j=(a_j,w_j,u_j/L)\in\R^{2D+1}$. The signature of the junction is
$S(\vartheta)=\{s_1,\dots,s_K\}$. For models with equal counts,
\[
 d_T(\vartheta,\vartheta')=\min\Bigl\{1,\ \min_\sigma
 \Bigl(\textstyle\sum_{j}\bigl|s_j-s'_{\sigma(j)}\bigr|^2\Bigr)^{1/2}\Bigr\},
\]
the inner minimum taken over permutations $\sigma$; for unequal counts
$d_T=1$. Two nonparallel supporting lines fix the vertex. Along a single
balanced line, a pair of opposite arms with equal weights, the vertex
has no observable position, which is why the target records offsets
rather than the vertex.

\paragraph{Assumptions.}
(A1) $1\le K\le K_{\max}$, $w_j\ge w_*>0$, $|a_i-a_j|\ge\delta>0$ for
$i\ne j$, and $|p|\le\bar p$. (A2) The blur has the form
\eqref{eq:mixture} with fixed $J$, $B_0$ and $0<\lambda\le\Lambda$.
(A3) The observations are positions in the fixed window $U$, or their
counts on a fixed partition of $U$. These bounds define a compact
semialgebraic class $\Ccal$ of pairs $(\vartheta,\eta)$, the union of its
strata $K=1,\dots,K_{\max}$. The floors in (A1) restrict the arms only;
blur components may vanish, merge or coincide.

\section{The blur the data cannot see}\label{sec:invisible}

Blurring a source $\mu$ with $\ph_\Sigma$ and with $\ph_{\Sigma+H}$ gives
two intensities. When are they equal? The answer depends on the source
only through one derivative.

\begin{lemma}[Invisible covariance changes]\label{lem:invisible}
Let $\mu$ be a positive measure on $\R^D$ with $\mu(\{|x|\le r\})$
bounded by a polynomial in $r$, and let $\Sigma$ and $\Sigma+H$ be
positive definite. Then $\ph_\Sigma*\mu=\ph_{\Sigma+H}*\mu$ if and only
if $H:\nabla^2\mu=\sum_{i,k}H_{ik}\partial_i\partial_k\mu=0$ as a
distribution.
\end{lemma}

The condition does not involve $\Sigma$, and it is linear in $H$. For a
single straight line along $a$, $\nabla^2\mu$ lives on the normal space
of the line, so every $H$ whose normal block vanishes is invisible. A
junction combines several lines and their ends, and the ends are what
make most changes visible.

\begin{proposition}[Gaussian blur ambiguity at junctions]\label{prop:classification}
Let $J=1$ and $b_1=0$. If two models have equal spatial scores on a
nonempty open subset of $U$, their signatures are equal. Their
covariances can differ only if (i) the arms form one balanced line with
direction $a$ and $\Sigma'-\Sigma$ vanishes on $a^\perp\times a^\perp$,
or (ii) the arms form two nonparallel balanced lines with directions $a$
and $b$ and $\Sigma'-\Sigma=t(ab^\top+ba^\top)$ for some $t\in\R$. The
vertices agree, except for a shift along a single balanced line.
Conversely, each such change with $\Sigma'\succ0$ leaves the score
unchanged.
\end{proposition}

The example \eqref{eq:example} is case (ii) with $a=e_1$, $b=e_2$. The two
lines need not be orthogonal or equally bright; each only has to carry
equal weight on its two arms. A half-line, an unbalanced pair or a third
direction breaks the ambiguity, because the end of an arm, or a third
normal space, sees every remaining direction of $H$.

\paragraph{The invisible direction in camera images.}
The proposition predicts something that can be tested on an instrument
without fitting a blur. Smooth a measured image of a crossing twice, with
two known Gaussian kernels whose covariances differ by $\pm h\,H$. For two
balanced straight lines with directions $a,b$ and
$H\propto ab^\top+ba^\top$, the two smoothed images agree for every
point-spread function of the instrument, because the difference of the
two kernels annihilates the source before any blur acts on it. A
generic change of the same size does not. We ran this test on the
DL-SMLM collection \citep{zhao25}, which pairs 100-frame
diffraction-limited image stacks of labeled microtubules with
super-resolution reconstructions. At 91 four-arm crossings in 21 cells,
whose opposite arms lie within $15^\circ$ of a straight line according to
the super-resolution reference, we measured how much the image changes
under the predicted invisible perturbation and under two controls of equal
Frobenius size: the same perturbation turned by $45^\circ$, and an
isotropic one. The change is measured as the energy of the difference
image near the crossing, estimated from two independent frame averages so
that the energy of zero-mean noise cancels, and scaled by the local
contrast (Appendix~\ref{app:camera}). On analytic images of balanced
crossings the test returns below $2\times10^{-9}$, and on unbalanced ones
$0.008$ to $0.021$ times the isotropic response. On the measured images
(Figure~\ref{fig:camera}) the mean response is $0.0023$ along
$ab^\top+ba^\top$, against $0.0188$ for the turned perturbation and
$0.0205$ for the isotropic one, and the invisible direction gives the
smallest response at every one of the 91 crossings. The isotropic
comparison was fixed in advance; the turned comparison was added after
the first result. Measured crossings are finite, slightly unbalanced and
slightly curved, so the response is small, not zero. The camera, in
other words, barely sees the blur direction that the theory says no
instrument can see.

\begin{figure}[t]
\centering\includegraphics[width=\columnwidth]{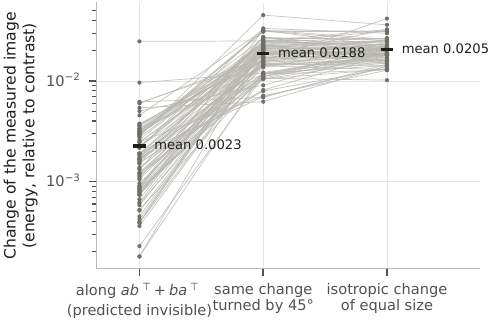}
\caption{How much a measured image of a microtubule crossing changes
when its blur covariance is changed by a fixed amount, at 91 crossings
in 21 cells of the DL-SMLM collection. Gray lines join the three
responses of one crossing; black bars are means. The change along
$ab^\top+ba^\top$, where $a,b$ are the directions of the two filaments,
is the one Proposition~\ref{prop:classification} calls invisible.}
\label{fig:camera}
\end{figure}

\paragraph{Why a class of blurs is needed.}
Some restriction on the blur is unavoidable. Let $\nu_H$ and $\nu_V$ be
centered uniform measures on equal horizontal and vertical segments and
$\ph$ an isotropic Gaussian density. The smooth positive centered kernels
$k_H=\ph*\nu_V$ and $k_V=\ph*\nu_H$ satisfy $k_H*\nu_H=k_V*\nu_V$ by
commutativity of convolution: under an unrestricted unknown kernel, a
horizontal segment and a vertical one produce the same image. Mixtures of
Gaussians with eigenvalues in $[\lambda,\Lambda]$ exclude such kernels
while still describing blurs with several widths and heavy shoulders.

\section{Stability through the ambiguity}\label{sec:stability}

\begin{theorem}[Uniform geometric stability]\label{thm:inverse}
There are constants $C_g,C_{\rm pix}<\infty$ and a finite partition of
$U$ into equal cubes such that, for all $(\vartheta,\eta)$ and
$(\vartheta',\eta')$ in $\Ccal$,
\begin{align}
 d_T(\vartheta,\vartheta')&\le
 C_g\norm{g_{\vartheta,\eta}-g_{\vartheta',\eta'}}_{L^2(U)},
 \label{eq:inverse}\\
 d_T(\vartheta,\vartheta')&\le
 C_{\rm pix}\norm{P_{\vartheta,\eta}-P_{\vartheta',\eta'}}_2.
 \label{eq:bininverse}
\end{align}
The constants depend on the bounds in (A1)--(A3), not on any
identifiability margin of the blur. Mixture labels, component parameters
and the number of active components need not be identifiable.
\end{theorem}

Identifiability alone would not give this. The pair $F(t)=t^3$,
$T(t)=t$ is identified, yet no Lipschitz inverse exists at $t=0$. Near a
balanced crossing the full parameter behaves in this way, because the
covariance is weakly identified. What rules the failure out for the
target is a comparison of orders along curves, which the next
proposition turns into a uniform bound. It is stated for any analytic
model.

\begin{proposition}[Analytic order criterion]\label{prop:ordercriterion}
Let $\mathcal Z\subset\R^m$ be compact and semianalytic, and let $F$,
with values in a real Hilbert space, and $T$, with values in $\R^k$,
extend real-analytically to a neighborhood of $\mathcal Z$. Suppose that
$F(z)=F(z')$ implies $T(z)=T(z')$. If every analytic arc
$t\mapsto(z(t),z'(t))$ in $\mathcal Z^2$ that starts at a pair with equal
observations and has $T(z(t))\ne T(z'(t))$ for $t>0$ satisfies
$\ord_t\{T(z(t))-T(z'(t))\}\ge\ord_t\{F(z(t))-F(z'(t))\}$, then
$\norm{T(z)-T(z')}\le C\norm{F(z)-F(z')}$ on $\mathcal Z^2$.
\end{proposition}

The proof is short: a failing sequence and real-analytic curve selection
\citep{bierstone88} produce an arc on which the observation changes at a
strictly higher order than the target (Appendix~\ref{app:order}). The
criterion knows nothing about rays. It reduces a uniform inequality over a
compact family, with all its singular points, to a statement about
Taylor coefficients along one curve at a time. For junctions the work is
to verify that statement, and the far field does it.

\begin{figure}[t]
\centering\includegraphics[width=\columnwidth]{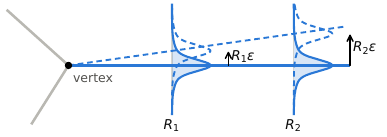}
\caption{Far along an arm the data show the blur only through its
profile across the arm (shaded). The mass of the profile is the weight of
the arm and its mean is the line of the arm, whatever the blur mixture.
Turning the arm by a small angle $\varepsilon$ (dashed) moves the profile
by about $R\varepsilon$ at distance $R$: a shift that grows without
bound, which no integrable profile can absorb.}
\label{fig:profile}
\end{figure}

\paragraph{Why the far field settles the orders.}
Take an arm $a$ and an orthonormal basis $E$ of $a^\perp$
(Figure~\ref{fig:profile}). Far along the arm, at $x=p+Ra+Ey$, the
intensity converges as $R\to\infty$ to
\begin{equation}\label{eq:mixtureprofile}
 w\,f(y),\qquad
 f(y)=\sum_\ell v_\ell\,\ph_{E^\top\Sigma_\ell E}(y-E^\top b_\ell),
\end{equation}
while every other arm decays like $e^{-cR^2}$. The profile $f$ is the
blur seen across the arm. It has mass one, and mean zero because the
mixture is centered. Its mass returns the weight and its mean the
supporting line, whatever the mixture. Directions carry the argument.
Suppose two analytic curves of models have scores that agree to order
$m-1$ in the curve parameter $t$, while their directions first differ at
order $r<m$, by $t^rh$ with $h\perp a$. At distance $R$ the far profile
moves by $Rt^rh$, so the order-$r$ coefficient contains
$R\,h^\top\nabla f$, which grows with $R$. It must vanish for every $R$,
which would make the integrable density $f$ constant along $h$. Hence
$h=0$: the target never changes before the observation.
Proposition~\ref{prop:ordercriterion} then gives \eqref{eq:inverse}, and
a derivative bound on the analytic family followed by cell averaging
gives \eqref{eq:bininverse}.

The floors $w_*$ and $\delta$ cannot be dropped: an arm of weight $t$, or
two arms at angles $\pm t$, change the law by $o(1)$ as $t\to0$ while
$d_T=1$. They restrict the arms only. Blur components may vanish or merge
freely, and the argument never divides by their weights.

\section{Rates and the price of counting}\label{sec:rates}

\subsection{The parametric rate, at the singular set}
\begin{theorem}[Minimax rate]\label{thm:statistics}
For the partition of Theorem~\ref{thm:inverse}, a measurable
minimum-distance estimator $\widehat S$ from $n$ counts satisfies
\begin{equation}\label{eq:risk}
 \sup_{(\vartheta,\eta)\in\Ccal}\E\,d_T(\widehat S,S(\vartheta))
 \le C_{\rm pix}\bigl(2n^{-1/2}+n^{-1}\bigr).
\end{equation}
There is $c>0$ such that every estimator, even one given the $n$ exact
positions, has worst-case expected loss at least $cn^{-1/2}$ for all
large $n$. The minimax rate is $\Theta(n^{-1/2})$.
\end{theorem}

The upper bound combines \eqref{eq:bininverse} with
$\E\norm{\widehat P-P}_2^2\le1/n$ for empirical frequencies; a single
rotating ray gives the lower bound (Appendix~\ref{app:minimax}). The
theorem fixes the order in $n$ for the whole class. Its constant is
existential, and Section~\ref{sec:counting} exhibits admissible pairs
that need more than $10^8$ points to be told apart, so the theorem is a
structural statement about the model class rather than a guarantee for a
window of 25 localizations. It holds on the singular set itself, where the
covariance is lost.

\begin{proposition}[Sharpness at a balanced crossing]\label{prop:crossing}
Take a balanced crossing in $\Ccal$ with nonparallel lines $a,b$,
weights $(\alpha/2,\alpha/2,(1-\alpha)/2,(1-\alpha)/2)$, $0<\alpha<1$,
vertex $p$, and one Gaussian with $\lambda I\prec\Sigma\prec\Lambda I$;
put $H=ab^\top+ba^\top$. There is $h_0>0$ such that the models with
covariances $\Sigma\pm h_0H$ have the same law, so every covariance
estimator has maximum risk at least $h_0\norm H_{\rm F}$ for every $n$.
Rotating the crossing and its covariance in the plane of $a,b$ gives
alternatives at geometric distance of order $n^{-1/2}$ whose $n$-point
laws stay within total variation $1/4$, so the signature risk is at
least $3\lambda/\{8\pi(\sqrt DR_0+|p|)^2\sqrt n\}$ for large $n$.
\end{proposition}

At a balanced crossing the covariance cannot be estimated at all, and
the arms are estimated at exactly the rate $n^{-1/2}$. Finite-mixture
theory would lead one to expect the opposite ordering near such a
point: estimation slows where the parameter degenerates
\citep{ho16}. The target escapes because the degeneracy lies entirely in
the blur.

\subsection{An explicit example: four counts}\label{sec:fourcells}
The contrast between the two rates can be seen in closed form. Rays
along $e_1,-e_1,e_2,-e_2$ carry weights
$(1/4+\tau/2,\ 1/4-\tau/2,\ 1/4,\ 1/4)$ with $0\le\tau\le\bar\tau<1/2$,
under one Gaussian with unit variances and unknown correlation
$|\rho|\le r\le0.8$. The vertex, the axes and the variances are known,
and the only data are the photon counts in the four quadrants of
$[-3,3]^2$. At $\tau=0$ this is the example \eqref{eq:example}: the
correlation is invisible. The mean of the quadrant signs $(s_x,s_y)$ is
$\tau\,(a(\rho),b(\rho))$ for two explicit functions $a$ and $b$; its
angle identifies $\rho$ and its length identifies $\tau$, and inverting
the empirical sign mean gives estimates $\widehat\tau$ and
$\widehat\rho$ in closed form (Appendix~\ref{app:fourcells}).

\begin{theorem}[Four quadrants, two rates]\label{thm:fourcells}
For $n\ge1$ photons,
\begin{equation}\label{eq:fourrisk}
 \E|\widehat\tau-\tau|\le1.4\sqrt{2/n},\qquad
 \E\,d_T(\widehat S,S)\le1.4/\sqrt n,
\end{equation}
\begin{equation}\label{eq:rhorisk}
 \E|\widehat\rho-\rho|\le\min\bigl\{2r,\ 4.1\sqrt2/(\tau\sqrt n)\bigr\},
 \qquad\tau>0.
\end{equation}
The minimax rates are $n^{-1/2}$ for the imbalance and\hfill\break
$\min\{1,(a\sqrt n)^{-1}\}$ for the correlation on $a\le\tau\le\bar\tau$,
uniformly in $a\in(0,\bar\tau]$. At $\tau=0$ the minimax absolute
correlation risk is exactly $r$ for every $n$.
\end{theorem}

Figure~\ref{fig:fourcells} shows both rates in 30{,}000 simulated count
vectors (six imbalances, five correlations, 200 replicates each, from
$10^2$ to $10^6$ photons). The imbalance error falls by a factor of 9.54
$[9.22,9.88]$ from $10^4$ to $10^6$ photons, close to the factor 10 that
$n^{-1/2}$ predicts, and it falls as fast at $\tau=0$ as anywhere else.
The correlation error does not move at $\tau\le10^{-4}$, where it stays
between 0.64 and 0.69 at every $n$. In a second set of 96{,}000 draws
with $\tau=h/\sqrt n$, the correlation error depends on $n$ and $\tau$
only through the local signal $h$: it is 0.67 to 0.69 at $h=0$ and 0.42
to 0.43 at $h=4$, from $10^3$ to $10^6$ photons. The 95\% confidence set
of Appendix~\ref{app:fourcells} gave the full correlation interval in
all 96{,}000 of these draws, as it should, while at $\tau=0.1$ its width
fell to 0.297 by $10^6$ photons.

\begin{figure*}[t]
\centering\includegraphics[width=\textwidth]{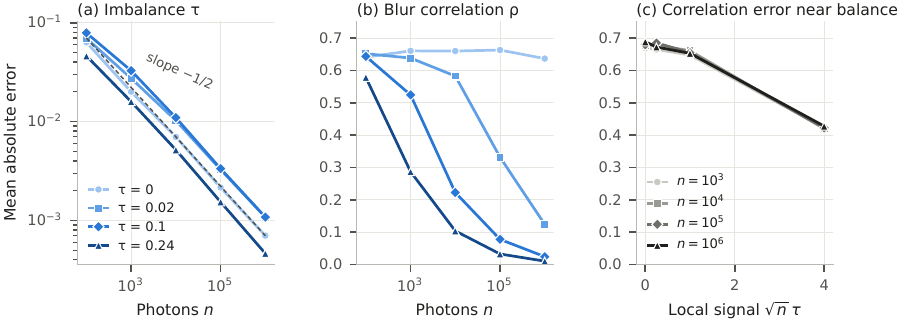}
\caption{Four quadrant counts of a calibrated crossing. (a) Mean
absolute error of the imbalance $\tau$; dashed: slope $-1/2$. (b) Mean
absolute error of the blur correlation $\rho$ at the same imbalances.
(c) Correlation error when $\tau=h/\sqrt n$, against the local signal
$h=\sqrt n\tau$: the curves for four sample sizes lie on top of each
other. Each point averages the specified correlations (five in a and b,
three in c) over 200 (a, b) or 2{,}000 (c) independent count vectors
per setting.}
\label{fig:fourcells}
\end{figure*}

\subsection{Counting arms}\label{sec:counting}
Theorem~\ref{thm:inverse} covers the count through its strata. Let
$\mathcal P_K$ be the image of the $K$-arm stratum in cell probabilities,
$a_*=\min_{P,i}P_i>0$ over all images, and
$\Delta=\min_{K\ne K'}\mathrm{dist}(\mathcal P_K,\mathcal P_{K'})$, which
is at least $C_{\rm pix}^{-1}$ when two strata are nonempty.

\begin{proposition}[Penalized count selection]\label{prop:count}
Let $\widehat K$ minimize, to within $\epsilon_n$ of the infimum, the
per-count negative log-likelihood $-\sum_i\widehat P_i\log P_i$ plus a
penalty $b_{K,n}$, over the strata. If
$\max_Kb_{K,n}-\min_Kb_{K,n}+\epsilon_n<\Delta^2/4$, then uniformly over
$\Ccal$,
\[
 \Pr(\widehat K\ne K)\le
 \exp\bigl[-n\{a_*\Delta/4-n^{-1/2}\}_+^2\bigr],
\]
and $\E\,d_T\le C_{\rm pix}\{2/(a_*\sqrt n)+\sqrt{2\epsilon_n}\}
+\Pr(\widehat K\ne K)$. Penalties $d_K\log n/(2n)$ and
$\epsilon_n=O(1/n)$ keep the rate $n^{-1/2}$.
\end{proposition}

No Fisher-information condition enters, so the singular points of the
blur do not disturb the count (Appendix~\ref{app:selection}). The margin
$\Delta$ can still be tiny. A planar two-arm model and a three-arm model,
with arm weights at least $0.04$ and one centered Gaussian each, satisfy
$\chi^2(\pi_2\Vert\pi_3)<1.556851\times10^{-8}$ on $[-3,3]^2$, with the
integration error included (Appendix~\ref{app:resolution}). From $10^6$
exact positions, every estimator therefore names the wrong count with
probability above $0.46868$ under one of the two models, and error below
$0.05$ under both needs more than 104 million positions.

A second limit concerns faint arms. Take an admissible three-arm density
$p_3$ and the window response $r_4$ of a fourth arm with the same vertex
and blur. The law $(1-\varepsilon)p_3+\varepsilon r_4$ is a three-arm
junction with an $\varepsilon$ share of structured clutter, and it is
also a four-arm junction whose fourth arm holds an $\varepsilon$ share of
the window. No observation confined to the window separates the two
descriptions at any sample size. Points beyond the window can, when a
true arm continues; Section~\ref{sec:faint} tests exactly this.

\section{Background, misfit and curved arms}\label{sec:beyond}

\subsection{An unknown background and the shape of the window}
Measured windows contain points that belong to no arm. Replace the
intensity by $q^\kappa_{\vartheta,\eta}=\kappa+q_{\vartheta,\eta}$ with an
unknown constant $\kappa\in[0,\kappa_{\max}]$, observed through
$\pi^\kappa=q^\kappa/(\kappa|U|+Z)$.

\begin{proposition}[Unknown constant background]\label{prop:background}
Theorems~\ref{thm:inverse} and~\ref{thm:statistics} and
Proposition~\ref{prop:count} hold on $\Ccal\times[0,\kappa_{\max}]$, with
the score $\nabla\log q^\kappa$, the cell law of $\pi^\kappa$ and new
finite constants. The background level $\kappa$ is identified.
\end{proposition}

The background is removed at spatial infinity along a direction that
carries no arm, and the rest of the argument runs unchanged
(Appendix~\ref{app:background}). A background that varies in space, even
an affine one, needs further positivity and growth arguments and is not
covered. The score inverse uses only a bounded window with nonempty
interior, so it holds on the disk that the fit of
Section~\ref{sec:method} observes; the cell inverse carries over to a
fixed grid of squares clipped by the disk, with the smallest clipped cell
in place of the cube volume (Appendix~\ref{app:background}).

\subsection{Misfit and contamination}
Measured junctions are not exactly rays in a Gaussian mixture. The inverse
turns every modeling error into a geometric error of the same order.

\begin{proposition}[Misfit]\label{prop:mismatch}
Let the true cell law $P_*$ satisfy $\norm{P_*-P_0}_2\le b$ for a model in
$\Ccal$ with signature $S_0$. A fit whose residual
$\norm{P_{\widehat\xi}-\widehat P}_2$ is within $\varepsilon_{\rm opt}$
of its minimum over $\Ccal$ satisfies
\begin{equation}\label{eq:mismatch}
 d_T(\widehat S,S_0)\le C_{\rm pix}\bigl\{2\norm{\widehat P-P_*}_2+2b
 +\varepsilon_{\rm opt}\bigr\}.
\end{equation}
Under contamination, $P_*=(1-\varepsilon)P_0+\varepsilon Q$ with arbitrary
$Q$, the minimax risk is of order $n^{-1/2}+\varepsilon$, and both terms
are necessary.
\end{proposition}

A kernel within $L^1$ distance $\eta$ of the class contributes
$b\le2\operatorname{diam}(U)\eta/Z$, and arms cut at length $\ell$
contribute a Gaussian tail in $\ell$ (Appendix~\ref{app:boundaries}). A
solver's termination flag does not bound $\varepsilon_{\rm opt}$, and
Section~\ref{sec:method} measures that term instead of assuming it away.

\subsection{Curved arms}\label{sec:curved}
Filaments bend. Suppose arm $j$ follows an arc-length curve
$\gamma_j:[0,\ell]\to\R^D$ with $\gamma_j(0)=p$ and $\gamma_j'(0)=a_j$,
emits with density $h_j(r)$ along its length, and deviates from its
tangent ray at a H\"older rate:
\begin{equation}\label{eq:holder}
 \begin{aligned}
 |\gamma_j'(r)-a_j|&\le\varepsilon r^\beta,\\
 |h_j(r)-1|&\le\varepsilon r^\beta,\\
 |\gamma_j(r)-p|&\ge c_0r,\qquad 0\le r\le\ell,
 \end{aligned}
\end{equation}
with $0<\beta\le1$ and $0<c_0\le1$. Lengths are measured in the units of
the blur. If the physical curve satisfies the first bound with constant
$M$ and the blur has scale $\sigma$, then $\varepsilon=M\sigma^\beta$:
the relevant curvature is the curvature seen across one blur width.

\begin{proposition}[Curved arms and distant structure]\label{prop:curvedarms}
Let $\vartheta$ be a junction in $\Ccal$ with arms as in \eqref{eq:holder},
and let $\nu$ be a positive measure of mass at most $m$ supported at
distance at least $L$ from $U$. Put
$\mu=\sum_jw_j(\gamma_j)_\#\bigl(h_j(r)\dd r\bigr)+\nu$ and
$\rho=\sqrt DR_0+\bar p+B_0$. There is $C<\infty$, depending only on the
class, $\beta$ and $c_0$, such that for every admissible blur and
$0<\varepsilon\le1$ the window density $\pi_\mu$ of $k_\eta*\mu$
satisfies
\[
 \begin{aligned}
 \norm{\pi_\mu-\pi_{\vartheta,\eta}}_{L^1(U)}
 &\le C\Bigl\{\varepsilon+e^{-(\ell-\rho)_+^2/(2\Lambda)}\\
 &\qquad +m\,e^{-(L-B_0)_+^2/(2\Lambda)}\Bigr\}.
 \end{aligned}
\]
The same bound holds for the spatial scores on $U$ once its right side is
below a threshold that depends only on the class. With
Proposition~\ref{prop:mismatch}, the minimum-distance estimator recovers
the tangent signature $S(\vartheta)$ with expected error
$O(n^{-1/2}+\varepsilon)$ plus the two exponential terms.
\end{proposition}

The tangents of a curved junction are thus recovered up to a bias of
order $M\sigma^\beta$, and nearby filaments matter only through a
Gaussian tail in their distance (Appendix~\ref{app:curvedproof}). The exponent $\beta$ is attained in
practice. In a numerical check on 72 curved junctions in three dimensions, with two to
four arms, anisotropic Gaussian blur, tangent angles that drift as
$cr^\beta$ and densities that drift as $1+dr^\beta$, we computed the
score of each curved junction, by quadrature along its arms, and the exact
score of its tangent rays at 223 points
within 2.25 blur widths of the vertex, for six blur scales
(Appendix~\ref{app:curvecheck}). The median log--log slope of the score
difference over the four smallest scales is 0.50, 0.75 and 1.00 for
$\beta=0.5$, $0.75$ and $1$, 24 junctions each
(Figure~\ref{fig:curves}).

\begin{figure}[t]
\centering\includegraphics[width=\columnwidth]{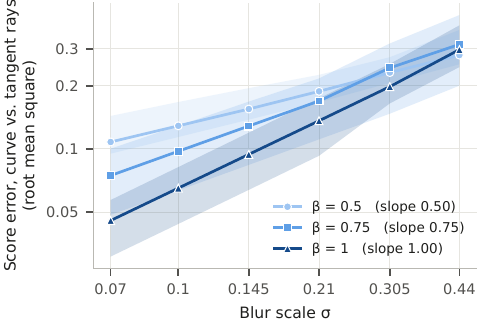}
\caption{Curved arms against their tangent rays: root-mean-square
difference of the spatial scores on 223 points within 2.25 blur widths
of the vertex, against the blur scale, for 24 random three-dimensional
junctions per H\"older exponent $\beta$. Lines: medians; bands:
interquartile ranges; slopes: medians of the per-junction slopes over the
four smallest scales. Proposition~\ref{prop:curvedarms} predicts slope
$\beta$.}
\label{fig:curves}
\end{figure}

At a fixed blur, curvature is a genuine obstruction rather than a
nuisance that more photons remove.

\begin{proposition}[A fixed-blur limit for curves]\label{prop:curves}
For $0<\beta\le1$ there is a bounded class of planar $C^{1,\beta}$
single branches, with unit density, unit initial tangent, bounded
$C^{1,\beta}$ norm and $|\gamma(r)|\ge3r/4$, on which every estimator
from $n$ window-conditioned exact positions, even with known vertex and
known Gaussian covariance, has worst-case expected tangent error at
least $c\,n^{-\beta/(2\beta+4)}$.
\end{proposition}

A smooth class allows tangent changes of order $h^\beta$ whose effect on
the observed density is of order $h^{\beta+2}$
(Appendix~\ref{app:curves}). For $\beta=1$ the rate is $n^{-1/6}$. The
$n^{-1/2}$ rate belongs to straight arms; for curved ones the blur scale
decides how straight an arm looks, and Proposition~\ref{prop:curvedarms}
says how much that costs.

\section{From the theorem to a fit}\label{sec:method}

The theorems concern an ideal estimator, a global minimum-distance fit
over the whole class. The experiments run a practical one, and
Table~\ref{tab:bridge} lists where the two differ.

\begin{table*}[t]
\centering\small
\caption{What the theory covers and what the experiments run.}
\label{tab:bridge}
\begin{tabular}{>{\raggedright\arraybackslash}p{0.13\textwidth}>{\raggedright\arraybackslash}p{0.38\textwidth}>{\raggedright\arraybackslash}p{0.38\textwidth}}
\toprule
 & Theory & Executed fit\\
\midrule
Source & infinite rays, weights $\ge w_*$, separation $\delta$; curved arms in Proposition~\ref{prop:curvedarms} & rays; arms at least $25^\circ$ apart, arm shares at least $0.02$\\
Blur & centered mixture of $J$ Gaussians, full covariances & two centered isotropic Gaussians\\
Background & constant (Proposition~\ref{prop:background}) & constant, share at most $0.8$\\
Window & cube; disk (Appendix~\ref{app:background}) & disk of radius 400\,nm\\
Data & independent positions or counts & localizations (with blinking); photons\\
Estimator & global minimum distance & local multistart likelihood\\
Count & penalized likelihood (Proposition~\ref{prop:count}) & BIC; in the microscopy studies, lowered when two detectors agree on fewer arms\\
Guarantee & Lipschitz inverse; rate $n^{-1/2}$ & measured, not proved (Section~\ref{sec:experiments})\\
\bottomrule
\end{tabular}
\end{table*}

\paragraph{The fitted density.}
We fit the points in a disk $W$ of radius $R$ around a candidate vertex
with the density
\begin{equation}\label{eq:fit}
 \begin{aligned}
 \pi(x)&=\frac{\omega_0}{|W|}+\sum_{j=1}^K\omega_j
 \frac{h_j(x)}{\int_Wh_j},\\
 h_j(x)&=\sum_{\ell=1}^{2}v_\ell\int_0^\infty
 \ph_{\sigma_\ell^2I}(x-p-ra_j)\dd r,
 \end{aligned}
\end{equation}
with background share $\omega_0$, arm shares $\omega_j$
($\sum_{j=0}^K\omega_j=1$), vertex $p$, directions $a_j$, and a
centered two-component isotropic Gaussian blur $(v_\ell,\sigma_\ell)$.
Each $h_j$ has a closed form, a normal density across the arm times a
normal distribution function along it (Lemma~\ref{lem:rayformula}). Up to
normalization, \eqref{eq:fit} is the intensity of the model
\eqref{eq:model}--\eqref{eq:mixture} plus a constant: the ray weights are
proportional to $\omega_j/\int_Wh_j$, and the constraints keep them above
a positive floor, so the fitted class lies inside the class of
Proposition~\ref{prop:background} (Appendix~\ref{app:background}). For
microscopy the constraints are: arms at least $25^\circ$ apart,
$\omega_j\ge0.02$, $\omega_0\le0.8$, $\sigma_\ell\in[3,200]$\,nm, and the
vertex within $200$\,nm of the window center in each coordinate, with
$R=400$\,nm.

\paragraph{Count.}
For each $K\in\{2,\dots,5\}$ we run sequential quadratic programming from
several starts: four evenly spaced direction patterns, up to six
combinations of peaks of the angular histogram of the points at $0.35R$
to $R$, and the arms returned by earlier likelihood fits, by a Hessian
ridge detector and by a line fit (Section~\ref{sec:assembled}). The
Bayesian information criterion (BIC) \citep{schwarz78},
$2\,\mathrm{NLL}+(2K+5)\log n$, picks a count $B$. With many
localizations the BIC adds arms that absorb blinking clusters and blur
misfit, so when the Hessian ridge detector and the line fit return the
same count $H=L<B$, we report the likelihood fit with that count. The
reported count $G$ satisfies
\[
 \{G\ne K_*\}\subseteq\{B\ne K_*\}\cup\{H=L<K_*\},
\]
and Appendix~\ref{app:selection} gives a margin under which both events
are exponentially rare.

\paragraph{What each part contributes.}
With the BIC count alone, the fit beats the Hessian ridge detector by
19.9 points $[12.6,26.7]$, 29.4 $[22.3,36.0]$ and 32.1 $[26.2,37.8]$ of
exact recovery in the three studies of Section~\ref{sec:assembled}. The
agreement rule changes this by $+4.2$, $+7.7$ and $-2.7$ points over 25
to 400 localizations and adds 12 to 22 points at 800, where the BIC
over-counts; it also removes faint arms that hold 2 to 8\% of the window
(Section~\ref{sec:faint}). Removing every start taken from the detectors
changes exact recovery by $+0.3$, $+0.2$ and $+0.3$ points and keeps at
least 99.7\% of the reported counts; with neither the detectors' starts
nor the rule, the fit still beats the detector by 20.1, 29.6 and 32.2
points. The fit is a local search, and we measured how far it falls
short. In the cross-laboratory study, 30 extra random starts improve 13\%
of the reported fits at 400 and 800 localizations by more than $0.5$
log-likelihood units and move no reported direction by more than
$1.2^\circ$. A stronger fixed-budget start protocol recovers the few fits
that were stuck far from better optima, at 1.5 times the run time, and
changes the reported count in 1 of 240 windows
(Appendix~\ref{app:diagnostics}). The median window takes 1.6 to 2.9\,s
for all four counts on one Apple M4 processor.

\section{Experiments}\label{sec:experiments}

We test the theory in settings of increasing distance from its
assumptions: simulated counts inside the model class, junctions assembled
from measured localizations with known arms, the acute angles of natural
junctions, faint arms, X-ray grating spectra whose directions come from
instrument calibration, and road intersections traced by taxis. Intervals are 95\% percentile bootstrap intervals that
resample the natural unit of dependence (fields of view, sources, map
tiles), with 5000 resamples for the endpoints of the three microscopy
studies and the X-ray study and 2000 for the natural junctions, the
exactly balanced crossings, the road intersections and the figures.

\subsection{Inside the model class}\label{sec:simulations}
Forty-eight planar sources, 12 in each combination of three or four arms
and two or three blur components, with arm weights at least $0.03$ and
blur covariance eigenvalues in $[0.36,1.96]$, were each observed as three
$16\times16$ count vectors on $[-3,3]^2$ at $n=10^3,\dots,10^6$ photons.
A blind fit searches up to four arms and three blur components and
chooses both by deviance plus $(2K+6J-2)\log n$. At $10^6$ photons it has
mean capped signature loss $0.023$ and the right count on all 144
vectors (Figure~\ref{fig:mixture}a). A one-Gaussian fit has loss $0.508$
and the right count on 83 (paired difference $-0.485$
$[-0.553,-0.411]$, source bootstrap): a blur model with too few
components does not merely lose precision, it invents arms. A fit given
the true blur reaches $0.014$. The blind fit selects two blur components
on all 144 high-count vectors, including the 72 whose blur had three:
the geometry does not need the blur's component count.
Figure~\ref{fig:mixture}b shows twelve balanced crossings observed under
two blurs whose covariances differ along $t(ab^\top+ba^\top)$. Their cell
laws agree to $10^{-16}$, and at $10^6$ photons the losses are $0.028$ and
$0.033$. The fit cannot tell the two blurs apart and does not need to.

\begin{figure*}[t]
\centering\includegraphics[width=\textwidth]{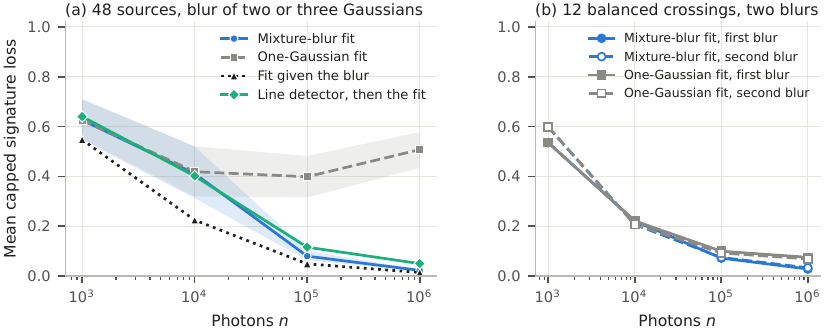}
\caption{Simulated counts inside the model class. (a) Mean capped
signature loss (one for a wrong count) on 48 sources; bands: 95\%
intervals from resampling sources within the four design cells. (b)
Twelve balanced crossings, each observed under two blurs with identical
cell laws.}
\label{fig:mixture}
\end{figure*}

With one Gaussian and unknown full covariance, the same behavior appears
at the level of the covariance itself. On twelve nonorthogonal crossings
with unknown vertex, fitted with four free directions, the fitted
covariance at exact balance approaches the family of covariances that
the data cannot distinguish (mean distance $0.452$, $0.157$, $0.044$ and
$0.010$ from $10^3$ to $10^6$ photons) while staying far from the
covariance that generated the data ($0.898$ to $0.833$). At an imbalance
of $0.06$, where the lines are no longer balanced, the covariance error at
$10^6$ photons drops to $0.171$. The arms do not notice the difference:
their mean signature loss at $10^6$ photons is $0.025$ to $0.032$ at each
of the four imbalances from $0$ to $0.06$.

\subsection{Junctions assembled from measured localizations}\label{sec:assembled}
Junctions whose arms are known are rare in measured data, so we assemble
them from measured localizations. We take straight, isolated microtubule
stretches from SMLM fields in the ShareLoc repository \citep{ouyang22},
cut each at a point, and rotate it rigidly, with no rescaling or
resampling, so that the cut lands on a common vertex. A cross joins two
stretches (four arms, crossing angle drawn from $30^\circ$ to $90^\circ$);
a T-junction joins a stretch and a half-stretch (three arms, $35^\circ$ to
$145^\circ$). Each has a control in which one arm of a straight pair is
turned by $25^\circ$ to $45^\circ$. A measured empty disk from the same
field supplies background, and all arms are at least $30^\circ$ apart.
Every localization is measured, with its own blinking and precision; only
the placement is synthetic. The window center is offset from the vertex
by up to 50\,nm, and the methods see localizations only. From each
junction we draw $n\in\{25,50,100,200,400,800\}$ localizations.

Three studies had written protocols, frozen before any evaluation fit:
data, construction, methods, settings, endpoints and analysis. Each has
120 junctions. The \emph{cross-laboratory} study uses 27 fields of COS-7,
REF52 and U2OS cells from two laboratories, two dyes and two camera
types; the \emph{U-373} study 28 fields of U-373\,MG cells; the
\emph{unused acquisitions} study 22 fields from 34 acquisitions that no
earlier analysis had opened, with reference directions taken from
localizations of the same stretches that are not in the fitted
subsample. No field was used for development. \emph{Exact recovery} means
the correct count and every arm within $10^\circ$ of its reference.

The main comparator is a Hessian ridge detector: it renders the
localizations at 10\,nm, smooths at 50\,nm, keeps the strongest 70\% of
ridge pixels between $0.35R$ and $R$, and reads the arms from peaks of
the strength-weighted orientation histogram. Its settings were tuned on
development fields. Because the fit estimates the blur of every window
while the detector smooths at one width, we also ran the same detector
with automatic scale selection at every pixel \citep{lindeberg98},
choosing among eleven scales from 15 to 100\,nm by a normalized ridge
strength. Its settings were calibrated on 240 junctions of an earlier
study that shares no field with these three, and it was then run once on
every stored window; since the three studies had been analysed before,
this comparison is retrospective (Appendix~\ref{app:adaptive}). Four more
detectors, Steger's detector in its ImageJ implementation
\citep{steger98}, a line fit, angular histograms and skeletonization, are
reported in Appendix~\ref{app:outcomes}; the likelihood fit beats each of
them in each study.

\begin{figure*}[t]
\centering\includegraphics[width=\textwidth]{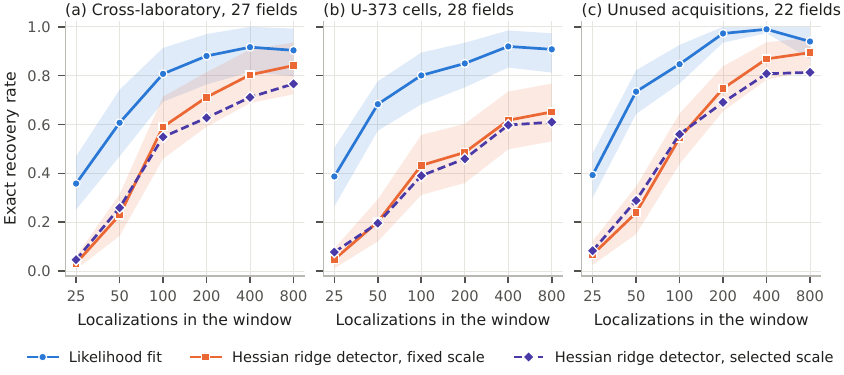}
\caption{Exact recovery (correct count, every arm within $10^\circ$) in
the three prospective studies, 120 junctions each. References: assigned
directions in (a) and (b); held-out localizations of the same filaments
in (c). Bands: 95\% field-bootstrap intervals (equal field weight) for
the fit and the fixed-scale detector. At 800 localizations, 97 junctions
in (c) have enough localizations, and one of them has no held-out
reference left for one arm; it counts as a failure for every method.}
\label{fig:recovery}
\end{figure*}

\begin{table*}[t]
\centering\small
\caption{The likelihood fit against the Hessian ridge detector. Exact
recovery: difference in percentage points, averaged over 25 to 400
localizations. Acute angles: error of the angle between adjacent arms
less than $60^\circ$ apart, in degrees, on the angles where both methods
of the pair find the correct count. Brackets: 95\% field-bootstrap
intervals (5000 resamples).}
\label{tab:endpoints}
\begin{tabular}{lccc}
\toprule
 & Cross-laboratory & U-373 & Unused acquisitions\\
\midrule
Fields & 27 & 28 & 22\\
\addlinespace
\multicolumn{4}{l}{\emph{Exact recovery, fit minus detector (points)}}\\
\quad fixed scale & $+24.1$ $[18.8,29.0]$ & $+37.1$ $[31.3,42.3]$ & $+29.4$ $[24.3,34.1]$\\
\quad selected scale & $+27.5$ $[21.4,33.4]$ & $+38.4$ $[32.2,44.6]$ & $+30.1$ $[23.8,36.2]$\\
\addlinespace
\multicolumn{4}{l}{\emph{Acute angles, signed error (degrees)}}\\
\quad fixed-scale detector & $+9.3$ $[6.2,12.8]$ & $+12.4$ $[8.7,15.9]$ & $+7.2$ $[5.0,9.4]$\\
\quad fit & $+0.4$ $[-0.7,1.4]$ & $+1.3$ $[-0.2,3.0]$ & $+0.9$ $[-0.3,2.0]$\\
\addlinespace
\multicolumn{4}{l}{\emph{Acute angles, absolute error (degrees)}}\\
\quad fixed-scale detector & $11.6$ $[9.4,14.4]$ & $15.5$ $[12.6,18.7]$ & $10.4$ $[8.6,12.4]$\\
\quad fit & $3.6$ $[2.8,4.5]$ & $4.3$ $[3.2,5.6]$ & $3.5$ $[3.1,3.9]$\\
\bottomrule
\end{tabular}
\end{table*}

Averaged over 25 to 400 localizations, the likelihood fit beats the
fixed-scale detector by 24.1, 37.1 and 29.4 points in the
cross-laboratory, U-373 and unused-acquisition studies, and the
selected-scale detector by 27.5, 38.4 and 30.1 points
(Figure~\ref{fig:recovery}, Table~\ref{tab:endpoints}). The gain is
largest where data are scarce. At 50 localizations the fit recovers 61\%,
68\% and 73\% of the junctions and the fixed-scale detector 23\%, 20\%
and 24\%. At 800 localizations the detector closes most of the gap in two
studies; on the U-373 cells it still trails by 26 points. Selecting the
scale does not help the detector: against the fixed scale it changes
exact recovery by $-3.5$ $[-7.3,0.3]$, $-1.3$ $[-5.1,2.3]$ and $-0.7$
$[-4.5,3.6]$ points over 25 to 400 localizations and loses 4 to 8 points
at 800. Calibration had to exclude scales below 50\,nm, which follow the
scatter of single localizations, and the selection then settles on 50 to
60\,nm at every size, close to the tuned fixed scale. Cutting the arms
at the vertex, as the construction does, costs the fixed-scale detector 5
to 8 points in a paired synthetic control and leaves the fit unchanged
(Appendix~\ref{app:diagnostics}); the gains include this effect.

\subsection{Rate and precision}\label{sec:rate}
In the two studies whose references are defined on the scale of the
window, the median arm error of fits with the correct count (76 to 118
of 120 fits per size) falls from $4.7^\circ$ at 25 localizations to
$1.4^\circ$ at 400 in the cross-laboratory study (log--log slope $-0.435$
$[-0.526,-0.337]$) and from $5.0^\circ$ to $1.4^\circ$ on the unused
acquisitions (slope $-0.457$ $[-0.536,-0.393]$;
Figure~\ref{fig:ratebias}a,b). Both slopes are consistent with the
$-1/2$ of Theorem~\ref{thm:statistics}. The local sensitivity at the
fitted parameters agrees: on the unused acquisitions the median sandwich
standard error of an arm is $21^\circ/\sqrt n$ from 100 to 800
localizations, and observed errors stay within 0.7 to 1.2 times the
clustered standard error from 50 localizations on. The detector's error
levels off near $3^\circ$, as a smoothing bias predicts. The
unconditional loss, which charges $90^\circ$ for a wrong count, is lower
for the fit at every size in all three studies; at 400 and 800
localizations in the cross-laboratory study the difference is within
noise. On U-373 cells the slope, $-0.331$ $[-0.380,-0.244]$, falls just
outside the prespecified range $[-0.65,-0.35]$, so that criterion fails.
Those references extend over twice the window, and slight curvature
gives them a floor; against the direction of each stretch's own
localizations within the window, the slope is $-0.442$ (post hoc;
Appendix~\ref{app:outcomes}).

\begin{figure*}[t]
\centering\includegraphics[width=\textwidth]{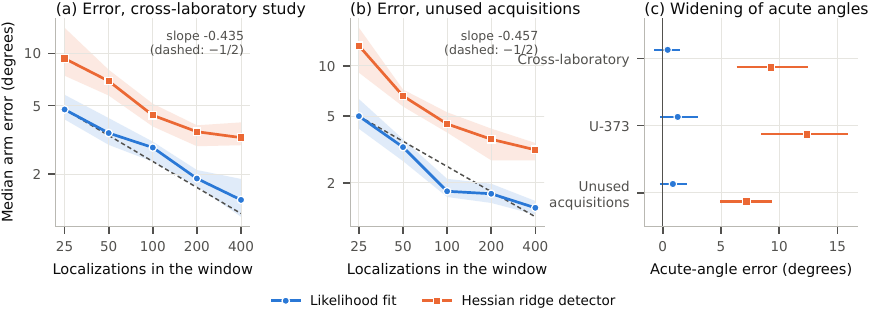}
\caption{(a,b) Median arm error of fits with the correct count; dashed
lines have slope $-1/2$. (c) Mean signed error (estimate minus
reference) of the angle between adjacent arms less than $60^\circ$
apart, on the junctions where both methods find the correct count (308,
305 and 299 angles). Bars and bands: 95\% field-bootstrap intervals.}
\label{fig:ratebias}
\end{figure*}

\subsection{Acute angles}\label{sec:acute}
For adjacent arms less than $60^\circ$ apart, on the 308, 305 and 299
angles where both methods find the correct count, the fixed-scale
detector's mean signed error is $9.3^\circ$, $12.4^\circ$ and $7.2^\circ$
in the three studies, and the fit's $0.4^\circ$, $1.3^\circ$ and
$0.9^\circ$ (Figure~\ref{fig:ratebias}c): the detector widens, the fit
does not. The selected-scale detector widens more, by $12.4^\circ$,
$13.3^\circ$ and $9.4^\circ$ on the angles where both detectors are
right. The widening is a property of the smoothed image, not of the
noise.

\begin{proposition}[A fixed-scale Hessian widens crossings]\label{prop:hessian}
Let two equally bright full lines cross at angle $\theta\in(0,\pi/2)$
under an isotropic Gaussian blur of total variance $V$, the instrument's
plus the smoothing's. At distance $t$ from the crossing along one line
put $r=t^2\sin^2\theta/V$ and $q=(r-1)e^{-r/2}$. The tangent, perpendicular
to the eigenvector of the most negative Hessian eigenvalue, turns away from
the other line by
\[
 \tfrac12\operatorname{atan2}\bigl(q\sin2\theta,\,1-q\cos2\theta\bigr),
\]
which is positive for every $r>1$ and largest at $r=3$.
\end{proposition}

The detector reads arms at $t\ge140$\,nm with $\sqrt V\approx57$\,nm for
a 28\,nm blur, so $r>1$ for every angle above $24^\circ$: both arms of an
acute pair turn apart. On noise-free renderings of a cross it reads
$30^\circ$ as $49^\circ$ to $52^\circ$ and $60^\circ$ as $65^\circ$ to
$68^\circ$ (Table~\ref{tab:noisefree}). Inside the core, $r<1$, the
tangent turns toward the other line instead, so a detector that reads
orientations close to the crossing mixes both signs. The likelihood
models the blur instead of adding to it.

\begin{table*}[t]
\centering\small
\caption{The frozen fixed-scale Hessian ridge detector on noise-free
renderings of two full lines through the window center, with 10\%
uniform background; every error is bias.}
\label{tab:noisefree}
\begin{tabular}{lccccc}
\toprule
True crossing angle & $30^\circ$ & $40^\circ$ & $50^\circ$ & $60^\circ$ & $75^\circ$\\
\midrule
Detector, 20\,nm blur & $49.0^\circ$ & $51.0^\circ$ & $57.7^\circ$ & $65.4^\circ$ & $76.3^\circ$\\
Detector, 30\,nm blur & $51.9^\circ$ & $54.7^\circ$ & $59.6^\circ$ & $67.9^\circ$ & $76.8^\circ$\\
\bottomrule
\end{tabular}
\end{table*}

\paragraph{Natural junctions.}
The bias reaches reported summaries. At 1{,}391
natural microtubule junctions in 73 fields we ran both methods on all
localizations of one half of each acquisition. On the 812 junctions in
71 fields where both return the reference count, the detector reads
junctions whose smallest angle is below $60^\circ$ as $4.3^\circ$
$[3.5,5.0]$ wider than the likelihood fit, and the difference fades as
the angle approaches $90^\circ$, as Proposition~\ref{prop:hessian}
predicts: $+4.5^\circ$ below $45^\circ$, $+4.1^\circ$ from $45^\circ$ to
$60^\circ$, $+2.8^\circ$ from $60^\circ$ to $75^\circ$ and $+0.1^\circ$
$[-0.6,0.7]$ above $75^\circ$ (Figure~\ref{fig:natural}). The share of
junctions with an angle below $60^\circ$ is 44.5\% by the fit and 35.7\%
by the detector, a difference of 8.7 points $[6.0,11.7]$. Natural
junctions have no independent truth, so this measures how far the choice
of method moves a reported summary; on assembled junctions, where the
truth is known, the detector's widening has the same sign and a larger
size.

\begin{figure*}[t]
\centering\includegraphics[width=\textwidth]{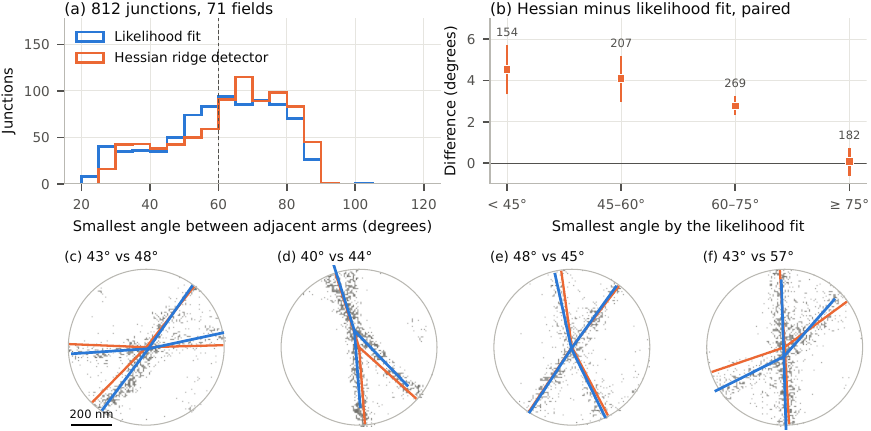}
\caption{Natural microtubule junctions. (a) Smallest angle between
adjacent arms on 812 junctions where both methods return the reference
count. (b) Hessian ridge detector minus likelihood fit, paired, by the
fit's angle; bars: 95\% field-bootstrap intervals; numbers: junctions per
bin. (c--f) Four acute junctions drawn with a fixed seed from the 361
whose angle by the fit is below $60^\circ$: localizations (up to 1{,}500
shown), likelihood fit (blue) and detector (orange); titles give the two
smallest angles.}
\label{fig:natural}
\end{figure*}

\subsection{Near the singular set}\label{sec:crossings}
A cross joins two straight filaments, so its opposite arms come from one
filament: their per-length intensity ratio has median 0.90, 0.84 and 0.83
in the three studies (10th percentiles 0.65 to 0.69). These junctions lie
near the balanced crossings of Proposition~\ref{prop:classification},
where the covariance is weakly identified. In the cross-laboratory study,
exact recovery at the 30 crosses differs from that at their 30 paired
controls, built from the same arms with one arm turned off its line, by
$-0.8$ points $[-8.5,6.9]$ over 25 to 400 localizations, inside the
prespecified equivalence margin of $\pm10$ points; at T-junctions the
difference is $-1.9$ $[-10.0,5.9]$, on the margin. A separate test on
exactly balanced crossings, built from four filaments with equalized line
weights and run with an earlier implementation of the fit, gave $-1.1$ $[-9.3,6.2]$ points at 100 localizations and $+0.6$
$[-7.3,7.9]$ at 400, with equal median errors. The blur is lost at these
crossings and the arms are recovered as well as anywhere else. Of the
four crossing tests, the one on U-373 cells missed its margin, with one
positive and one negative difference (Appendix~\ref{app:outcomes}).

\subsection{Faint arms and points beyond the window}\label{sec:faint}
Section~\ref{sec:counting} showed that points inside the window cannot
separate a faint arm from structured clutter. Points beyond the window
can. The continuation check keeps a candidate arm added by a larger count
only if its sector holds more localizations than the neighbouring sectors
in both of two rings between $R$ and $2R$ (a one-sided binomial test with
a family level of $0.05$). Part~B of the unused-acquisition study tests
it on 60 junctions whose second filament was thinned so that each of its
arms holds 2 to 8\% of the window, and on the same 60 geometries unthinned.
Both detectors miss these arms almost always, and the agreement rule
removes arms that the BIC had right. The continuation check restores
them: it adds $50.3$ points $[37.6,64.5]$ of count accuracy at 400 and
800 localizations, raising the correct count at 800 from 0.26 to 0.79,
and it adds no extra arm at 100, 400 or 800 localizations
(Figure~\ref{fig:faint}).

\begin{figure*}[t]
\centering\includegraphics[width=\textwidth]{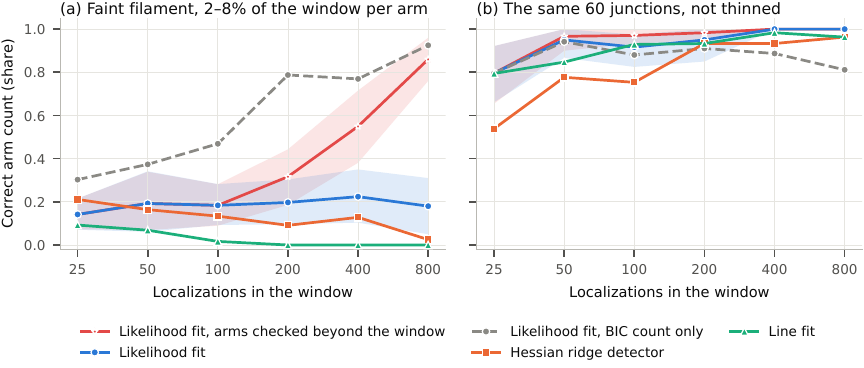}
\caption{Correct arm count in part B of the unused-acquisition study,
equal field weight. (a) Junctions whose second filament holds 2 to 8\% of
the window per arm. (b) The same 60 geometries, not thinned. Bands: 95\%
field-bootstrap intervals for the two agreement-rule variants. Few faint
junctions have 800 localizations inside the window.}
\label{fig:faint}
\end{figure*}

\subsection{X-ray photons}\label{sec:xray}
The High Energy Transmission Grating on the Chandra X-ray Observatory
disperses a point source into two straight spectra that cross at the
source image at $9.96^\circ$ \citep{canizares05}: a four-arm junction whose
directions follow from the spacecraft roll and the grating angles, not
from the photons. Photon arrivals are independent, as in
Theorem~\ref{thm:statistics}. After a first study with generic starts
failed its prespecified criteria, we fixed a new protocol and applied the
fit to 30 new sources taken in a fixed random order from the public
archive, with directions from calibration alone, photons of 0.5 to
8\,keV within 1000 detector pixels ($492''$) of the source image, and
$n\in\{100,200,400,800\}$ with three draws each
(Appendix~\ref{app:xraydesign}). The calibrated directions agree with a
map fitted to all photons of each observation within $0.0016^\circ$. The
likelihood runs with instrument floors (arms $4^\circ$ apart, arm shares
at least $0.004$) and chooses one or two blur components by BIC. The ray
model is an approximation here: each arm starts some distance from the
source image, its intensity follows the source spectrum, and the
line-spread function has wings.

When all four arms are found within $10^\circ$ (18, 24, 36 and 38 of 90
draws at 100, 200, 400 and 800 photons), the median error falls from
$0.073^\circ$ to $0.023^\circ$, with slope $-0.518$ $[-0.785,-0.308]$
(Figure~\ref{fig:xray}b): precision against a physical reference, at the
rate the theory predicts. Complete recovery within $1^\circ$ reaches 40\%
at 800 photons. On these thin, isolated spectra the simple geometric
methods keep up: a line fit and an angular histogram given the true
vertex also reach 41\%. The tuned Hessian ridge detector never recovers
all four arms within $1^\circ$ and returns fewer than four arms in 78\% of
the draws at 800 photons; the fit leads it by 38.3 points $[25.0,51.7]$
at 400 and 800 photons.

\begin{figure*}[t]
\centering\includegraphics[width=\textwidth]{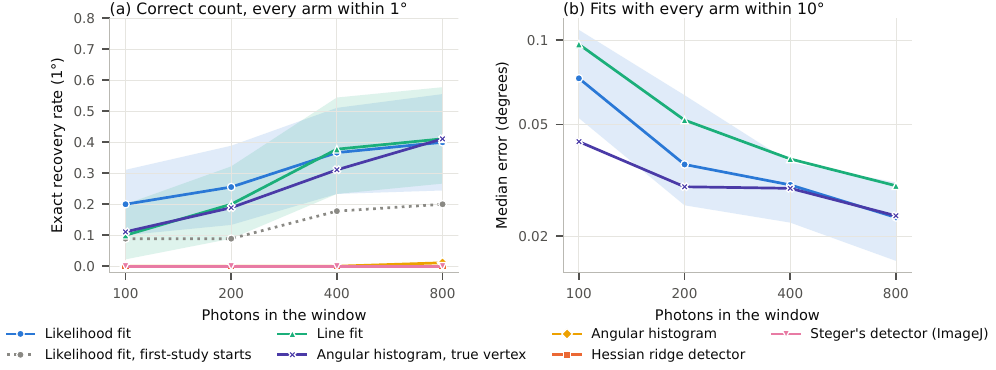}
\caption{Chandra grating spectra, 30 sources and three draws per size.
(a) Exact recovery within $1^\circ$ of the calibrated directions for every
method; bands: 95\% source-bootstrap intervals for the fit and the line
fit. (b) Median error of fits with all four arms within $10^\circ$; band:
95\% source-bootstrap interval for the fit.}
\label{fig:xray}
\end{figure*}

\subsection{Road intersections from taxi traces}\label{sec:gps}
A road intersection seen through GPS fixes is the same object: each fix is
a noisy position of a vehicle on one of the roads, and GPS error has heavy
tails. We used the Porto taxi traces released for the ECML/PKDD 2015
challenge \citep{moreira13}, kept the fixes of
moving vehicles, and took the true geometry of each intersection from
OpenStreetMap \citep{osm}. A site entered only if the maps of 2014 and
2026 agree on it (same number of arms, vertex within 3\,m, every arm
within $2^\circ$) and every arm runs straight to 40\,m. Of 117 sites, the
21 in randomly assigned one-kilometre tiles served for development; the
protocol, including the tuned settings of every detector, was frozen
before any estimate on the other 96 intersections in 27 tiles. From
windows of radius 40\,m we drew 25 to 1{,}600 fixes, three times each. The
fit is the earlier implementation of the two-Gaussian model with the
count chosen by BIC.

Averaged over all sizes, the fit recovers the intersection exactly 14.2
points $[8.0,21.7]$ more often than the tuned Hessian ridge detector, and
35.0 to 46.7 points more often than Steger's detector, a line fit, an
angular histogram and skeletonization (Figure~\ref{fig:gps}). The median
arm error of fits with the correct count falls with slope $-0.49$
$[-0.57,-0.42]$ from 25 to 400 fixes, while the detector's levels off at
$4.4^\circ$ to $4.7^\circ$ from 400 fixes on. The mixture earns its place
here. A one-Gaussian fit absorbs the heavy tail of GPS error with an extra
arm: from 200 fixes on it over-counts in 12 to 17\% of the windows,
against 5 to 9\% for two Gaussians, which recover 7.3, 8.7, 5.0 and 8.2
points more intersections at 200, 400, 800 and 1{,}600 fixes. Per-arm
sandwich intervals covered the map directions in only 74 to 84\% of cases
at 100 to 400 fixes, as map errors of a few metres would produce.

The same frozen settings failed every prespecified criterion on 39
intersections in Berlin and Chicago from a public map-construction
benchmark \citep{ahmed15} and on 8 in Beijing \citep{zheng10}. There the
Hessian detector led the fit by 19.5 points $[13.1,28.6]$, and in Beijing
the two tied ($-0.4$ $[-13.4,11.3]$). These windows carry little road:
the share of fixes within 10\,m of a mapped arm is 1.4 to 1.7 times what
uniform scatter would give, against 2.1 times in Porto, and 73.5\% of the
Berlin fixes sit on exactly repeated coordinates, which a continuous
likelihood explains with an arm of vanishing width. A model with a density
needs data that have one; where that fails, the smoothing in a detector
degrades more gracefully.

\begin{figure*}[t]
\centering\includegraphics[width=\textwidth]{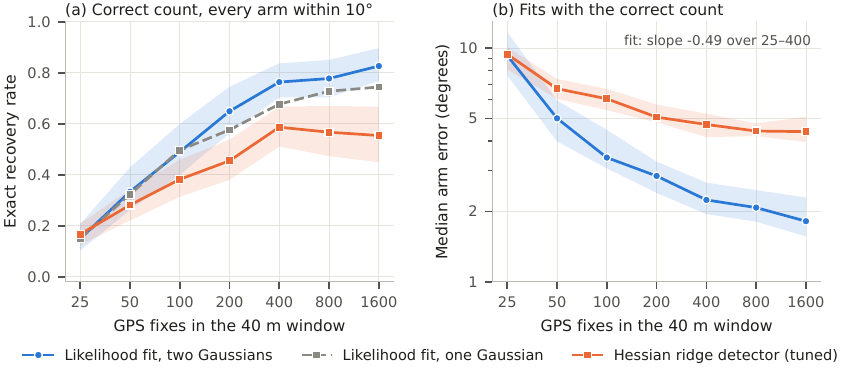}
\caption{Road intersections from Porto taxi traces, 96 evaluation
intersections in 27 one-kilometre tiles, against OpenStreetMap geometry.
(a) Exact recovery. (b) Median arm error of fits with the correct count.
Bands: 95\% tile-bootstrap intervals (2000 resamples).}
\label{fig:gps}
\end{figure*}

\subsection{The blur of measured images is a mixture}\label{sec:biosr}
The model allows several blur components because measured blurs have
wings. On 50 microtubule fields of the BioSR collection \citep{qiao21},
we predicted the mean of nine raw structured-illumination frames from the
paired high signal-to-noise reconstruction, with one or two Gaussian
widths and an affine background, fitting on alternating tiles and scoring
on the others. Two widths lowered the held-out mean squared error in all
50 fields, by $45.6\%$ $[42.2,49.0]$ on average. A stronger one-Gaussian
comparator, added after this result, with a fitted subpixel shift, a
quadratic background and an anisotropic covariance, left a reduction of
$51.3\%$ $[47.0,55.7]$, again in all 50 fields. For junction angles in these images the second width makes
no difference ($-0.035^\circ$ $[-0.226,0.148]$ on 222 four-arm sites), as
Theorem~\ref{thm:inverse} leads one to expect: the geometry does not
depend on how the blur is split into components, so a richer blur class
costs nothing where it is not needed and protects the count where it is
(Section~\ref{sec:simulations}).

\section{Related work}\label{sec:related}

\paragraph{Singular and weakly identified models.}
Parameter estimation in finite mixtures slows down when components merge
or vanish \citep{ho16}, the Laplace approximation behind the BIC can fail
at such points \citep{drton17}, and weak identification in
semiparametric models has a theory of its own \citep{kaji21}. Our blur is such a mixture, and it is worse than weakly
identified at a balanced crossing: it is not identified at all. The
geometry nevertheless keeps the rate $n^{-1/2}$ through every one of these
singularities, because the degeneracy lies entirely in the nuisance.
Target identifiability under generative indeterminacy \citep{xi23}
separates what data can determine from what they cannot; we add a
uniform Lipschitz bound and a rate for the determined part, and the
analytic order criterion of Proposition~\ref{prop:ordercriterion} is a
general route to such bounds. General Lipschitz stability results for
inverse problems with finite measurements
\citep{bourgeois13,alberti22} assume an injective forward map, which
fails here.

\paragraph{Deconvolution and super-resolution.}
Deconvolution with an unknown error distribution is usually handled by
repeated measurements or by restricting the error law
\citep{meister09}. Super-resolution and sparse deconvolution recover point
sources through a known kernel \citep{candes14,duval15,poon19}. Our
kernel is unknown and our sources are rays, and we recover the rays
without recovering the kernel. Gaussian mixtures are also used as models
of measured point-spread functions \citep{samuylov19}, which supports the
blur class we assume.

\paragraph{Filaments and graphs from points.}
Statistical work on filaments in point clouds tests for them against
uniform clutter \citep{ariascastro06}, estimates them as the medial axis
of the support \citep{genovese12} or as ridges of a kernel density
estimate \citep{genovese14}, and reconstructs the graph they form, with
guarantees on its topology \citep{aanjaneya12,lecci14}. These methods
target the filaments and their connectivity without a blur model; we
target the geometry at a single junction under an unknown blur, and our
guarantees are rates for angles, weights and lines.

\paragraph{Ridge and vessel filters.}
Ridge detection reads lines from Hessian eigenvectors at a chosen or
selected scale \citep{lindeberg98,frangi98}, and \citet{steger98} removes
the position bias of asymmetric line profiles. The angle bias at a
crossing is a different effect, which Proposition~\ref{prop:hessian}
computes for the fixed-scale Hessian; scale selection did not remove it
in our data. The spatial score $\nabla\log q$ is the object of score
matching \citep{hyvarinen05}; here it is an analytic observation rather
than a learned field.

\section{Discussion}\label{sec:discussion}

A blur that cannot be calibrated need not blur the geometry. The
mechanism is general: far along an arm the data show the blur only
through a centered transverse profile, whose mass and mean return the
weight and the line, and a change of direction shifts that profile by an
amount that grows with distance. Proposition~\ref{prop:ordercriterion}
turns this kind of order comparison into a uniform Lipschitz bound. The
same route is open for other targets observed through a nuisance that
cannot be identified, such as an edge under an unknown point-spread
function or a source position under an unknown beam; in each case the
work is an asymptotic argument that shows the target changing no later
than the observation.

The practical lesson is equally direct. Modeling the blur beats removing
it by smoothing, even when the blur itself is lost. On junctions assembled
from measured localizations from three laboratories, on natural
microtubule junctions, on X-ray grating spectra and on road intersections
traced by taxis, the likelihood fit recovers junctions more often than a
Hessian ridge detector, at a fixed or a selected scale, or matches the best
simple method, and it does not widen acute angles. Measured widefield
images show that actual crossings sit at the ambiguity the theory
describes, and a blur class with several widths pays off where the error
has heavy tails.

The results also mark where the approach stops. It needs a candidate
window and point data; finding the windows is a detection problem that we
do not treat. It needs data with a density: on raw widefield frames of
blinking fluorophores, whose pixels are not independent point positions,
the fit over-counts arms and the Hessian detector recovers 18 to 21 points
more junctions, and on GPS traces with little road signal and stacked
duplicate coordinates it lost to the same detector. The theorems describe
a global fit of the model, and the local search we run falls short of it
by an amount that we measured rather than bounded. And the rate
$n^{-1/2}$ belongs to straight arms: curvature costs a bias set by the
curvature at the scale of the blur, which no number of points removes at
a fixed blur.

\paragraph{Data.}
All data are public. The microscopy tables are ShareLoc records on
Zenodo, released under the Creative Commons Attribution 4.0 license; the
record numbers are listed in Appendix~\ref{app:data}. The X-ray event
files are in the Chandra Data Archive, the widefield image stacks in the
DL-SMLM collection on figshare, the structured-illumination images in the
BioSR collection, the Porto taxi traces in the UCI Machine Learning
Repository, the Berlin and Chicago traces at mapconstruction.org, the
Beijing traces in the GeoLife collection, and the road maps in
OpenStreetMap.

\vspace{1.6ex}
\noindent\begin{minipage}{\columnwidth}
\paragraph{Use of AI tools.} Generative AI tools assisted with code implementation, experimental design and analysis, literature review, figure preparation, and drafting and editing this manuscript. The authors are responsible for the claims, proofs, results, references, and final submitted text.
\end{minipage}

\clearpage
\onecolumn
\appendix
\pdfbookmark[0]{Appendix}{appendix}
\begin{center}{\LARGE\bfseries Appendix}\end{center}
\vspace{1ex}
\noindent The appendix gives the proofs that do not appear in the main text
(Appendix~\ref{app:proofs}), the design of the experiments
(Appendix~\ref{app:design}), every pre-declared outcome
(Appendix~\ref{app:outcomes}) and diagnostics of the fit
(Appendix~\ref{app:diagnostics}).

\section{Proofs}\label{app:proofs}

Throughout, the blur of \eqref{eq:mixture} is stored in $J$ slots
$(v_\ell,b_\ell,\Sigma_\ell)$. A slot may carry weight zero; the weights
sum to one. Directions are written in local analytic charts of the unit
sphere, and finitely many charts cover the compact class $\Ccal$.

\subsection{The ray formula and the far profile}\label{app:mixture}

\begin{lemma}[Ray formula]\label{lem:rayformula}
Fix a unit vector $a$, a vertex $p$, a component mean $b$, a covariance
$\Sigma\succ0$ and a coefficient $c\ge0$. Put $z=x-p-b$, $Q=\Sigma^{-1}$,
$A=a^\top Qa$ and $B=Q-Qaa^\top Q/A$. Let $E$ be an orthonormal basis
of $a^\perp$ and $C=E^\top\Sigma E$. Then
\begin{equation}\label{eq:rayformula}
 c\int_0^\infty\ph_\Sigma(z-ra)\dd r
 =c\,\ph_C(E^\top z)\,\Phi\!\left(\frac{a^\top Qz}{\sqrt A}\right),
 \qquad B=EC^{-1}E^\top,
\end{equation}
where $\Phi$ is the standard normal distribution function.
\end{lemma}

\begin{proof}
Completing the square in $r$ gives
$(z-ra)^\top Q(z-ra)=z^\top Bz+A(r-a^\top Qz/A)^2$. The integral of
the second factor over $r\ge0$ is $\sqrt{2\pi/A}\,\Phi(a^\top Qz/\sqrt
A)$. In the orthonormal basis $(a,E)$, the Schur complement of $C$ in
$\Sigma$ equals $1/A$, so $|\Sigma|A=|C|$ and $B=EC^{-1}E^\top$; the
remaining prefactor is then $\ph_C(E^\top z)$.
\end{proof}

\paragraph{Analyticity.}
Near the compact set of admissible parameters and window points,
complexify the direction charts and the covariance entries. The
coefficient of $r^2$ in the exponent keeps a real part of at least
$c>0$, so every
fixed parameter or spatial derivative of the integrand in
\eqref{eq:rayformula} is dominated by $C(1+r)^M\exp(-cr^2+Cr)$ with
$c>0$. Integrating holomorphic functions under this bound shows that $q$
is jointly real-analytic in parameters and position. On the original compact
set $q$ has a positive minimum, so $g=\nabla\log q$ and $q/Z$ are
analytic on a possibly smaller neighborhood. When a slot has weight zero
the complex extension may give it a negative weight; positivity of $q$
is needed only near the compact window, where it holds by continuity.

\paragraph{Decay along a direction.}
For a unit vector $v$, the Gaussian tails in \eqref{eq:rayformula} give
\begin{equation}\label{eq:orientedrate}
 \lim_{R\to\infty}\frac{-2\log q(Rv)}{R^2}
 =\min_{j,\ \ell:\,v_\ell>0}\ \min_{s\ge0}\
 (v-sa_j)^\top\Sigma_\ell^{-1}(v-sa_j).
\end{equation}
Each term is nonnegative and vanishes exactly when $v$ is one of the
oriented directions $a_j$. The constraint $s\ge0$ matters: minimizing
over all $s\in\R$ would also vanish on the opposite half-line. The
bounded vertex and component means only affect lower-order terms.

\paragraph{The far profile.}
Along a ray $a_j$ every other ray has a positive rate in
\eqref{eq:orientedrate}. With \eqref{eq:rayformula} this gives, uniformly
for bounded $y$,
\begin{equation}\label{eq:profile}
 q(p+Ra_j+Ey)=w_jf_{a_j,\eta}(y)+o(1),\qquad R\to\infty,
\end{equation}
with $f_{a,\eta}$ as in \eqref{eq:mixtureprofile}. The profile integrates
to one, and its first moment is $\sum_\ell v_\ell E^\top b_\ell=0$
because the mixture is centered.

\begin{lemma}[Identification]\label{lem:identification}
If two admissible models have equal spatial scores on a nonempty open
subset of $U$, their signatures are equal.
\end{lemma}

\begin{proof}
Equal scores make $\log q'-\log q$ constant on the open set, so $q'=cq$
there, and by analyticity on all of $\R^D$. Equation
\eqref{eq:orientedrate} then gives the same count and the same oriented
directions. For a matched ray, \eqref{eq:profile} for both models reads
$w_j'f'(y-E^\top p')=c\,w_jf(y-E^\top p)$. Integrating in $y$ gives
$w_j'=cw_j$, and summing over $j$ gives $c=1$. Taking first moments gives
$E^\top p'=E^\top p$, which is the equality of the transverse offsets.
Nothing about the mixture slots is used.
\end{proof}

\subsection{The analytic order criterion}\label{app:order}

\begin{proof}[Proof of Proposition~\ref{prop:ordercriterion}]
Put $f(z,z')=\norm{T(z)-T(z')}^2$ and $Q(z,z')=\norm{F(z)-F(z')}^2$, both
real-analytic in the pair. If the conclusion failed, there would be pairs
with $f>0$ and $Q/f\to0$. By compactness they accumulate at a pair
$(z_*,z_*')$ with $Q=0$, hence $F(z_*)=F(z_*')$ and, by the
identification hypothesis, $f=0$ there. The semianalytic set
$\{(z,z',s):f>0,\ s>0,\ Q<s^2f\}$ has $(z_*,z_*',0)$ in its closure,
so real-analytic curve selection \citep{bierstone88} gives an analytic
arc $(z(t),z'(t),s(t))$ that starts there and lies in the set for $t>0$.
Let $r$ and $m$ be the first orders at which the target and the
observation differ along the arc, so $\ord f=2r$ and $\ord Q=2m$. Since
$s(0)=0$, $\ord(s^2f)\ge2+2r$, and $Q<s^2f$ forces $2m>2r$. This
contradicts the hypothesis $r\ge m$. If instead $F(z(t))=F(z'(t))$ for all
$t$, identification gives $f\equiv0$, contradicting $f>0$.
\end{proof}

The criterion is applied in each local chart where the ray matching is
fixed. Pairs with different counts lie in finitely many compact strata
whose observation images are disjoint by Lemma~\ref{lem:identification},
and hence a positive distance apart.

\subsection[All-order rigidity and the inverse bound]{All-order rigidity and the proof of \eqref{eq:inverse}}

\begin{lemma}[All-order rigidity]\label{lem:mixturerigidity}
Let two admissible analytic parameter arcs have equal scores at $t=0$,
and let their score difference be $O(t^m)$ in $L^2(U)$. After the rays
are matched at $t=0$, their directions, weights and transverse offsets
differ by $O(t^m)$.
\end{lemma}

\begin{proof}
Write $q=q(t,\cdot)$ and $q'=q'(t,\cdot)$ for the two intensities along
the arcs. By Lemma~\ref{lem:identification} the signatures agree at
$t=0$: the counts are equal and, after relabeling, $a_j'(0)=a_j(0)$,
$w_j'(0)=w_j(0)$ and the transverse offsets agree.

\emph{From scores to intensities.} Fix $x_*$ in the interior of $U$, put
$c(t)=q'(t,x_*)/q(t,x_*)$ and
\[
 H(t,x)=\log q'(t,x)-\log q(t,x)-\log c(t).
\]
Both intensities are positive and jointly analytic, so $H=\sum_kt^kH_k$
near $t=0$, uniformly on a neighborhood of $\bar U$, with every $H_k$
real-analytic on $\R^D$, $H_k(x_*)=0$ and $\nabla_xH=g'-g$. If $k_0$ is
the first index with $\nabla H_{k_0}\ne0$ in $L^2(U)$, then
$\norm{g'-g}_{L^2(U)}=t^{k_0}\{\norm{\nabla H_{k_0}}_{L^2(U)}+O(t)\}$.
The hypothesis therefore gives $\nabla H_k=0$ in $L^2(U)$ for all $k<m$,
and at every point of $U$ by continuity. Since $U$ is connected and
$H_k(x_*)=0$, $H_k=0$ on $U$, and on $\R^D$ by analytic continuation.
So $H=O(t^m)$ at every point, and $q'-cq=cq(e^H-1)$ has no coefficient
below degree $m$ anywhere on $\R^D$. This survives composition with an
analytic curve: $(q'-cq)(t,x(t))=O(t^m)$.

\emph{Localizing on a far arm.} Fix a matched ray $j$. Let $\mathcal E$
be an analytic map on the direction chart that sends a unit vector $a$
to an orthonormal basis of $a^\perp$, and put $E(t)=\mathcal E(a_j(t))$
and $E'(t)=\mathcal E(a_j'(t))$, so that $E(0)=E'(0)$. Evaluate both
intensities at $x(t)=p(t)+Ra_j(t)+E(t)y$, with $R>0$ and $y\in\R^{D-1}$
bounded. By \eqref{eq:rayformula}, the matched unprimed ray contributes
$w_j(t)f_j(t,y)$, where $f_j(t,\cdot)=f_{a_j(t),\eta(t)}$, times its
end factor $\Phi$. The matched primed ray is centered on its own line:
its transverse argument is $E'(t)^\top\{x(t)-p'(t)\}=A(t)y+d(t)+Re(t)$,
with
\[
 A=E'^\top E,\qquad d=E'^\top(p-p'),\qquad e=E'^\top a_j,
\]
so it contributes $w_j'(t)f_j'(t,A(t)y+d(t)+Re(t))$ times its end
factor. Because $e(0)=0$, the $k$-th Taylor coefficient of this term is
a polynomial in $R$ of degree at most $k$, whose coefficients involve
derivatives of $f_j'$ at the $R$-free point $(0,A(0)y+d(0))$. Every
other oriented ray of either model, the opposite one included, has a
positive constrained rate in \eqref{eq:orientedrate} at $a_j(0)$, and
both end factors differ from one by Gaussian tails in $R$.
Differentiating in the parameters only adds polynomial factors in $R$
and $r$, and bounded mixture shifts do not change the leading rate.
Hence, for each order $k$, the $k$-th coefficient of
$(q'-cq)(t,x(t))$ equals the $k$-th coefficient of
\begin{equation}\label{eq:farpair}
 \Delta_R(t,y)=w_j'(t)\,f_j'\bigl(t,A(t)y+d(t)+Re(t)\bigr)
 -c(t)\,w_j(t)\,f_j(t,y)
\end{equation}
up to $O(R^Ne^{-c_0R^2})$, uniformly for bounded $y$. The constants depend
on $k$ and on the chart, not on a lower bound for the slot weights. By
the first step, every coefficient of $\Delta_R$ below degree $m$ is
$O(R^Ne^{-c_0R^2})$.

\emph{Directions.} Suppose that $a_j'(t)-a_j(t)=t^rh+O(t^{r+1})$ with
$h\ne0$ and $r<m$. Differentiating $|a_j'|^2=|a_j|^2=1$ gives
$h\perp a_j(0)$, so $u=E(0)^\top h\ne0$. Since $E'^\top a_j'=0$,
\[
 e(t)=E'(t)^\top\{a_j(t)-a_j'(t)\}=-t^ru+O(t^{r+1}).
\]
So $Re(t)$ first enters at order $r$, and there only linearly. In the
order-$r$ coefficient of $\Delta_R$ the only part that depends on $R$ is
\[
 -R\,w_j'(0)\,u^\top\nabla f_j'\bigl(0,A(0)y+d(0)\bigr);
\]
the rest of that coefficient, including all of $c\,w_jf_j$, is free of
$R$. Divide the order-$r$ coefficient by $R$ and let $R\to\infty$. As
$w_j'(0)>0$, the derivative $u^\top\nabla f_j'(0,\cdot)$ vanishes at
$A(0)y+d(0)$ for every $y$, and $y\mapsto A(0)y+d(0)$ is onto because
$A(0)=I$. So $f_j'(0,\cdot)$ is constant along every line in the
direction $u$. A positive integrable density cannot be: its integral
over $\R^{D-1}$ would be infinite. Hence $h=0$, a contradiction, and
$a_j'-a_j=O(t^m)$.

\emph{Weights and offsets.} Now $E'-E=\mathcal E(a_j')-\mathcal E(a_j)
=O(t^m)$, so $A=I+O(t^m)$ and $e=E'^\top(a_j-a_j')=O(t^m)$. Below degree
$m$ the coefficients of \eqref{eq:farpair} no longer depend on $R$, and
letting $R\to\infty$ gives, for every $y\in\R^{D-1}$,
\[
 w_j'(t)\,f_j'\bigl(t,y+d(t)\bigr)=c(t)\,w_j(t)\,f_j(t,y)+O(t^m).
\]
Each fixed $t$-derivative of either side is bounded by a polynomial in
$y$ times a Gaussian, uniformly for small $t$, so these coefficient
identities may be integrated in $y$ against $1$ and against $y$. Both
profiles integrate to one, so $w_j'-cw_j=O(t^m)$. Summing over $j$ with
$\sum_jw_j=\sum_jw_j'=1$ gives $c=1+O(t^m)$, and then
$w_j'-w_j=O(t^m)$. Both profiles have first moment zero because the
mixture is centered, so the first moment of the left side is
$-w_j'(t)d(t)$ and that of the right side vanishes; with $w_j'(0)>0$
this gives $d=O(t^m)$. The transverse offsets are $EE^\top p$ and
$E'E'^\top p'$, and
\[
 E'E'^\top p'-EE^\top p=-E'd+(E'E'^\top-EE^\top)p=O(t^m)
\]
because the vertices are bounded.
\end{proof}

\begin{proof}[Proof of \eqref{eq:inverse}]
Near any pair with equal scores the ray matching is fixed. The squared
uncapped target distance $f$ and the squared score distance
$Q=\int_U|g-g'|^2$ are analytic in the chart parameters, and
Lemma~\ref{lem:identification} gives the identification hypothesis of
Proposition~\ref{prop:ordercriterion}. Lemma~\ref{lem:mixturerigidity}
gives its order condition: along an arc whose scores differ at order $m$,
the targets differ at order at least $m$. Proposition~\ref{prop:ordercriterion}
therefore bounds the target distance by a multiple of the score distance
near every equal-score pair. Pairs with different counts are positively
separated in score, and away from equal-score pairs the ratio of the
capped distance to the score distance is bounded by compactness. This
proves \eqref{eq:inverse}.
\end{proof}

\subsection{From scores to cell probabilities}\label{app:pixels}

\begin{lemma}[Derivative domination]\label{lem:derivatives}
Let $\mathcal Z$ be compact semianalytic and $F(z,x)$ jointly
real-analytic near $\mathcal Z\times U$. For each spatial multi-index
$\alpha$ there is $C_\alpha<\infty$ with
$\norm{\partial_x^\alpha F(z,\cdot)}_{L^2(U)}\le
C_\alpha\norm{F(z,\cdot)}_{L^2(U)}$ for all $z\in\mathcal Z$.
\end{lemma}

\begin{proof}
The squared norms $Q_0(z)$ and $Q_\alpha(z)$ are analytic. If
$Q_\alpha/Q_0$ were unbounded, curve selection in
$\{Q_\alpha>0,\ s>0,\ Q_0<s^2Q_\alpha\}$ would give an analytic arc
ending at $s=0$. Along it, the first nonzero Taylor coefficient of $F$
has some order $m$, and spatial differentiation cannot create a
nonzero coefficient of lower order. Hence $\ord Q_0=2m$ and
$\ord Q_\alpha\ge2m$, which is incompatible with $Q_0<s^2Q_\alpha$. If
all coefficients vanish, $F$ is zero along the arc, contrary to
$Q_\alpha>0$.
\end{proof}

\begin{proof}[Proof of \eqref{eq:bininverse}]
Apply the lemma to differences of normalized densities on each pair of
count strata. Let $F=\pi-\pi'$, $m=\inf\pi>0$ and $G=\sup|\nabla\pi|<\infty$
over the compact class and window. From
$g-g'=\nabla F/\pi-F\nabla\pi'/(\pi\pi')$ and \eqref{eq:inverse},
\[
 d_T\le C_g\bigl(D_1/m+G/m^2\bigr)\norm F_2,
\]
where $D_1$ is the domination constant for first derivatives. Choose an
integer $s>D/2+1$. Sobolev embedding on the cube and the lemma for all
derivatives up to order $s$ give $\norm{\nabla F}_\infty\le B\norm F_2$
with one finite $B$. For a partition into equal cubes of volume $v$ and
diameter at most $h$, the cell-average projection $\Pi_hF$ satisfies
$\norm{F-\Pi_hF}_2\le h\sqrt{|U|}\,B\norm F_2$. Take $h$ with
$h\sqrt{|U|}B\le1/2$; then $\norm F_2\le2\norm{\Pi_hF}_2
=2v^{-1/2}\norm{P-P'}_2$. One partition and one constant serve all the
finitely many pairs of strata. The analytic extension is that of $q/Z$
across the fixed window; no window indicator is differentiated.
\end{proof}

\subsection{Proof of Lemma~\ref{lem:invisible} and Proposition~\ref{prop:classification}}\label{app:classification}

\begin{proof}[Proof of Lemma~\ref{lem:invisible}]
The growth bound makes $\mu$ a tempered distribution, and
$q=\ph_\Sigma*\mu$ and $q'=\ph_{\Sigma+H}*\mu$ are smooth functions of
polynomial growth with $\widehat q=e^{-\xi^\top\Sigma\xi/2}\widehat\mu$ and
$\widehat{q'}=e^{-\xi^\top(\Sigma+H)\xi/2}\widehat\mu$. Both multipliers are
Schwartz functions because $\Sigma$ and $\Sigma+H$ are positive definite.
Hence $q=q'$ if and only if
$e^{-\xi^\top\Sigma\xi/2}\bigl(1-e^{-\xi^\top H\xi/2}\bigr)\widehat\mu=0$. Write
$1-e^{-u/2}=u\,m(u)$ with $m(u)=(1-e^{-u/2})/u$ for $u\ne0$ and $m(0)=1/2$;
$m$ is smooth and positive on $\R$. Tested against compactly supported
smooth functions, the identity may be divided by the smooth positive
function $e^{-\xi^\top\Sigma\xi/2}m(\xi^\top H\xi)$, so it is equivalent to
$(\xi^\top H\xi)\widehat\mu=0$ as a distribution. The left side is a
tempered distribution, and a tempered distribution that vanishes on
compactly supported test functions vanishes, so the condition is
$(\xi^\top H\xi)\widehat\mu=0$ in the tempered sense, which is the Fourier
transform of $-H:\nabla^2\mu=0$. No exponentially growing function ever
multiplies a tempered distribution in this argument. Conversely, if
$H:\nabla^2\mu=0$, every $\Sigma+\tau H$ with $\tau\in[0,1]$ is positive
definite, and
$\partial_\tau(\ph_{\Sigma+\tau H}*\mu)=\tfrac12(H:\nabla^2\ph_{\Sigma+\tau H})*\mu
=\tfrac12\ph_{\Sigma+\tau H}*(H:\nabla^2\mu)=0$, by the heat equation of
the Gaussian in its covariance.
\end{proof}

\begin{proof}[Proof of Proposition~\ref{prop:classification}]
The intensities grow at most polynomially, so they have Fourier
transforms as tempered distributions. Lemma~\ref{lem:identification}
already identifies the directions, weights and transverse offsets, and
two nonparallel supporting lines fix the vertex. For one Gaussian the far
profile of a ray is $w\,\ph_{E^\top\Sigma E}(E^\top(x-p))$.

\emph{One line.} On a single supporting line, write
$\Sigma=\left(\begin{smallmatrix}s&c^\top\\c&C\end{smallmatrix}\right)$
in longitudinal and normal coordinates, $\ell=C^{-1}c$ and
$v=s-c^\top C^{-1}c>0$. If the two opposite weights are $w_+$ and $w_-$
(zero for a missing arm), Gaussian conditioning gives
\begin{equation}\label{eq:linefactor}
 q(t,z)=\ph_C(z-p_\perp)\bigl[w_-+(w_+-w_-)\,
 \Phi\{(t-p_\parallel-\ell^\top(z-p_\perp))/\sqrt v\}\bigr].
\end{equation}
If $w_+\ne w_-$, the variation in $t$ and $z$ identifies $\ell$,
$p_\parallel$ and $v$ once the normal parameters are known. If
$w_+=w_->0$ the line is a constant full line: $p_\parallel$ does not
enter, and only the normal block $C$ of the covariance matters.

\emph{Which $H$ are possible.} With the source fixed, Lemma~\ref{lem:invisible} reduces the question to $H:\nabla^2\mu=0$. In a tube around one supporting line, away
from the vertex, arbitrary normal test functions force
$P_{a^\perp}HP_{a^\perp}=0$. So the quadratic form $\xi^\top H\xi$
vanishes on the hyperplane $a^\perp$ and is divisible by its defining
linear form. Three distinct supporting directions give three distinct
linear factors, which forces $H=0$. For two nonparallel directions $a,b$
the only possibility is $H=t(ab^\top+ba^\top)$. Writing
$\Delta_a,\Delta_b$ for the positive-minus-negative weight of each line,
\[
 H:\nabla^2\mu=2t\,(\Delta_aD_b+\Delta_bD_a)\,\delta_p
\]
as distributions, and linear independence forces $\Delta_a=\Delta_b=0$
whenever $t\ne0$: both lines are balanced. Conversely two balanced lines
annihilate this derivative. For one balanced line, a vanishing normal
block suffices and leaves $D$ covariance directions free. Unbalanced
lines were already identified through \eqref{eq:linefactor}.
\end{proof}

\subsection{Proof of Theorem~\ref{thm:statistics}}\label{app:minimax}

\emph{Upper bound.} The image of $\Ccal$ in cell probabilities is
compact, so it has a finite Euclidean $1/n$-net. Fix a parameter
representative for each net point and a deterministic order, and return
the representative nearest to the empirical frequencies $\widehat P$.
This rule is measurable, and
$\norm{P_{\widehat\xi}-P_0}_2\le2\norm{\widehat P-P_0}_2+1/n$. For
multinomial frequencies,
$\E\norm{\widehat P-P_0}_2^2=(1-\norm{P_0}_2^2)/n\le1/n$. Jensen's
inequality and \eqref{eq:bininverse} give \eqref{eq:risk}.

\emph{Lower bound.} Take one ray with vertex $0$, weight one, an
admissible isotropic Gaussian, and direction
$a(t)=(\cos t,\sin t,0,\dots,0)$. This submodel lies in every nonempty
class, since the separation condition is empty for one ray. Its signature
distance from $t=0$ is $2\sin(|t|/2)$, of order $|t|$. On the compact
window $\pi_t$ is positive with a bounded parameter derivative, so
$\log u\le u-1$ gives $\KL(\pi_0\Vert\pi_t)\le Ct^2$. With $t=b/\sqrt n$
and a small fixed $b$, Pinsker's inequality bounds the total variation of
the $n$-fold products by $1/4$, and the two-point testing inequality
\citep{tsybakov09} gives worst-case risk at least $c/\sqrt n$, even for
exact positions. Counts are a function of the positions and cannot do
better.

\emph{Conditioning on the total.} With exposure $A$ unrestricted, put
$\xi=AZ$. Up to factors that do not involve the model, the Poisson
likelihood is $e^{-\xi}\xi^N\prod_{i\le N}\pi(X_i)$, and every $\xi>0$ is
allowed for every shape. Maximizing over $\xi$ gives $\widehat\xi=N$ when
$N>0$ and leaves exactly the conditional likelihood of the positions; for
$N=0$ the supremum is approached as $A\downarrow0$ and carries no shape
information. This equivalence uses the free exposure and does not hold
under a calibrated or shared exposure. \hfill$\square$

\subsection{Proof of Proposition~\ref{prop:crossing}}\label{app:crossinglower}
Put $H=ab^\top+ba^\top$,
$m_\Sigma=\min\{\lambda_{\min}(\Sigma)-\lambda,\ \Lambda-\lambda_{\max}(\Sigma)\}>0$
and $h_0=m_\Sigma/\{2(1+|a^\top b|)\}$. Since
$\norm H_{\rm op}=1+|a^\top b|$, every $\Sigma+hH$ with $|h|\le h_0$ is
admissible. Let $R_t$ rotate the plane of $a,b$ by the angle $t$ and fix
its orthogonal complement, and let the model $(t,h)$ have the four
directions rotated by $R_t$, the same vertex and weights, and covariance
$\Sigma_{t,h}=R_t(\Sigma+hH)R_t^\top$. Rotation keeps the weights and the
separations. For each fixed $t$, Proposition~\ref{prop:classification}
makes the law independent of $h$.

\emph{Covariance.} At $t=0$ the laws for $h=\pm h_0$ are equal, so any
estimator has the same distribution under both, and the triangle
inequality gives maximum Frobenius risk at least
$\tfrac12\norm{2h_0H}_{\rm F}=h_0\norm H_{\rm F}$, with
$\norm H_{\rm F}=\sqrt{2+2(a^\top b)^2}$.

\emph{Geometry.} Put $A_*=(\sqrt DR_0+|p|)^2/\lambda$. The intensity of a
balanced line is a transverse Gaussian whose rotated precision has
operator norm at most $1/\lambda$; differentiating in $t$ gives
$|\partial_t\log q_t(x)|\le A_*$ on $U$, and the positive sum of two
lines inherits the bound. After window normalization the parameter score
$\partial_t\log\pi_t$ is $\partial_t\log q_t$ minus its mean under
$\pi_t$, so its $L^2(\pi_t)$ norm is at most $A_*$ and
$\norm{\partial_t\sqrt{\pi_t}}_2\le A_*/2$. The Hellinger distance between
$\pi_0$ and $\pi_t$ is then at most $A_*|t|/2$, and the product affinity
gives
\[
 \TV\bigl(\pi_0^{\otimes n},\pi_t^{\otimes n}\bigr)\le\sqrt n\,A_*|t|/2.
\]
Choose $t_*>0$ so small that the identity matching of the four
directions is optimal and $4\sin(t_*/2)<1$, and put
$t_n=\min\{t_*,1/(2A_*\sqrt n)\}$. The models $(0,-h_0)$ and
$(t_n,+h_0)$ then have product laws within total variation $1/4$, and
their four directions alone give $d_T\ge4\sin(t_n/2)\ge4t_n/\pi$. The
two-point inequality gives
\[
 \inf_{\widehat S}\max_{i=0,1}\E_id_T(\widehat S,S_i)\ge\frac{3t_n}{4\pi},
\]
which for large $n$ equals
$3\lambda/\{8\pi(\sqrt DR_0+|p|)^2\sqrt n\}$. The two alternatives
approach different covariance representatives, $\Sigma-h_0H$ and
$\Sigma+h_0H$; no pointwise lower bound for a single known source is
claimed. \hfill$\square$

\subsection{Four quadrants: proof of Theorem~\ref{thm:fourcells}}\label{app:fourcells}
\paragraph{Setting.}
Rays along $e_1,-e_1,e_2,-e_2$ carry weights
$(1/4+\tau/2,\,1/4-\tau/2,\,1/4,\,1/4)$, with $0\le\tau\le\bar\tau<1/2$,
under one Gaussian with unit marginal variances and correlation
$|\rho|\le r\le0.8$, $r>0$. We observe counts in the four quadrants of
$[-3,3]^2$. With $\ph,\Phi$ the standard normal density and distribution
function,
\begin{equation}\label{eq:calibratedq}
 q_{\tau,\rho}(x,y)=\frac{\ph(x)+\ph(y)}4
 +\tau\ph(y)\Bigl[\Phi\Bigl(\frac{x-\rho y}{\sqrt{1-\rho^2}}\Bigr)-\frac12\Bigr],
 \qquad Z=3\{2\Phi(3)-1\}.
\end{equation}
Let
\[
 d(\rho)=Z^{-1}\int_0^3\!\!\int_0^3\ph(y)
 \Bigl[\Phi\Bigl(\frac{x-\rho y}{\sqrt{1-\rho^2}}\Bigr)-\frac12\Bigr]\dd x\dd y,
\]
$a(\rho)=2\{d(\rho)+d(-\rho)\}$ and $b(\rho)=2\{d(\rho)-d(-\rho)\}$.
For signs $s_x,s_y\in\{-1,1\}$ the cell law is
\begin{equation}\label{eq:quadrants}
 P_{s_x,s_y}=\tfrac14\bigl[1+\tau\{a(\rho)s_x+b(\rho)s_y\}\bigr],
 \qquad \mu:=\E(s_x,s_y)=\tau(a,b).
\end{equation}
The balanced term integrates to $Z$; central reflection changes the sign
of the imbalance term, whose integral over the square vanishes. The
$(+,+)$ cell carries $\tau d(\rho)$, the $(+,-)$ cell $\tau d(-\rho)$ after
$y\mapsto-y$, and reflection gives the other two.

\paragraph{Estimator.}
Let $s_*(\rho)=-b(\rho)/a(\rho)$, which is odd and strictly increasing,
and $\mathcal K_r=\{v:v_1\ge0,\ |v_2|\le s_*(r)v_1\}$. Project the
empirical sign mean $\widehat\mu$ onto $\mathcal K_r$ to get $v$. If $v=0$
return $(\widehat\tau,\widehat\rho)=(0,0)$; otherwise
\begin{equation}\label{eq:estimator}
 \widehat\rho=s_*^{-1}(-v_2/v_1),\qquad
 \widehat\tau=\clip_{[0,\bar\tau]}\{v_1/a(\widehat\rho)\}.
\end{equation}
The inverse is a bracketed scalar solve. With no photons the procedure
returns the whole parameter set.

\paragraph{Bounds on the calibration functions.}
Let $s=\sqrt{1-\rho^2}$ and $D_0=2\Phi(3)-1$. Direct differentiation
gives $\partial_\rho\{\ph(y)\Phi((x-\rho y)/s)\}=\partial_y\ph_{\Sigma_\rho}(x,y)$;
integrating first in $y$ over $[0,3]$ and then in $x$,
\begin{equation}\label{eq:dprime}
 d'(\rho)=\frac{\ph(3)[\Phi(3(1-\rho)/s)-\Phi(-3\rho/s)]
 -\ph(0)[\Phi(3/s)-1/2]}{Z}.
\end{equation}
Also $a(0)=D_0+\frac23\{\ph(3)-\ph(0)\}$ and $b(0)=0$. Put
$A_1=2\ph(3)/Z$, $B_1=2\ph(0)/Z$ and $\beta_1=\{2\ph(0)D_0-2\ph(3)\}/Z$.
The bracketed differences in \eqref{eq:dprime} lie in $[0,1]$, the sum of
the two upper boundary terms at $\rho$ and $-\rho$ is
$\ph(3)\{\Phi(3(1-\rho)/s)+\Phi(3(1+\rho)/s)-1\}\in[0,\ph(3)]$, and
$\Phi(3/s)-1/2\ge D_0/2$. Hence
\begin{equation}\label{eq:abintervals}
 |a'|\le A_1,\quad -B_1\le b'\le-\beta_1<0,\quad |b|\le rB_1,\quad
 a(0)-rA_1\le a\le a(0)+rA_1 ,
\end{equation}
with the worst case at $r=4/5$. Rational enclosures of $\pi$, $\ph(0)$,
$e^{-9/2}$ and $\Phi(3)$ (Machin's series for $\pi$, bisection for
$(2\pi)^{-1/2}$, a positive Taylor tail for the exponential and an
alternating integrated series for $\Phi(3)$) give, on $|\rho|\le4/5$,
\begin{equation}\label{eq:calibratedbounds}
 .731<a<.737,\qquad|a'|<.003,\qquad|b|<.214,\qquad-b'>.262,
\end{equation}
as strict inequalities between terminating rationals, not floating-point
samples.

\paragraph{Proof of Theorem~\ref{thm:fourcells}.}
Since $a$ is even and $b$ odd, $s_*$ is odd, with $|s_*|<.214/.731<.293$
and $s_*'=(-b'a+ba')/a^2>\{(.731)(.262)-(.214)(.003)\}/(.737)^2>.350$, so
the inverse in \eqref{eq:estimator} exists on the whole cone. For $v\ne0$
put $z=-v_2/v_1$, $\rho_v=s_*^{-1}(z)$ and $F(v)=v_1/a(\rho_v)$. Then
$\nabla F=\bigl(1/a+a'z/(a^2s_*'),\ a'/(a^2s_*')\bigr)$ and
$\norm{\nabla F}_2\le1/.731+(.003)\sqrt{1+.293^2}/\{(.731)^2(.350)\}<1.4$.
The cone is convex, so integrating along segments and extending
continuously to the vertex shows that $F$ is $1.4$-Lipschitz; $F$ is
nonnegative and positively homogeneous, and projection onto the cone and
clipping to $[0,\bar\tau]$ do not increase the error against a true point
in those sets. For $\mu=\tau(a(\rho),b(\rho))$ and $v\ne0$, the identity
$\mu_1\{s_*(\rho_v)-s_*(\rho)\}=(\mu_2-v_2)+s_*(\rho_v)(\mu_1-v_1)$ (up to
a common sign), with $s_*'>.350$ and $\mu_1\ge.731\tau$, gives
$\tau|\rho_v-\rho|\le\sqrt{1+.293^2}\norm{v-\mu}_2/\{(.731)(.350)\}
<4.1\norm{v-\mu}_2$; at $v=0$ the same follows from
$|\rho|\le|s_*(\rho)|/.350$ and $|\mu_2|=\mu_1|s_*(\rho)|$. The sign
observations have squared norm two, so
$\E\norm{\widehat\mu-\mu}_2^2=(2-\norm\mu_2^2)/n\le2/n$; projection, the
two deterministic bounds and Jensen's inequality give
\eqref{eq:fourrisk}--\eqref{eq:rhorisk}. For two calibrated models the
sign matrix $S$ of the four cells has $S^\top S=4I_2$, so
$\norm{\mu-\mu'}_2=2\norm{P-P'}_2$; the signature changes only in the two
unequal weights, so $d_T=|\tau-\tau'|/\sqrt2$ and
\begin{equation}\label{eq:fourinverse}
 d_T\le1.4\sqrt2\,\norm{P-P'}_2 .
\end{equation}
A partition refining the quadrants inherits the estimator by summation;
for an aligned $16\times16$ grid the summation matrix has operator norm
eight, giving the constant $8(1.4\sqrt2)<15.84$ in this family. This does
not cover general geometry: on a generic fixed-$K$ class with two
nonparallel rays the signature has local dimension
$K(D-1)+(K-1)+D=D(K+1)-1$, a Lipschitz inverse from a differentiable cell
map needs rank at least that, and $M$ cells give rank at most $M-1$, so
$M\ge D(K+1)$ is necessary (for one ray, $M\ge2D-1$). Four cells cannot
resolve generic planar four-arm geometry.

\paragraph{Lower bounds near balance.}
With $\pi_0=(\ph(x)+\ph(y))/(4Z)$ on the square,
\begin{equation}\label{eq:localdensity}
 \pi_{\tau,\rho}=\pi_0(1+\tau h_\rho),\qquad
 h_\rho=\frac{4\ph(y)}{\ph(x)+\ph(y)}
 \Bigl[\Phi\Bigl(\frac{x-\rho y}{\sqrt{1-\rho^2}}\Bigr)-\frac12\Bigr],
\end{equation}
with $\E_0h_\rho=0$, $|h_\rho|\le2$ and $1+\tau h_\rho\ge\kappa_0=1-2\bar\tau>0$.
From $\log u\le u-1$,
\begin{equation}\label{eq:klbound}
 \KL(\pi_{\tau,\rho}\Vert\pi_{\tau',\rho'})\le\kappa_0^{-1}
 \E_0(\tau h_\rho-\tau'h_{\rho'})^2 .
\end{equation}
With $c_0=\sqrt{\kappa_0}/(12\sqrt2)$ and large $n$, compare
$(c_0/\sqrt n,-r)$ with $(2c_0/\sqrt n,r)$: the product
Kullback--Leibler divergence is at most $1/8$ and the total variation at
most $1/4$, so the imbalance risk is at least $3c_0/(16\sqrt n)$ and, on
this pair, the correlation risk at least $3r/8$. For a known imbalance
$a>0$, $|\partial_\rho h_\rho|\le H_*=4\ph(0)\,3(1+r)/(1-r^2)^{3/2}$, and
two correlations separated by
$\Delta_n(a)=\min\{2r,\sqrt{\kappa_0}/(\sqrt8H_*a\sqrt n)\}$ give risk at
least $3\Delta_n(a)/16$. With \eqref{eq:rhorisk} this proves the two
minimax rates. At $\tau=0$ all correlations give one law, and the two
endpoints force risk at least $r$, attained by the estimate zero.

\paragraph{Confidence sets.}
With $e_{n,\alpha}=2\sqrt{\log(4/\alpha)/n}$, Hoeffding's inequality and a
union bound give $\Pr\{\norm{\widehat\mu-\mu}_2>e_{n,\alpha}\}\le\alpha$, so
\begin{equation}\label{eq:confidence}
 \bigl\{(\tau,\rho):\norm{\tau(a(\rho),b(\rho))-\widehat\mu}_2\le e_{n,\alpha}\bigr\}
\end{equation}
covers with probability at least $1-\alpha$ for every parameter,
including $\tau=0$, where it may contain the whole correlation interval.
For fixed $\rho$, the admissible $\tau$ solve
$\norm{m_\rho}_2^2\tau^2-2(\widehat\mu^\top m_\rho)\tau+\norm{\widehat\mu}_2^2
-e_{n,\alpha}^2\le0$ with $m_\rho=(a(\rho),b(\rho))$.

\subsection{Counting arms}\label{app:selection}

\begin{proof}[Proof of Proposition~\ref{prop:count}]
Write $\ell_n(P)=-\sum_i\widehat P_i\log P_i$, let $P_0$ be the true cell
law, $e=\norm{\widehat P-P_0}_2$ and $d=\norm{P-P_0}_2$. Then
\[
 \ell_n(P)-\ell_n(P_0)=\KL(P_0\Vert P)
 -\langle\widehat P-P_0,\log P-\log P_0\rangle
 \ge\tfrac12d^2-ed/a_*,
\]
by Pinsker's inequality and because $\log$ is $1/a_*$-Lipschitz on
$[a_*,1]$. For a wrong count $d\ge\Delta$. On the event
$e\le a_*\Delta/4$ the right side is at least $d^2/4\ge\Delta^2/4$, which
exceeds the penalty range plus $\epsilon_n$, so no wrong count is within
$\epsilon_n$ of the infimum. The frequencies satisfy $\E e\le n^{-1/2}$;
replacing one observation changes $e$ by at most $\sqrt2/n$, and the
bounded-differences inequality gives $\Pr(e>n^{-1/2}+u)\le e^{-nu^2}$,
which proves the count bound. When the count is right the penalty cancels
against $P_0$, so $d^2/2\le ed/a_*+\epsilon_n$ and
$d\le2e/a_*+\sqrt{2\epsilon_n}$. Theorem~\ref{thm:inverse} on this event,
loss at most one on its complement, and $\E e\le n^{-1/2}$ give the risk
bound. For penalties $d_K\log n/(2n)$ the range tends to zero and the
count error is eventually exponentially small.
\end{proof}

This argument uses only the separation of the images, not a regular
Laplace approximation, which can fail for singular models \citep{drton17}.
The executed fit uses the BIC of \citet{schwarz78} on positions, with the
nominal parameter count; the proposition explains why a vanishing
penalty works, while the multistart search supplies no bound on
$\epsilon_n$ and the executed window is a disk, not the partition of
Theorem~\ref{thm:inverse}.

\paragraph{The agreement rule.}
The executed count $G$ equals the BIC count $B$ unless the Hessian ridge
detector and the line fit return one count $H=L<B$, in which case $G$ is
the likelihood fit with that count. For any target count $K_*$,
\begin{equation}\label{eq:guardcontain}
 \{G\ne K_*\}\subseteq\{B\ne K_*\}\cup\{H=L<K_*\}.
\end{equation}
The two detectors need not be independent, but a shared undercount
cannot be ignored. The line fit counts peaks of an angular histogram of
localizations in the annulus from $0.35R$ to $R$ about the window
center: 90 bins, smoothed by a fixed nonnegative kernel $W$ with unit
sum, a local maximum $i$ accepted when
$z_i=h_i-0.25\max_jh_j-0.75\,\mathrm{med}_jh_j>0$, then peaks within
$25^\circ$ of a higher one suppressed. Let $b_i$ be the population bin
probabilities, $h=Wb$, and $I$ the accepted maxima before suppression.
Suppose $|I|\ge K_*$, the peaks in $I$ are strict and more than
$25^\circ$ apart, and
\begin{equation}\label{eq:histmargin}
 m=\min\Bigl\{\min_{i\in I}\{h_i-h_{i-1},h_i-h_{i+1},z_i\},\
 \min_{i\notin I}\max\{h_{i-1}-h_i,h_{i+1}-h_i,-z_i\}\Bigr\}>0,
\end{equation}
with circular bin indices.

\begin{proposition}[Agreement rule under a count margin]\label{prop:guardmargin}
From $n$ independent positions, \eqref{eq:histmargin} gives
$\Pr(L<K_*)\le180\exp(-nm^2/8)$. Suppose also that the event densities of
each count form a compact class with densities bounded below, and that
the observation law $P$ has a population log-likelihood margin
\[
 \sup_{q\in\mathcal Q_{K_*}}P\log q-\max_{k\ne K_*}\sup_{q\in\mathcal Q_k}P\log q
 \ge\gamma>0 .
\]
If every saved per-count fit is within $a_n=o(n)$ of its global empirical
optimum and the penalty range is $o(n)$, then for all large $n$,
$\Pr(G\ne K_*)\le C_\gamma e^{-c_\gamma n}+180e^{-nm^2/8}$, and the risk of
any loss bounded by one exceeds that of the saved $K_*$ fit by at most
this amount.
\end{proposition}

\begin{proof}
Hoeffding's inequality and a union bound over the 90 bins give
$\Pr(\norm{\widehat b-b}_\infty>u)\le180e^{-2nu^2}$. Smoothing with $W$
does not increase the sup norm, and each comparison in
\eqref{eq:histmargin}, $z_i$ included, moves by at most $2u$. With
$u=m/4$ every accepted and rejected peak keeps its status, and the
separation makes suppression irrelevant; this is the first bound.
Compactness and the density floor give finite sup-norm nets for $\log q$;
Hoeffding's inequality on a fixed net and a fine enough net remainder give
$\Pr\{\sup_q|(P_n-P)\log q|>\gamma/4\}\le C_\gamma e^{-c_\gamma n}$. On the
complement a saved fit with the true count beats every other count by at
least $n\gamma/2-a_n$ before penalties, so $B=K_*$ for large $n$. Apply
\eqref{eq:guardcontain}; on $\{G=K_*\}$ the output is the saved $K_*$ fit.
\end{proof}

The margin cannot be dropped under arbitrary mismatch. Take an admissible
three-ray density $p_3$ and a separated fourth-ray response $r_4$ with
the same vertex and blur. The law
\begin{equation}\label{eq:samelawcount}
 P=(1-\varepsilon)p_3+\varepsilon r_4
\end{equation}
is a three-arm junction with an $\varepsilon$ share of structured
contamination, and it is also a four-arm junction whose fourth arm holds
an $\varepsilon$ share of the window, with weights inside the declared
floors. No observation confined to the window separates the two
descriptions at any sample size, so one of them has count error at least
$1/2$. Localizations beyond the window can separate them when a true arm
continues; Appendix~\ref{app:outcomes} tests exactly this.

\subsection{The two-versus-three-arm pair}\label{app:resolution}
Angles are in radians, $a(\theta)=(\cos\theta,\sin\theta)$, and every
decimal below is an exact terminating rational. Model A has two rays,
model B three:
\[
\begin{array}{c|c|c}
 &\text{A}&\text{B}\\ \hline
 (\theta_j)&(.002650992803,\ .449766929644)&(0,\ .251,\ .502)\\
 (w_j)&(.936667614136,\ .063332385864)&(11/12,\ 1/24,\ 1/24)\\
 p&(-.000480863669,\ .000704779337)&(0,0)
\end{array}
\]
Model B has covariance $(97/50)I$; model A has covariance
$\left(\begin{smallmatrix}1.939479076894&.000257587554\\
.000257587554&1.940885106126\end{smallmatrix}\right)$. Both blurs are
single centered Gaussians inside a two-slot class with the second slot
inactive. The class has $D=2$, $K\le3$, at most two slots, arm weights at
least $.04$, direction chords at least $.25$, $|p|\le.5$, slot means of
norm at most $.3$, spectra in $[.36,1.96]$, window $[-3,3]^2$ and $L=1$.
The shortest chord of B is $2\sin(.251/2)>.25$ and that of A exceeds
$.25$; the weights, principal minors and vertex bounds satisfy the
remaining inequalities directly. The counts differ, so $d_T=1$.

\paragraph{Integration.}
In a ray frame $(a,b)$ put $C=b^\top\Sigma b$, $c=a^\top\Sigma b$,
$V=\det\Sigma/C$, $z=x-p$ and $t=(a^\top z-c\,b^\top z/C)/\sqrt V$. The
intensity of one ray and its gradient are
\begin{equation}\label{eq:hardprofile}
 q_a=w\,\phi_C(b^\top z)\Phi(t),\qquad
 \nabla q_a=w\,\phi_C(b^\top z)\Bigl[-\frac{b^\top z}{C}\,b\,\Phi(t)
 +\phi(t)\frac{a-(c/C)b}{\sqrt V}\Bigr].
\end{equation}
Summing gives $q_A$ and $q_B$, each divided by its own $Z=\int_Uq$. With
integration error included,
\begin{equation}\label{eq:hardbounds}
\begin{gathered}
 \int_U\frac{(\pi_A-\pi_B)^2}{\pi_B}<1.556851\times10^{-8},\qquad
 \norm{g_A-g_B}_{L^2(U)}^2<8.324915\times10^{-6},\\
 \norm{P_A-P_B}_2<8.288917\times10^{-6},
\end{gathered}
\end{equation}
the last for $16\times16$ equal cells. An independent floating-point
quadrature (512 intervals per side, eight-point Gauss rule) gives
$1.5568505\times10^{-8}$, $8.3249123\times10^{-6}$ and
$8.2889160\times10^{-6}$ for the three quantities, just below the
enclosures. Each side of $U$ is divided into
256 intervals and tensor eight-point Gauss quadrature is applied on each
small square. For an integrand analytic on every one-coordinate complex
circle of radius $\rho=1/8$ with modulus at most $M$, the total error is
at most $2|U|\{(6/256)/\rho\}^{16}(8!)^4M/\{17\,(16!)^2\}$. Both
intensities are bounded below on the enlarged cube in $\R^2$,
$q_A>.00026119$ and $q_B>.00026336$. With $\lambda_0=1.93$, a
truncated-normal moment bound on the ray integral keeps their
one-coordinate complex continuations above $.7764$ times these bounds on
the circles. Gaussian integrals and Cauchy bounds then give $M<557.871$
for the chi-square integrand and $M<3.065\times10^7$ for the squared-score
integrand, so the quadrature errors are below $3.329\times10^{-17}$ and
$1.829\times10^{-12}$. The normalizers use the same rule with the bound on
derivatives of the entire Gaussian. The finite sums are evaluated with
outward rounding, and the trigonometric and Gaussian functions are
enclosed by rational Taylor remainders; the displayed bounds absorb both
arithmetic and integration error. The integration nodes are unrelated to
the 256 detector cells.

\paragraph{Consequences.}
Any score or pixel inverse on this class, applied to this pair, needs
$C_g>346.5$ and $C_{\rm pix}>120{,}643$ on this grid. For an arbitrary
partition, $v=P_A-P_B$ sums to zero, so $\norm v_2\le\norm v_1/\sqrt2$, and
data processing with Cauchy--Schwarz gives
$\norm v_1\le\norm{\pi_A-\pi_B}_1\le\sqrt{\chi^2(\pi_A\Vert\pi_B)}$; the
inverse constant therefore exceeds $\sqrt{2/(1.556851\times10^{-8})}>11{,}334$.
For $n$ positions the product chi-square divergence is
$(1+\chi^2)^n-1$, and the total variation of the product laws is at most
half its square root. Every count estimator therefore has maximum error at
least
\begin{equation}\label{eq:hardtesting}
 \frac12-\frac14\sqrt{\exp(1.556851\times10^{-8}\,n)-1},
\end{equation}
which exceeds $.46868$ at $n=10^6$. A separate Pinsker bound shows that
error at most $.05$ under both models needs
$n>1.62/(1.556851\times10^{-8})$, more than 104 million positions: a
necessary count, not a sufficient one. The pair changes both the
covariance and the vertex, so the bound does not transfer to an
estimator that is told either.

A second mechanism works with known, fixed blur. Let $f_\theta$ be the
window intensity of a single ray at angle $\theta$ under one fixed
Gaussian and a common vertex, and let
\[
 \nu_{4,u}=\tfrac14(\delta_{-3u}+\delta_{-u}+\delta_u+\delta_{3u}),\qquad
 \nu_{3,u}=\tfrac{25}{82}\bigl(\delta_{-\sqrt{41/5}\,u}+\delta_{\sqrt{41/5}\,u}\bigr)
 +\tfrac{16}{41}\delta_0 .
\]
Their angular moments agree through degree five and differ by
$(144/5)u^6$ in degree six, so a Taylor expansion of the analytic ray
image gives $\int f_\theta\dd(\nu_{4,u}-\nu_{3,u})(\theta)
=(u^6/25)\,\partial_\theta^6f_0+O(u^8)$. Window normalization keeps an
$O(u^6)$ difference, and on a fixed window with positive density the
chi-square divergence is $O(u^{12})$. This describes merging arms; as
$u\to0$ it leaves every fixed separation class.

\subsection{Proof of Proposition~\ref{prop:background}, the disk window and the fitted class}\label{app:background}
\begin{proof}
\emph{Analyticity.} $\kappa$ enters linearly, so $q^\kappa$ is jointly
real-analytic, and $q^\kappa\ge\min q>0$ on $\Ccal^\kappa\times U$.
\emph{Decay away from the rays.} By \eqref{eq:orientedrate}, for every
unit $v$ that is not an oriented direction of the models involved,
$q(Rv)\le Ce^{-cR^2}$, and every fixed parameter derivative has a
polynomial-times-Gaussian envelope; since the directions are finitely
many and $D\ge2$, such $v$ exist. \emph{Identification.} Equal scores
on an open set give $q'+\kappa'=c(q+\kappa)$ on $\R^D$ by analyticity;
evaluating at $Rv$ and letting $R\to\infty$ gives $\kappa'=c\kappa$, hence
$q'=cq$ and, by Lemma~\ref{lem:identification}, equal signatures.
\emph{Rigidity.} Along analytic arcs with score difference $O(t^m)$, the
first step of Lemma~\ref{lem:mixturerigidity} applied to
$\log\{(q'+\kappa')/(c(q+\kappa))\}$ shows that every Taylor coefficient of
$\Delta=q'+\kappa'-c(q+\kappa)$ below degree $m$ vanishes on $\R^D$.
Evaluating a coefficient at $Rv$ and letting $R\to\infty$ removes every
term that contains $q$ or $q'$, so the coefficients of $\kappa'-c\kappa$
vanish below degree $m$, and so do those of $q'-cq$; the rest of the
proof of Lemma~\ref{lem:mixturerigidity} applies unchanged.
\emph{Inverse and cells.} Proposition~\ref{prop:ordercriterion} and the
proof of \eqref{eq:inverse} use only these three facts and compactness,
and the cell step uses a positive lower bound
$\inf\pi^\kappa\ge\min q/\max Z^\kappa$, a bound on $|\nabla\pi^\kappa|$ and
Lemma~\ref{lem:derivatives}, all available on the compact analytic
family. The proofs of Theorem~\ref{thm:statistics} and
Proposition~\ref{prop:count} use only the cell inverse, compact images
with positive cell probabilities, and the one-ray submodel, which has
$\kappa=0$.
\end{proof}

The background in Proposition~\ref{prop:background} is one constant
level. A background that varies in space, even an affine one, would
need further positivity and growth arguments, and it is not covered
here.

\paragraph{The disk window.}
The proof of \eqref{eq:inverse} uses only that the window is compact
with nonempty interior, so the score inverse also holds on the disk $W$
of the executed fit ($D=2$). The cell step carries over with clipped
cells. Intersect a fixed grid of squares of side $h$ with $W$ and drop
the pieces of zero area. The remaining cells $B_i$ are convex, have
diameter at most $\sqrt2\,h$, and the smallest has area $v_{\min}>0$.
Lemma~\ref{lem:derivatives} holds with $U$ replaced by $W$, and Sobolev
embedding on the disk gives $\norm{\nabla F}_{L^\infty(W)}\le
B_W\norm F_{L^2(W)}$ for $F=\pi^\kappa-\pi'^\kappa$. On a convex cell, $F$
differs from its average by at most $\sqrt2\,h\,\norm{\nabla F}_\infty$,
so the cell-average projection satisfies $\norm{F-\Pi_hF}_{L^2(W)}\le
\sqrt2\,h\sqrt{|W|}\,B_W\norm F_{L^2(W)}$. Since
$\norm{\Pi_hF}_{L^2(W)}^2=\sum_i(P_i-P_i')^2/|B_i|$, taking
$\sqrt2\,h\sqrt{|W|}B_W\le1/2$ gives
\[
 \norm F_{L^2(W)}\le2\,v_{\min}^{-1/2}\norm{P_W-P_W'}_2,
\]
where $P_W$ lists the probabilities of the clipped cells. With the score
inverse on $W$ this is \eqref{eq:bininverse} on the disk, with
$v_{\min}$ in place of the cube volume. Cells cut thin by the boundary
make $v_{\min}$ small and the constant large; like the statement on
the cube, this one is existential.

\paragraph{The fitted class.}
The executed density \eqref{eq:fit} has background share $\omega_0$, arm
shares $\omega_j$ and $m_j=\int_Wh_j$, the window mass of a unit-intensity
ray under the fitted blur. With $u_j=\omega_j/m_j$ and $u=\sum_ju_j$,
\[
 \pi(x)\ \propto\ \frac{\omega_0}{|W|\,u}+\sum_{j=1}^K\frac{u_j}{u}\,h_j(x),
\]
which is $q^\kappa$ with ray weights $w_j=u_j/u$ summing to one, blur
$k_\eta$ with two centered isotropic components, and background level
$\kappa=\omega_0/(|W|u)$. On the constraint set (vertex in a box, widths in
a closed interval, directions at least $25^\circ$ apart) the masses $m_j$
are continuous and positive, so $0<m_-\le m_j\le m_+<\infty$. Then
$u\ge(1-\omega_0)/m_+\ge0.2/m_+$ because $\omega_0\le0.8$, which gives
$\kappa\le4m_+/|W|$; and $u\le\sum_j\omega_j/m_-\le1/m_-$, which gives
$w_j\ge(0.02/m_+)\,m_-=0.02\,m_-/m_+$. The
executed class therefore lies in a class $\Ccal^\kappa$ with
$J=2$, $w_*=0.02\,m_-/m_+$, $\delta=2\sin(12.5^\circ)$,
$\kappa_{\max}=4m_+/|W|$, $\bar p=283$\,nm (a box of $\pm200$\,nm) and
$\lambda,\Lambda$ the squared width bounds.
This places the fitted model inside the theorem; it does not bound the
error of the local search.

\subsection{Misfit, contamination and finite arms}\label{app:boundaries}
\begin{proof}[Proof of Proposition~\ref{prop:mismatch}]
The model $P_0$ is itself a candidate, so the fitted residual obeys
$\norm{P_{\widehat\xi}-\widehat P}_2\le\norm{P_0-\widehat P}_2+\varepsilon_{\rm opt}
\le\norm{P_0-P_*}_2+\norm{P_*-\widehat P}_2+\varepsilon_{\rm opt}$. Hence
\[
 \norm{P_{\widehat\xi}-P_0}_2\le\norm{P_{\widehat\xi}-\widehat P}_2
 +\norm{\widehat P-P_*}_2+\norm{P_*-P_0}_2
 \le2\norm{\widehat P-P_*}_2+2b+\varepsilon_{\rm opt},
\]
and \eqref{eq:bininverse} gives \eqref{eq:mismatch}. Under contamination,
$\norm{P_*-P_0}_2=\varepsilon\norm{Q-P_0}_2\le\sqrt2\,\varepsilon$, and the
frequencies of $n$ independent points drawn from $P_*$ satisfy
$\E\norm{\widehat P-P_*}_2\le n^{-1/2}$, so a global fit has expected loss
$O(n^{-1/2}+\varepsilon)$. The order $\varepsilon$ is necessary. Take two
rotated one-ray laws $p_0,p_t$ at target distance of order $|t|$ with
$\TV(p_0,p_t)\le C|t|$, and choose $t=c\varepsilon$ so that the total
variation is at most $\varepsilon/(1-\varepsilon)$. The function
$(1-\varepsilon)\max(p_0,p_t)$ integrates to
$(1-\varepsilon)\{1+\TV(p_0,p_t)\}\le1$; filling the remaining mass with
any law gives one law that is an $\varepsilon$-contamination of both, so
the two parameter points have identical contaminated laws at every $n$ and
the risk is $\Omega(\varepsilon)$. The clean lower bound of
Theorem~\ref{thm:statistics} supplies $\Omega(n^{-1/2})$.
\end{proof}

\paragraph{A nearby point-spread function.}
If the actual kernel $k_*$ satisfies $\norm{k_*-k_\eta}_1\le\eta$ for one
admissible mixture, then, since every line meets the convex cube $U$ in a
segment of length at most $\operatorname{diam}(U)$ and the ray weights sum
to one, Tonelli's theorem gives
$\int_U|(k_*-k_\eta)*\mu_\vartheta|\le\operatorname{diam}(U)\eta$. For two
nonnegative functions with integrals $Z_1,Z_2$ over $U$,
$\norm{q_1/Z_1-q_2/Z_2}_{L^1(U)}\le2\norm{q_1-q_2}_{L^1(U)}/Z_2$, so
$\norm{\pi_*-\pi_\eta}_1\le2\operatorname{diam}(U)\eta/Z_\eta$. The same
bound holds for cell laws, because
$\norm{P-P'}_2\le\norm{P-P'}_1\le\norm{\pi-\pi'}_{L^1(U)}$.

\paragraph{Finite strokes.}
If every ray is cut at length $\ell$, put
$\bar m=(\Lambda/\lambda)(\sqrt DR_0+\bar p+B_0)$. For $z>0$ and
$\ell\ge\bar m+z\sqrt\Lambda$, each conditional longitudinal Gaussian has mean
at most $\bar m$ and variance at most $\Lambda$; its mass beyond $\ell$
relative to its mass beyond zero is at most $2\Phi(-z)$, because the
ratio increases with the mean and at mean $\bar m$ the denominator is at least
one half. Positivity bounds the removed window fraction by $2\Phi(-z)$, so
\begin{equation}\label{eq:finiteguard}
 \norm{\pi_{\rm stroke}-\pi_{\rm ray}}_1\le4\Phi(-z).
\end{equation}

\paragraph{Where the target ends.}
No uniform modulus exists when a ray weight tends to zero or two
directions merge: a second arm of weight $t$ makes the law converge to a
one-ray law as $t\to0$, and two equal arms at angles $\pm t$ converge to
one direction, while $d_T$ stays one. This is why Theorem~\ref{thm:inverse}
keeps the floors $w_*$ and $\delta$. Blur slots are different, because
their labels are not part of $S$. The longitudinal position of a single
line also has a separate obstruction: for opposite weights
$(1\pm\tau)/2$, the blurred intensity in longitudinal and normal
coordinates is $\ph(y)\{1/2+\tau[\Phi(x-p_\parallel)-1/2]\}$, two
vertices a fixed distance apart give densities that differ by
$O(\tau)$ and a Kullback--Leibler divergence of $O(\tau^2)$, and with
$\tau_n=n^{-1}$ their $n$-point laws cannot be told apart. No uniform
estimator of the longitudinal position exists on a class that approaches
balance.

\subsection{Curved arms: proof of Proposition~\ref{prop:curvedarms}}\label{app:curvedproof}
Put $\rho_1=\sqrt DR_0+\bar p$, so $|x-p|\le\rho_1$ for $x\in U$, and
$G(s)=(2\pi\lambda)^{-D/2}\exp\{-(s-B_0)_+^2/(2\Lambda)\}$ for $s\in\R$,
which is nonincreasing. Every component of an admissible kernel obeys
$\ph_\Sigma(z-b)\le G(|z|)$ and
$|\nabla\ph_\Sigma(z-b)|\le\lambda^{-1}(|z|+B_0)G(|z|)$, and so does the
mixture. Since $(s+B_0)e^{-(s-B_0)_+^2/(2\Lambda)}\le
C_1e^{-(s-B_0)_+^2/(4\Lambda)}$ for $s\ge0$, the function
$E(s)=\sup\{|\nabla k_\eta(w)|:|w|\ge s_+\}$ satisfies
$E(s)\le C_2e^{-(s-B_0)_+^2/(4\Lambda)}$. Write
$Q_j(x)=\int_0^\ell k_\eta(x-\gamma_j(r))h_j(r)\dd r$ and
$Q_j^0(x)=\int_0^\infty k_\eta(x-p-ra_j)\dd r$, so that
$k_\eta*\mu=\sum_jw_jQ_j+k_\eta*\nu$ and $q_{\vartheta,\eta}=\sum_jw_jQ_j^0$.
Fix $x\in U$ and an arm $j$, and drop the index $j$.

\emph{Near part.} Put $\delta(r)=\gamma(r)-p-ra$. Since $\gamma$ has unit
speed, $|\delta(r)|\le\int_0^r|\gamma'(s)-a|\dd s\le\varepsilon
r^{1+\beta}/(1+\beta)$, which is at most $r/2$ for
$r\le r_\varepsilon=((1+\beta)/(2\varepsilon))^{1/\beta}$. For such $r$,
every point $w$ of the segment from $x-p-ra$ to $x-\gamma(r)$ has
$|w|\ge r-|\delta(r)|-|x-p|\ge r/2-\rho_1$, so the mean-value theorem gives
$|k_\eta(x-\gamma(r))-k_\eta(x-p-ra)|\le|\delta(r)|E(r/2-\rho_1)$. With
$h(r)\le1+r^\beta$ for $\varepsilon\le1$ and $k_\eta(x-p-ra)\le G(r-\rho_1)$,
the difference $|k_\eta(x-\gamma(r))h(r)-k_\eta(x-p-ra)|$ is at most
\[
 \varepsilon\bigl\{r^{1+\beta}(1+r^\beta)E(r/2-\rho_1)+r^\beta G(r-\rho_1)\bigr\}
 \qquad\text{for }0\le r\le\min(\ell,r_\varepsilon),
\]
and the integral of the braces over $(0,\infty)$ is a finite constant
$C_3$, so this part contributes at most $C_3\varepsilon$.

\emph{Far part.} If $r_\varepsilon<\ell$, then for $r>r_\varepsilon$ the
bound $|\gamma(r)-p|\ge c_0r$ gives $k_\eta(x-\gamma(r))\le
G(c_0r-\rho_1)$, and $k_\eta(x-p-ra)\le G(r-\rho_1)\le G(c_0r-\rho_1)$
because $c_0\le1$. Hence the two integrands over $(r_\varepsilon,\ell]$
contribute at most $T(r_\varepsilon)=\int_{r_\varepsilon}^\infty(2+r^\beta)
G(c_0r-\rho_1)\dd r$. The function $T$ is bounded and decays like a
Gaussian, and $r_\varepsilon\ge(2\varepsilon)^{-1/\beta}$, so
$T(r_\varepsilon)\le C_4\varepsilon$ for all $\varepsilon\in(0,1]$.

\emph{Beyond the arm.} The tangent ray continues past $\ell$:
$\int_\ell^\infty k_\eta(x-p-ra)\dd r\le\int_\ell^\infty G(r-\rho_1)\dd r
\le C_5e^{-(\ell-\rho)_+^2/(2\Lambda)}$ with $\rho=\rho_1+B_0$, by the
Gaussian tail bound $\int_t^\infty e^{-s^2/(2\Lambda)}\dd s\le
\tfrac12\sqrt{2\pi\Lambda}\,e^{-t^2/(2\Lambda)}$ for $t\ge0$.

\emph{Distant structure.} Every point of the support of $\nu$ is at
distance at least $L$ from $x$, so $(k_\eta*\nu)(x)\le mG(L)$.

Summing over arms with $\sum_jw_j=1$,
$\sup_U|k_\eta*\mu-q_{\vartheta,\eta}|\le C\{\varepsilon+
e^{-(\ell-\rho)_+^2/(2\Lambda)}+me^{-(L-B_0)_+^2/(2\Lambda)}\}$. The
normalizer $Z_{\vartheta,\eta}$ is continuous and positive on the compact
class, hence bounded below, and
$\norm{q_1/Z_1-q_2/Z_2}_{L^1(U)}\le2\norm{q_1-q_2}_{L^1(U)}/Z_2$ gives the
density bound. For the scores, the same three steps applied to
$\nabla k_\eta$, whose gradient has a polynomial-times-Gaussian envelope,
bound $\sup_U|\nabla(k_\eta*\mu)-\nabla q_{\vartheta,\eta}|$ by the same
right side. On the compact class, $q_{\vartheta,\eta}\ge q_*>0$ and
$|\nabla q_{\vartheta,\eta}|\le G_*$ on $U$; once the intensity bound is
below $q_*/2$, the identity
$\nabla\log q_1-\nabla\log q_2=\nabla(q_1-q_2)/q_1-(q_1-q_2)\nabla q_2/(q_1q_2)$
transfers both bounds to the scores. The last claim follows from
Proposition~\ref{prop:mismatch} with
$b\le\norm{\pi_\mu-\pi_{\vartheta,\eta}}_{L^1(U)}$. \hfill$\square$

\paragraph{Why the small-curvature bound is uniform.}
The constants depend on the class, $\beta$ and $c_0$ only; they do not
depend on how the curves bend beyond the region $r\le r_\varepsilon$,
because there the lower bound $|\gamma(r)-p|\ge c_0r$ alone controls the
contribution. Without that bound a branch could fold back into the
window, and no local statement would hold.

\subsection{Curves at a fixed blur: proof of Proposition~\ref{prop:curves}}\label{app:curves}
\paragraph{Straight segments.}
Cutting ray $j$ at length $\ell_j$ changes its intensity from
\[
\begin{aligned}
&w_j\ph_{C_j}(z_{\perp,j})\Phi(m_j/\sqrt{V_j})\\
&\text{to}\quad w_j\ph_{C_j}(z_{\perp,j})
 [\Phi(m_j/\sqrt{V_j})-\Phi((m_j-\ell_j)/\sqrt{V_j})],
\end{aligned}
\]
where $m_j(x)$ and $V_j$ are the conditional longitudinal mean and
variance. For fixed $V_j$ and $\ell_j$ the ratio of the removed term to the
ray term increases with $m_j$ (the monotone likelihood ratio of a normal
truncated below at zero), and $m_j$ is affine, so its maximum over a cube
is at a corner. With
$\varepsilon=\max_j\Phi((m_j^{\max}-\ell_j)/\sqrt{V_j})/\Phi(m_j^{\max}/\sqrt{V_j})$,
every partition satisfies
$\norm{P_{\rm segment}-P_{\rm ray}}_2\le\norm{P_{\rm segment}-P_{\rm ray}}_1\le2\varepsilon$,
because the removed window fraction $a$ is at most $\varepsilon$ and the
normalized ray law is $(1-a)\pi_{\rm segment}+a\pi_{\rm removed}$.

\begin{proof}[Proof of Proposition~\ref{prop:curves}]
Use covariance $I$, vertex zero, $r\in[0,1]$ and $\gamma_0(r)=re_1$. For
small $h>0$ put $\delta=c_1h^\beta$ with a small fixed $c_1>0$,
$a_\delta=(\sqrt{1-\delta^2},\delta)$, $f(s)=s(1-s)^2$ on $[0,1]$ and
$f=0$ beyond, and $\gamma_h(r)=re_1+hf(r/h)(a_\delta-e_1)$. The curve
starts with tangent $a_\delta$ and rejoins the straight branch at $r=h$
with matching derivative. Since $f'(0)=1$, $f'(1)=0$ and the extended $f'$
is Lipschitz, the $C^{1,\beta}$ seminorms are bounded in $h$: for $r\le h$,
$|\gamma_h'(r)-a_\delta|\le C\delta(r/h)^\beta=Cc_1r^\beta$, and for
$r\ge h$ the same follows from $\delta\le c_1r^\beta$. Also $0\le f(s)\le s$,
so $|\gamma_h(r)|\ge(1-\delta^2)r\ge3r/4$ for $\delta\le1/2$. With
$q_h(x)=\int_0^1\ph_I(x-\gamma_h(r))\dd r$, the Gaussian gradient is
bounded, the sources differ only on $[0,h]$, and
$\int_0^h|\gamma_h-\gamma_0|\dd r=h^2|a_\delta-e_1|\int_0^1f\le C\delta h^2$,
so $\norm{q_h-q_0}_\infty\le C\delta h^2$. On a fixed window the
intensities and normalizers are bounded below, so
$\norm{\pi_h-\pi_0}_\infty\le C\delta h^2$ and
$\KL(\pi_0\Vert\pi_h)\le\int(\pi_0-\pi_h)^2/\pi_h\le C\delta^2h^4\le
C'h^{2\beta+4}$. With $h=bn^{-1/(2\beta+4)}$ and small $b$ the product
divergence is at most $1/8$ and the total variation at most $1/4$, while
the tangents differ by at least $\delta$; the two-point inequality gives
the bound. Counts cannot improve on exact positions, and the construction
already fixes the vertex, covariance, density and count.
\end{proof}

A smooth class thus allows tangent changes of order $h^\beta$ whose
effect on the observed density is of order $h^{\beta+2}$, and more
photons at a fixed blur cannot remove this. The $n^{-1/2}$ rate belongs
to the exact-ray model.
\subsection{Proof of Proposition~\ref{prop:hessian}}\label{app:hessian}
Let the lines have unit directions $a_1$ and $a_2$, with $a_2$ at angle
$\theta$ counterclockwise from $a_1$, unit normals $n_1,n_2$, and equal
line densities. After blurring and smoothing with an isotropic Gaussian
of total variance $V$, the image is proportional to
$\phi_V(n_1^\top x)+\phi_V(n_2^\top x)$, and
$\nabla^2\phi_V(n^\top x)=\phi_V''(n^\top x)\,nn^\top$ with
$\phi_V''(z)=\phi_V(z)(z^2/V-1)/V$. At $x=ta_1$, $n_1^\top x=0$ and
$(n_2^\top x)^2=t^2\sin^2\theta=rV$, so
\[
 \nabla^2(\text{image})(ta_1)\ \propto\ -n_1n_1^\top+q\,n_2n_2^\top,
 \qquad q=(r-1)e^{-r/2},
\]
with a positive factor $\phi_V(0)/V$. For a symmetric $2\times2$ matrix
$\sum_kc_ku_ku_k^\top$ with $u_k$ at angle $\psi_k$, the eigenvector of
the largest eigenvalue lies at half the angle of
$\sum_kc_k(\cos2\psi_k,\sin2\psi_k)$. Applied to
$n_1n_1^\top-q\,n_2n_2^\top$, whose top eigenvector is the ridge normal,
with $n_1$ at $90^\circ$ and $n_2$ at $90^\circ+\theta$, this gives a
normal, and hence a tangent, turned counterclockwise by
$\tfrac12\operatorname{atan2}(-q\sin2\theta,\,1-q\cos2\theta)$, that is,
clockwise, away from $a_2$, by
$\tfrac12\operatorname{atan2}(q\sin2\theta,\,1-q\cos2\theta)$. Since
$-1\le q\le2e^{-3/2}<1$ and $|\cos2\theta|<1$, the second argument is
positive, so the angle has the sign of $q$ and is positive exactly when
$r>1$; it increases with $q$, which is largest at $r=3$. By symmetry the tangent on $a_2$ turns away from
$a_1$ by the same amount, so both arms of the acute pair turn apart.
\hfill$\square$

Inside the core, $r<1$, the tangent turns toward the other line: at
$t=v_0/\sin\theta$ the unsmoothed tangent is exact, and smoothing by $s$
turns it toward the other line by
$e^{-1/2}\sin(2\theta)\,s^2/(2v_0^2)+O(s^4)$. A ridge detector that reads
orientations near the crossing thus mixes both signs; the tuned detector
reads them at $t\ge0.35R=140$\,nm, where $r>1$. A direct evaluation of the
eigenvectors agrees with the closed form to machine precision (for
example $+6.444^\circ$ at $\theta=30^\circ$, $t=140$\,nm and
$+12.024^\circ$ at $\theta=45^\circ$, $t=140$\,nm, with $\sqrt V=57.3$\,nm).

\section{Experimental design}\label{app:design}
\subsection{Data}\label{app:data}
Every microscopy table comes from ShareLoc \citep{ouyang22} and is
distributed on Zenodo under the Creative Commons Attribution 4.0 license.
Each row is one localization: a position in nanometres, a frame number
and the fitting software's precision estimate, which the methods do not
use.

\begin{itemize}[leftmargin=*,itemsep=1pt]
\item \emph{Development and first study.} ``2D microtubules stained with
Alexa647'' (NeuroCyto, CNRS--Aix-Marseille University), Zenodo records
5534358 (part~1, development) and 5534354, 5534351, 5534346, 5534342
(parts~2--5, 29 fields). The Hessian ridge detector was tuned on part-1
fields 1 and 4.
\item \emph{Cross-laboratory study.} 27 fields of COS-7, REF52 and U2OS
cells from two laboratories (Mechanobiology Institute, Singapore:
records 5174118, 5174498, 5184659, 5214573, 5215248; University of
W\"urzburg: records 5510686, 7182290, 7182242), two dyes (Alexa Fluor 647
and Alexa 532) and two camera types (EMCCD and sCMOS).
\item \emph{U-373 study.} 28 fields of U-373\,MG cells, Alexa 647,
EMCCD camera (Imaging and Modeling Unit, Institut Pasteur), from the 17
records of the series ``Microtubules immuno-labeled with Alexa 647 with
widefield images''.
\item \emph{Unused acquisitions.} 34 tables that no earlier analysis had
opened: ``Microtubules in COS7 cells'' parts 1 and 2 and ``Microtubules
immuno-labeled with CF680'' (Abbelight; records 5512383, 5512818,
5512205) and single cells ``Microtubules immuno-labeled with Alexa 647
CSV and raw data'' (Institut Pasteur; records 5773453, 5779541,
5783066). Part~A uses 22 fields and part~B 10 fields.
\item \emph{Natural junctions.} 1,391 sites in 73 fields of the first
study's parts~2--5 and of the cross-laboratory records
(Appendix~\ref{app:natural}), and 348 sites in 88 U-373 cells.
\item \emph{Camera images.} The DL-SMLM microtubule collection
\citep{zhao25}: for each of 33 cells, a 100-frame diffraction-limited
image stack and a super-resolution reconstruction of the same field.
Cells 1--4 were development data; the crossings of the camera test come
from cells 5--33.
\end{itemize}

\emph{X-ray.}\label{app:xraydesign} Chandra event files come from the public Chandra Data
Archive. Candidates are all archived observations with the High Energy
Transmission Grating on the ACIS-S detector and exposure of at least
20\,ks; observations within $60''$ are one source, and each source keeps
its longest exposure. Sources are ordered by a seeded random permutation
and admitted in that order by five rules fixed in advance: timed read
mode with the grating inserted; at least 200 events in each grating and
a pipeline dispersion map with residual below one pixel; a point-like
zeroth order (median distance of its photons at most 3 pixels); at least
4,000 window photons and 2,000 first-order photons; and zeroth-order
photons at least $0.10$ of the window photons. The last rule was added
after a first pass showed sources whose arms hold 0.1 to 0.5\% of the
photons; that pass was discarded before any evaluation estimate. The first study used 30 sources after 9 development
sources; the retest used the next 30 new sources in the same order:
Iota Ori, NGC 4507, NGC 6814, 4U 0114+65, GK Per, PKS 2149-306, SS 433,
XTE J1710-281, H 1821+643, Mkn 421, IM Peg, NGC 3783, NGC 4388, Mrk 509,
PG 0844+349, MR 2251-178, V745 Sco, Mrk 590, IGR J18179-1621,
OAO 1657-415, Ark 120, VW Cep, HD 110432, NGC 985, MCG 8-11-11, AO Psc,
Apep, TY Pyx, HD 199143 and 4U 1626-67. The reference directions are
minus the nominal roll angle plus the grating angle ($+5.235^\circ$ for the
high-energy grating, $-4.725^\circ$ for the medium-energy grating) plus a
constant offset measured on the 39 earlier sources ($-0.02753^\circ$ and
$-0.02830^\circ$, spread $0.0009^\circ$); no retest photon enters them. On
all 30 retest sources they agree with a map fitted to every photon within
$0.0016^\circ$, and every comparison below is the same to three decimals
under the calibration, disjoint-photon and all-photon references.

\emph{Road intersections.} Porto: taxi traces of July 2013 to June 2014
sampled every 15\,s \citep{moreira13}, in a box of about $9\times5.5$\,km;
a fix counts as moving if it is at least 15\,m from the previous or the
next fix of the same trip. Sites are OpenStreetMap intersections with
three or four arms, every arm reaching 40\,m without another intersection
or dead end and straight within $2^\circ$ beyond 14\,m, no other drivable
road in the disk of radius 40\,m, arms at least $25^\circ$ apart, and the
2014 and 2026 maps in agreement. Traffic rules: at least 400 moving fixes
in the disk, and every arm with at least 10 fixes from at least 3 trips
within 10\,m of the arm beyond 14\,m and at least 5\% of the fixes near
the arms. Of the 117 sites, those in one-kilometre tiles assigned to
development by a fixed hash (21) tuned the detectors; the other 96, in 27
tiles, form the evaluation. The window center is the map vertex plus a
seeded offset of at most 5\,m. External sites: Berlin and Chicago from the
map-construction benchmark \citep{ahmed15} (one map snapshot; at least 100
moving fixes), and Beijing from GeoLife \citep{zheng10} with the 2013 and
2026 maps in agreement; all used the frozen Porto settings.

\subsection{Assembled junctions}\label{app:assembly}
\emph{Straight stretches.} A candidate center lies on the rendered
skeleton of a field, at least 700\,nm from every branch and end point. A
total-least-squares line through the localizations within 600\,nm,
refitted three times on the inliers within 100\,nm, must hold at least 85\% of
them with perpendicular root-mean-square at most 60\,nm; each half needs
at least 150 localizations and its own direction within $2^\circ$ of the
line; no other skeleton pixel may lie within 600\,nm; stretches are at
least 1200\,nm apart, so no two share localizations. The stored
localizations lie within 300\,nm of the line and 450\,nm of the center
along it, and a stretch is used if its bootstrap direction standard error
is at most $0.5^\circ$. Background disks have no skeleton pixel within
700\,nm and at most 60 localizations within 400\,nm. For the U-373 study
the same rules run at twice the length (stored extent 850\,nm, clearance
900\,nm, at least 300 localizations per half), so that localizations from
400 to 800\,nm exist beyond the window.

\emph{Designs.} All arms and the background of a junction come from one
field, and no stretch is used twice within a junction. A cross places two
full stretches at an angle drawn uniformly from $30^\circ$ to $90^\circ$;
its control turns one half of each stretch by $25^\circ$ to $45^\circ$, so
no collinear pair remains. A T-junction places a full stretch and one
half of another at $35^\circ$ to $145^\circ$; its control turns one half of
the full stretch by $25^\circ$ to $45^\circ$. Every arm pair is at least
$30^\circ$ apart. For a half-stretch the localizations with along-coordinate
in $(0,450]$\,nm are rotated rigidly so the cut point lands on the vertex;
nothing is rescaled or resampled. The whole configuration is rotated
uniformly, a measured empty disk of the same field is added with its own
random rotation, and the window center is offset from the vertex by a
seeded draw uniform in a disk of radius 50\,nm. Each study draws 30 pairs
of each design: 120 junctions. From each junction we draw one subsample
of $n\in\{25,50,100,200,400,800\}$ localizations inside the 400\,nm
window (three subsamples in the first study); junctions with fewer
localizations skip the larger sizes.

\emph{References.} In the first, cross-laboratory and U-373 studies the
reference directions are the assigned directions of the placed
stretches. In the unused-acquisition study the reference for each arm is
the total-least-squares direction of the same stretch's localizations
with along-coordinate in $(0,400]$\,nm and within 100\,nm of the line,
excluding the localizations in the scored subsample; the assigned
directions serve as a sensitivity analysis.

\emph{Faint arms.} Part~B of the unused-acquisition study draws 30
crosses and 30 T-junctions from 10 fields with long stretches, and thins
the second stretch so that each of its arms holds a share of the window
drawn uniformly from 2\% to 8\%; the same 60 geometries unthinned form
the twin set. Localizations from 400 to 800\,nm are kept for the
continuation check, thinned by the inside sampling fraction so the
intensity per area is unchanged.

\subsection{Methods}\label{app:methods}
\emph{Likelihood fit.} Coordinates are in units of 100\,nm and $R=4$. The
parameters are the vertex, $K$ directions, a background share and $K$
arm shares summing to one, two log widths and two slot weights summing
to one. Bounds: vertex in $[-2,2]^2$, background share in $[10^{-4},0.8]$,
arm shares in $[0.02,1]$, widths in $[0.03,2]$; directions are ordered
with cyclic gaps of at least $25^\circ$. Each arm's density is divided by
its exact mass in the disk. Starts: vertex at the center, background
share $0.1$, arm shares $0.9/K$, widths spaced geometrically from $0.2$ to
$0.5$, equal slot weights; directions from four evenly spaced patterns,
up to six combinations of angular-histogram peaks at $0.35R$ to $R$
(relative threshold $0.1$) at least $25^\circ$ apart, and the directions
returned by two earlier likelihood implementations (one and two
Gaussians), a fixed-width fit, the Hessian ridge detector and the line
fit. Each start runs sequential least-squares quadratic programming (at
most 500 iterations, tolerance $10^{-8}$), and the best feasible
objective is kept; every output satisfies the constraints. The BIC uses
$2K+5$ free parameters. The agreement rule of Section~\ref{sec:method}
then sets the count. On one Apple M4 processor (10 cores, 16\,GB
memory), the median window with all four counts takes 1.6\,s at 25
localizations and 2.9\,s at 800 (90th percentiles 3.1 and 5.0\,s); no
GPU is used.

\emph{Earlier implementation.} The first study and the natural-junction
analyses use an earlier implementation of the same model: two-Gaussian
ray likelihood, counts 2 to 5 by BIC among fits that satisfy the
$25^\circ$ and $0.02$ floors, nested starts that add one arm at a residual
or histogram peak, and no agreement rule.

\emph{Hessian ridge detector.} Localizations are rendered at 10\,nm;
second derivatives come from Gaussian derivative filters of width
50\,nm. At each pixel the ridge normal is the eigenvector of the most
negative Hessian eigenvalue and the tangent is perpendicular to it,
oriented away from the window center; the ridge strength is the negative
of that eigenvalue. Pixels in the annulus from $0.35R$ to $R$ with
strength above the 30\% quantile enter a strength-weighted 72-bin
orientation histogram, smoothed with a Gaussian kernel; peaks that exceed the
median by more than $0.1$ of the gap between maximum and median, at
least $25^\circ$ apart, are arms, each refined by the weighted circular mean within $10^\circ$. The
width, quantile and threshold were tuned on development fields.

\emph{Other detectors.} Steger's detector is the ImageJ Ridge Detection
algorithm \citep{steger98} in a Python port, run on a 20\,nm rendering
with line width $5.5$ and contrasts $0$ and $230$, tuned on 100
development windows; arms are runs of detected line points in the
annulus, one total-least-squares direction per run, merged within
$25^\circ$. The line fit finds angular-histogram peaks about the window
center (90 bins, relative threshold $0.25$), fits a total-least-squares
line to the localizations of each arm's $12^\circ$ wedge beyond $0.35R$,
then iterates six times on strips of half-width 80\,nm around each line, with the
vertex at the least-squares intersection. The angular histogram uses the
same peaks about a vertex estimated as the mode of the smoothed density
within $R/2$, each direction the circular mean within $10^\circ$. The
skeleton method renders at 10\,nm, smooths at 25\,nm, thresholds by
Otsu's rule, skeletonizes and fits one direction to each skeleton
component between $0.4R$ and $0.9R$. The angular histogram about the true
vertex is reported as a reference with privileged information.

\emph{X-ray.} Photons of 0.5 to 8\,keV within 1000 pixels of the
zeroth-order image are scaled to radius $4$; a disk of 12 pixels around
the zeroth order and a strip of half-width 3 pixels along the readout
streak are masked. The likelihood uses floors of $4^\circ$ and $0.004$,
widths in $(0.002,2)$, a width start of $0.004$ (one detector pixel), arm
share starts of 5\%, directions from a $0.5^\circ$ angular histogram, and
nested starts; the BIC chooses the count from 2 to 5 and between one and
two Gaussians. The first study used the same model with generic starts
(width $0.02$). Baselines were tuned on the development sources by grid
search; for the retest the objective was exact recovery within $1^\circ$
(Table~\ref{tab:xraytuned}).

\begin{table}[h]
\centering\small
\caption{Tuned X-ray baselines for the retest (development means over
100 to 800 photons of exact recovery within $1^\circ$ and $10^\circ$).
Units of pixel and smoothing are window units ($1=250$ detector
pixels).}
\label{tab:xraytuned}
\begin{tabular}{l>{\raggedright\arraybackslash}p{0.5\textwidth}cc}
\toprule
Method & Settings & $1^\circ$ & $10^\circ$\\
\midrule
Line fit & 720 bins, threshold 0.1, wedge $2^\circ$, strip 0.01, inner radius 0.1 & .288 & .304\\
Histogram, true vertex & 720 bins, threshold 0.1, half-width $1^\circ$ & .263 & .266\\
Hessian ridge detector & pixel 0.02, smoothing 0.1, quantile 0.99, threshold 0.25, 360 bins & .013 & .045\\
Histogram, estimated vertex & 360 bins, threshold 0.1, half-width $1^\circ$, vertex pixel 0.05 & .003 & .080\\
Steger's detector & pixel 0.1, line width 12, high contrast 200 & .000 & .000\\
\bottomrule
\end{tabular}
\end{table}

\subsection{The detector with automatic scale selection}\label{app:adaptive}
The detector is the fixed-scale Hessian detector with its scale chosen at
every pixel \citep{lindeberg98}. The Hessian is computed at eleven scales
$s$ from 15 to 100\,nm, and the pixel keeps the scale with the largest
normalized ridge strength $s^{2\gamma}(\lambda_{\max}-\lambda_{\min})$,
where $\lambda_{\min}\le\lambda_{\max}$ are the Hessian eigenvalues and
the strength is set to zero off bright ridges; the tangent at that scale
enters the unchanged arm-extraction rule. The settings were calibrated on
the 240 junctions of the first study, which shares no field with the
three evaluation studies, for exact recovery averaged over the six sizes.
The first calibration grid, whose finest scale was 15\,nm, reached $0.33$
against $0.56$ for the fixed width: pixels at scales of 25\,nm or less
carried a median 62 to 82\% of the orientation-histogram weight when
$\gamma<1$, and the best setting of that grid ($\gamma=1.25$) selected 85
to 100\,nm from 100 localizations on, which merges neighbouring arms. Two
rules that pick one scale per window scored at most $0.41$. Before any
evaluation run, the finest scale was therefore added to the calibrated
settings; the final choice, $\gamma=0.5$ with scales from 50 to 100\,nm,
reached $0.555$ against $0.561$. The detector then ran once on every
stored window, and every replayed window reproduced the stored output of
the fixed-scale detector exactly. It takes 29\,ms per window against
8\,ms for the fixed scale. On part~B of the unused acquisitions the fit
leads it by $23.0$ $[20.6,25.3]$ points over 25 to 400 localizations.

\subsection{The camera test}\label{app:camera}
Candidate sites are the 412 junction sites of cells 5--33 with a
super-resolution reference. A site is a crossing if it has four reference
arms, the pairing that minimizes the deviation from antiparallel leaves
both pairs within $15^\circ$ of a straight line, and the two line axes are
at least $25^\circ$ apart; sites within 16 pixels of the image border or
with a saturated pixel in that margin in any frame are excluded. This
leaves 91 crossings in 21 cells. For line axes $a,b$ put
$H=(ab^\top+ba^\top)/\norm{ab^\top+ba^\top}_{\rm F}$ and smooth an image
with Gaussian kernels of covariance $4I\pm3.2H$ (in squared pixels), both
positive definite; the difference of the two smoothed images is the
response. Controls replace $H$ by $I/\sqrt2$ (isotropic, fixed in advance)
or by $H$ turned by $45^\circ$ (added after the first result), with the
same Frobenius norm. Frames are split by index modulo three into three
averages; the energy is the mean over a disk of radius 6 pixels about the
site of the product of the responses of the first two averages, which
removes the energy of zero-mean noise that is independent between them,
divided by the square of the 10--90\% intensity range of the third
average in the same disk. On analytic images of full crossings at three
angles and three rotations, the ratio of the $H$ response to the isotropic
response is below $2\times10^{-9}$ for balanced lines and between $0.008$
and $0.021$ when the two halves of one line carry weights $1.25$ and $0.75$. Using
the first and the last 30 frames instead of interleaved groups changes the
mean energies by less than $0.5\%$.

\subsection{The curve check}\label{app:curvecheck}
Each of the 72 junctions has two to four arms with random directions at
least $28^\circ$ to $62^\circ$ apart, weights with ratio up to 6, a random
anisotropic Gaussian covariance of determinant one and condition number up
to 16, arm length 2.2 to 3.8, and zero to three point sources of weight
$0.005$ to $0.04$ at distance 1.1 to 2.2 from the vertex. The tangent of
arm $j$ turns in the plane of $a_j$ and a random normal by the angle
$c_jr^\beta$ with $|c_j|$ from $0.05$ to $0.65$, and its density is
$1+d_jr^\beta$ with $d_j$ from $-0.18$ to $0.28$. For each blur scale
$\sigma\in\{0.07,0.1,0.145,0.21,0.305,0.44\}$ the score
$\sigma\nabla\log q(p+\sigma z)$ of the curved junction is computed by
180-node Gauss--Legendre quadrature along each arm, up to arc length
$13\sigma$, and compared with the
closed-form score of the tangent rays at the vertex and 74 directions at
radii $0.55$, $1.25$ and $2.25$. The slope of each junction is the
least-squares slope of the log root-mean-square difference on
$\log\sigma$ over the four smallest scales; the protocol, seed and
endpoint were fixed before the run.

\subsection{Protocols and analysis}\label{app:prereg}
Each study had a written protocol fixed before any evaluation fit: data
and field lists, construction seeds, every method and setting, the
references, the endpoints with their pass criteria, and the analysis
code. Result files were not changed after the runs. One study was
restarted after a machine failure with fewer workers and single-threaded
linear algebra; this changed no setting and is recorded in its protocol.
Every junction and source stays in every denominator; a failed fit
counts as a wrong count.

\emph{Endpoints.} Exact recovery: correct count and every arm within
$10^\circ$ ($1^\circ$ for X-ray) under the best one-to-one matching. The
main endpoint averages exact recovery over 25 to 400 localizations,
first over fields within a size, then over sizes. Capped loss: mean arm
error in degrees if the count is right, $90$ otherwise. Rate: least-squares
slope of the log median arm error of fits with the correct count on
$\log n$ over 25 to 400 localizations (100 to 800 photons, fits with all
arms within $10^\circ$, for X-ray), with pass band $[-0.65,-0.35]$.
Acute angles: for adjacent reference arms less than $60^\circ$ apart, the
estimated minus the reference angle, on junctions where both the fit
and the Hessian detector have the right count, averaged with equal field
weight. Crossing tests: exact recovery at a cross or T-junction minus
that at its paired control, averaged over 25 to 400, pass if the interval
lies within $\pm10$ points.

\emph{Uncertainty.} Intervals are 95\% percentile intervals from 5000
bootstrap resamples of fields (sources for X-ray) with equal field
weight; figures use 2000 resamples. Resampling fields keeps within-field
dependence (blinking, drift, shared background) inside the unit.

\subsection{Natural junctions}\label{app:natural}
\paragraph{Sites and references.}
Natural sites were chosen on one half of each acquisition (alternating
blocks of 1,000 frames): branch points of the skeleton of a 20\,nm
rendering, isolated by 500\,nm, with at least 400 localizations within
400\,nm in each half, three or four arms by a model-free angular count
between 150 and 400\,nm, and at least 90\% (70\% for the later
cross-laboratory fields) of those localizations within $12^\circ$ of a
counted arm; up to 12 sites per field by a seeded permutation. The other
half supplies subsamples of 25 to 400 localizations. The reference is
built from the site-selection half: the angular-histogram count and
directions about the skeleton vertex. On assembled junctions, where the
truth is known, this reference has median arm error $4.18^\circ$; a
second reference, centered on the vertex of the line fit and refined by
circular means within $12^\circ$, has $1.94^\circ$. Neither uses the
likelihood fit.

\section{Complete outcomes}\label{app:outcomes}
Tables~\ref{tab:ledger}--\ref{tab:ledger3} list every pre-declared
criterion of the microscopy and X-ray studies in the order the studies
were run; the road-intersection studies are reported in
Section~\ref{sec:gps}. Post hoc analyses are labelled and change no
verdict.

\begin{table}[htbp]
\centering\footnotesize
\caption{Every pre-declared criterion, first and cross-laboratory
studies. Differences are the likelihood fit minus the named method, in
points of exact recovery (capped loss in degrees, lower is better);
brackets are 95\% field-bootstrap intervals.}
\label{tab:ledger}
\begin{tabular}{p{0.19\textwidth}p{0.55\textwidth}p{0.17\textwidth}}
\toprule
Criterion & Result & Verdict\\
\midrule
\multicolumn{3}{l}{\emph{First study: 240 junctions, 29 fields, earlier implementation of the fit}}\\
Beat every detector & Hessian $+28.6$ $[24.6,33.2]$; Steger $+59.7$ $[57.2,62.4]$; line fit $+22.0$ $[17.8,26.3]$; histogram $+40.1$ $[36.5,43.5]$; skeleton $+52.3$ $[48.5,56.1]$ & pass\\
Rate & $-0.420$ $[-0.451,-0.391]$ & pass\\
Sandwich coverage in $[0.90,0.98]$ & 0.834 / 0.828 / 0.742 at 100 / 200 / 400 & fail\\
\midrule
\multicolumn{3}{l}{\emph{Cross-laboratory study: 120 junctions, 27 fields}}\\
Beat every detector & Hessian $+24.1$ $[18.8,29.0]$; Steger $+55.8$ $[48.1,62.8]$; line fit $+26.2$ $[21.5,30.6]$; histogram $+38.0$ $[30.2,46.1]$; skeleton $+49.3$ $[42.4,55.7]$ & pass\\
Lower capped loss, 400 and 800 & Hessian $-1.8$ $[-5.5,2.6]$; Steger $-29.7$ $[-37.8,-22.1]$; line fit $-5.5$ $[-9.8,-1.9]$; histogram $-15.0$ $[-22.4,-9.2]$; skeleton $-18.1$ $[-25.1,-11.7]$ & fail (Hessian tie)\\
Rate & $-0.435$ $[-0.526,-0.337]$ & pass\\
Joint regions cover $\ge0.90$ & 0.877 / 0.935 / 0.880 at 100 / 400 / 800 & fail\\
Crossing tests within $\pm10$ & cross $-0.8$ $[-8.5,6.9]$; T $-1.9$ $[-10.0,5.9]$ & pass (at the margin)\\
Detector widens acute angles & Hessian $+9.3$ $[6.2,12.8]$; fit $+0.4$ $[-0.7,1.4]$ & pass\\
\bottomrule
\end{tabular}
\end{table}

\begin{table}[htbp]
\centering\footnotesize
\caption{Every pre-declared criterion, U-373 and unused-acquisition
studies; conventions as in Table~\ref{tab:ledger}.}
\label{tab:ledger2}
\begin{tabular}{p{0.19\textwidth}p{0.55\textwidth}p{0.17\textwidth}}
\toprule
Criterion & Result & Verdict\\
\midrule
\multicolumn{3}{l}{\emph{U-373 study: 120 junctions, 28 fields}}\\
Beat every detector & Hessian $+37.1$ $[31.3,42.3]$; Steger $+60.6$ $[54.5,65.9]$; line fit $+33.5$ $[27.5,39.5]$; histogram $+43.5$ $[35.8,51.3]$; skeleton $+58.4$ $[51.6,64.7]$ & pass\\
Lower capped loss, 400 and 800 & Hessian $-2.7$ $[-4.3,-1.3]$; Steger $-36.4$ $[-44.7,-28.4]$; line fit $-2.1$ $[-4.3,-0.2]$; histogram $-17.8$ $[-25.0,-11.2]$; skeleton $-19.5$ $[-25.4,-14.3]$ & pass\\
Rate & $-0.331$ $[-0.380,-0.244]$; post hoc, window references: $-0.442$ & fail\\
Crossing tests within $\pm10$ & cross $+4.1$ $[-2.1,12.4]$; T $-5.4$ $[-12.5,1.3]$ & fail\\
Detector widens acute angles & Hessian $+12.4$ $[8.7,15.9]$; fit $+1.3$ & pass\\
Continuation restores faint arms & no junction had an arm below 8\% & not testable\\
No new extra arms & 0 at every size; no count changed & pass\\
Pooled regions & coverage 0.881 / 0.982 / 0.798; 3.7 / 3.9 / 2.5 counts; widths $22.9^\circ$ / $26.8^\circ$ / $13.6^\circ$ & fail\\
\midrule
\multicolumn{3}{l}{\emph{Unused acquisitions: part A, 120 junctions, 22 fields; part B, 120 junctions, 10 fields}}\\
Beat every detector & Hessian $+29.4$ $[24.3,34.1]$; Steger $+60.8$ $[56.3,65.2]$; line fit $+25.5$ $[22.8,28.3]$; histogram $+39.6$ $[31.1,47.4]$; skeleton $+57.9$ $[52.4,63.4]$ & pass\\
Rate & $-0.457$ $[-0.536,-0.393]$ (assigned directions: $-0.422$) & pass\\
Detector widens acute angles & Hessian $+7.2$ $[5.0,9.4]$; fit $+0.9$ & pass\\
Uncertainty rule & coverage 0.98 / 1.00 / 0.92; 1.42 / 1.12 / 1.20 counts; widths $69.7^\circ$ / $13.2^\circ$ / $10.9^\circ$ at 100 / 400 / 800 & fail (width at 100)\\
Continuation restores faint arms & $+50.3$ $[37.6,64.5]$ points of count accuracy at 400 and 800 & pass\\
No new extra arms & 0 at 100, 400 and 800 & pass\\
\bottomrule
\end{tabular}
\end{table}

\begin{table}[htbp]
\centering\footnotesize
\caption{Every pre-declared criterion, Chandra studies; conventions as in
Table~\ref{tab:ledger}.}
\label{tab:ledger3}
\begin{tabular}{p{0.19\textwidth}p{0.55\textwidth}p{0.17\textwidth}}
\toprule
Criterion & Result & Verdict\\
\midrule
\multicolumn{3}{l}{\emph{First study: 30 sources}}\\
Lower capped loss, 400 and 800 & histogram, true vertex $-1.0$ $[-10.1,7.4]$; line fit $-1.0$ $[-10.1,7.4]$; histogram $-26.3$; Steger $-24.5$; Hessian $+15.8$ $[2.8,27.9]$ & fail\\
Rate & $-2.85$ $[-3.71,-1.04]$ & fail\\
Sandwich coverage in $[0.90,0.99]$ & 0.54 $[0.44,0.64]$ & fail\\
\midrule
\multicolumn{3}{l}{\emph{Retest: 30 new sources}}\\
Beat every detector, $1^\circ$, 400 and 800 & Hessian $+38.3$ $[25.0,51.7]$; Steger $+38.3$ $[25.0,51.7]$; histogram $+37.8$ $[24.4,51.1]$; histogram, true vertex $+2.2$ $[-4.4,8.9]$; line fit $-1.1$ $[-8.9,6.7]$ & fail (two ties)\\
Revised over original starts & $+19.4$ $[8.3,31.1]$ & pass\\
Rate, fits with all arms within $10^\circ$ & $-0.518$ $[-0.785,-0.308]$ (18, 24, 36, 38 of 90 draws) & pass\\
Sandwich coverage in $[0.90,0.99]$ & 0.83 $[0.76,0.89]$ & fail\\
\bottomrule
\end{tabular}
\end{table}

\paragraph{Every detector.}
Figure~\ref{fig:detectors} shows exact recovery for the six methods in the
four studies with assembled junctions. The likelihood fit has the highest
exact recovery at every size in every study. The closest detector is the
line fit in three studies and the Hessian detector in the
cross-laboratory study.

\begin{figure}[h]
\centering\includegraphics[width=\textwidth]{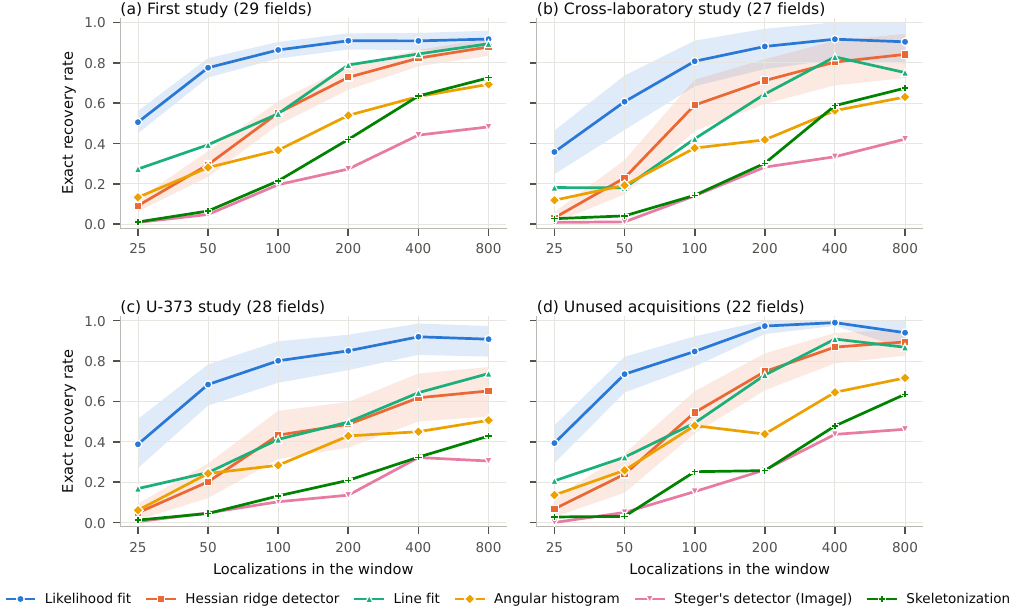}
\caption{Exact recovery by method in the four studies with assembled
junctions. Panel (a) uses the earlier implementation of the fit and three
subsamples per size. Bands: 95\% field-bootstrap intervals for the fit
and the Hessian detector.}
\label{fig:detectors}
\end{figure}

\paragraph{Unconditional loss and denominators.}
The rate plots condition on a correct count. The unconditional capped
loss, which charges $90^\circ$ for a wrong count, is lower for the fit than
for the detector at every size in all three studies, and the two are close
at 800 localizations. Table~\ref{tab:denom} gives the number of fits behind
every point of the rate plots.
\begin{table}[h]
\centering\footnotesize\setlength{\tabcolsep}{4pt}
\caption{Fits with the correct count (microscopy) or with all four arms
within $10^\circ$ (X-ray), out of all windows or draws at each size: the
denominators of Figures~\ref{fig:ratebias} and~\ref{fig:xray}.}
\label{tab:denom}
\begin{tabular}{llcccccc}
\toprule
Study & Method & 25 & 50 & 100 & 200 & 400 & 800\\
\midrule
Cross-laboratory & fit & 76/120 & 102/120 & 108/120 & 109/120 & 114/120 & 112/120\\
 & Hessian & 47/120 & 83/120 & 99/120 & 104/120 & 112/120 & 113/120\\
U-373 & fit & 85/120 & 105/120 & 109/120 & 115/120 & 117/120 & 117/120\\
 & Hessian & 54/120 & 96/120 & 88/120 & 107/120 & 114/120 & 118/120\\
Unused acquisitions & fit & 87/120 & 105/120 & 109/120 & 115/120 & 118/120 & 96/97\\
 & Hessian & 52/120 & 78/120 & 104/120 & 113/120 & 116/120 & 96/97\\
\midrule
 & & & & 100 & 200 & 400 & 800\\
Chandra retest & fit & & & 18/90 & 24/90 & 36/90 & 38/90\\
 & line fit & & & 12/90 & 20/90 & 34/90 & 37/90\\
 & histogram, true vertex & & & 10/90 & 17/90 & 28/90 & 37/90\\
\bottomrule
\end{tabular}
\end{table}

\paragraph{Local sensitivity.}
Each fit in the unused-acquisition study stores per-arm sandwich
standard errors at its fitted parameters, with independent localizations
and with 1000-frame blocks of each stretch as clusters. For reported fits
with the correct count (Table~\ref{tab:local}), the median plain standard
error times $\sqrt n$ is $21.0^\circ$ to $21.3^\circ$ from 100 to 800
localizations: the local inverse of the model is stable over the
evaluated sizes, and its scale is what Theorem~\ref{thm:statistics}
predicts in order. From 50 localizations on, observed errors against the
held-out references are 0.7 to 1.2 times the clustered standard error;
at 25 localizations they are 2.7 times larger, where counting and the
start search dominate.

\begin{table}[h]
\centering\small
\caption{Observed arm error and stored standard errors of reported fits
with the correct count, unused acquisitions (medians over arms, degrees).}
\label{tab:local}
\begin{tabular}{rrrrrr}
\toprule
$n$ & arms & error & plain SE & clustered SE & error / clustered SE\\
\midrule
25 & 301 & 4.08 & 1.54 & 1.35 & 2.71\\
50 & 365 & 2.65 & 2.40 & 2.34 & 1.12\\
100 & 378 & 1.76 & 2.11 & 2.22 & 0.69\\
200 & 401 & 1.48 & 1.48 & 1.63 & 0.90\\
400 & 412 & 1.25 & 1.06 & 1.20 & 1.00\\
800 & 335 & 1.15 & 0.75 & 1.00 & 1.17\\
\bottomrule
\end{tabular}
\end{table}

\paragraph{Where the Hessian detector ties.}
With 400 and 800 localizations, the capped loss of the fit and the
detector tie in the cross-laboratory study ($-1.8^\circ$ $[-5.5,2.6]$).
The capped loss charges $90^\circ$ for a wrong count, so at large $n$ it is
dominated by counting: the plain BIC count over-counts (count accuracy
0.76 at 800 localizations, capped loss $23.0^\circ$), and the agreement
rule brings it to 0.93 and $7.5^\circ$. On the U-373 cells the rule reaches
count accuracy 0.97 at 400 and 800 localizations against 0.72 and 0.73
for the BIC and 0.95 and 0.98 for the Hessian detector, with pooled exact
recovery 0.89 and 0.87 against 0.58 and 0.62 for the detector. Restarts
(Appendix~\ref{app:diagnostics}) show that part of the remaining count
errors at large $n$ are optimization failures of the larger-count fit.

\paragraph{The U-373 rate.}
The U-373 stretches define each reference direction over $\pm850$\,nm,
while the method sees at most 400\,nm; slight curvature gives the
reference a floor. Scored post hoc against the direction of each
stretch's own localizations within the window, the slope is $-0.442$ and
the median errors at 400 and 800 localizations are $1.68^\circ$ and
$1.48^\circ$ ($2.33^\circ$ and $2.31^\circ$ against the prespecified
references). The prespecified verdict stays a failure. The first study
has the same floor at a smaller scale: its references are full
$1.2\,\mu$m stretches, and post hoc arm-level references give slope
$-0.505$.

\paragraph{Crossing tests.}
Four tests compare crossings with controls built from the same arms.
Counted over the same 400\,nm of each half, the opposite arms of a cross
have per-length intensity ratios (smaller over larger) with median 0.90,
0.84 and 0.83 in the cross-laboratory, U-373 and unused-acquisition
studies (10th percentiles 0.69, 0.68 and 0.65), so these crosses lie
near, not exactly at, balance. In
the first study the difference for the fit was between $0.0$ and $+2.2$
points at every size with every interval containing zero (for example
$+2.2$ $[-5.6,9.8]$ at 100 localizations). A separate test with a frozen protocol used
the earlier implementation on exactly balanced crossings, built from four
different filaments with equalized line weights and all arms cut at the
vertex; it gave $-1.1$
$[-9.3,6.2]$ at 100 localizations and $+0.6$ $[-7.3,7.9]$ at 400, with
equal median errors ($1.33^\circ$ and $1.38^\circ$ at 400); the Hessian
detector instead gained at T-junctions ($+9.4$ $[2.2,17.3]$ and $+6.7$
$[2.2,12.3]$), because a straight-through line helps a filter. The
cross-laboratory test passed at its margin (the T-junction interval ends
at $-10.0$), and the U-373 test failed its $\pm10$ margin with one
positive and one negative difference, $+4.1$ $[-2.1,12.4]$ and $-5.4$
$[-12.5,1.3]$.

\paragraph{Faint arms.}
\begin{table}[h]
\centering\small
\caption{Count accuracy at 400 and 800 localizations, part B of the
unused-acquisition study (60 faint junctions, 60 unthinned twins).}
\label{tab:faint}
\begin{tabular}{lccc}
\toprule
Method & Faint, count & Faint, exact recovery & Twins, count\\
\midrule
Likelihood fit, agreement rule & 0.19 / 0.26 & 0.17 / 0.26 & 1.00 / 1.00\\
\quad with continuation check & 0.55 / 0.79 & 0.51 / 0.74 & 1.00 / 1.00\\
Likelihood fit, BIC count only & 0.74 / 0.89 & 0.64 / 0.79 & 0.88 / 0.80\\
Earlier implementation & 0.68 / 0.89 & 0.58 / 0.84 & 0.95 / 0.97\\
Hessian ridge detector & 0.11 / 0.05 & 0.09 / 0.05 & 0.93 / 0.95\\
Line fit & 0.00 / 0.00 & 0.00 / 0.00 & 0.98 / 0.95\\
\bottomrule
\end{tabular}
\end{table}

The agreement rule removes faint arms that the BIC had right (21, 29 and
12 of 120 junctions at 100, 400 and 800 localizations): both detectors
miss an arm that holds 2 to 8\% of the window. The continuation check
restores most of them without a single new extra arm, while the BIC alone
over-counts the twins. Detectors cannot see these arms at all.

\subsection{Uncertainty}
Sandwich intervals for single arms under-covered in the first study
(0.83 to 0.74) and on X-ray photons (0.83), and the joint regions
calibrated on the first study missed the 0.90 floor in the
cross-laboratory study at 100 and 800 localizations. The rule tested on
the unused acquisitions keeps the counts whose BIC lies within a
calibrated gap $c_n$ of the reported count, and gives each arm the
interval $\widehat\theta\pm z_n\{\mathrm{SE}^2+(1^\circ)^2\}^{1/2}$, with a
field-clustered sandwich standard error and a floor of $1^\circ$ for
effects the sandwich does not see. The thresholds were fixed on 84
development fields ($c_n=4.56,0,0$ and $z_n=6.68,3.63,3.63$ at 100, 400
and 800 localizations). It covers at every size and keeps 1.1 to 1.4
counts instead of 2.3 to 3.8, but at 100 localizations its mean width is
$69.7^\circ$ (median $34.9^\circ$), so the prespecified width limit of
$20^\circ$ fails there.

\subsection{X-ray photons}
\emph{First study.} All three criteria failed. The fit and the two
strongest detectors succeeded and failed on largely the same sources;
when the fit missed, it usually returned the medium-energy pair without
the faint high-energy pair. The tuned Hessian detector returned the
pixel axes $0^\circ,90^\circ,180^\circ,270^\circ$ in every case: the
capped-loss tuning objective rewarded a constant four-arm answer, which
also produced its apparent win. On data drawn from the model itself the
original fit found the four arms within $1^\circ$ in 20\% of cases, and
the objective at the truth was 69 units better than the returned fit: the
generic start (width five times the line width, arm shares near 23\%)
let thin arms fade. A start matched to thin lines raised clean-data
recovery to 90\%, and the retest was written on that basis, with new
sources and exact recovery at $1^\circ$ as the tuning objective.

\emph{Retest.} Table~\ref{tab:xray} lists every method. The revised
starts double the original fit's recovery, the error of fits with all
arms within $10^\circ$ falls at the $n^{-1/2}$ rate, and the fit ties the
line fit and the histogram given the true vertex: all three are limited
by counting, the fit through over- and under-counts, the other two
through over-counts of 49 to 59\% at 400 and 800 photons. Among successful
fits at 800 photons the median errors are $0.023^\circ$ for the fit,
$0.024^\circ$ for the histogram given the true vertex and $0.030^\circ$
for the line fit. The one-Gaussian variant with the same starts reached
0.52 at 800 photons, above the BIC blur choice (0.40), which mostly chose
two Gaussians. Single-arm sandwich intervals covered 0.69, 0.82, 0.87 and
0.86 at the four sizes.

\begin{table}[h]
\centering\small
\caption{Chandra retest, share of 90 source draws at 100 / 200 / 400 / 800
photons.}
\label{tab:xray}
\begin{tabular}{lcccc}
\toprule
Method & Exact, $1^\circ$ & Count right & Over-count & Under-count\\
\midrule
Likelihood fit & .20/.26/.37/.40 & .39/.31/.50/.52 & .41/.33/.18/.20 & .20/.36/.32/.28\\
First-study starts & .09/.09/.18/.20 & .28/.16/.31/.30 & .21/.24/.20/.22 & .51/.60/.49/.48\\
One Gaussian only & .17/.31/.44/.52 & .36/.41/.50/.59 & .47/.33/.27/.24 & .18/.26/.23/.17\\
Two Gaussians only & .19/.23/.33/.41 & .42/.29/.46/.53 & .38/.36/.21/.17 & .20/.36/.33/.30\\
Line fit & .10/.20/.38/.41 & .13/.22/.38/.41 & .83/.70/.53/.49 & .03/.08/.09/.10\\
Histogram, true vertex & .11/.19/.31/.41 & .12/.20/.32/.42 & .84/.73/.59/.51 & .03/.07/.09/.07\\
Hessian ridge detector & .00/.00/.00/.00 & .09/.07/.10/.08 & .51/.39/.24/.14 & .40/.54/.66/.78\\
\bottomrule
\end{tabular}
\end{table}

\paragraph{Natural junctions against the reference.}
Across 1,391 sites in 73 fields, exact recovery against the second
reference is 54.9\% for the likelihood fit and 27.5\% for the Hessian
detector at 100 localizations, and 57.3\% against 52.8\% at 400; on the
783 sites where every method family agrees on the count, 77.7\% against
38.4\% at 100 and 84.8\% against 72.5\% at 400. The paired gains over the
detector are $+31.0$ points $[28.3,33.5]$ at 50 localizations, $+27.4$
$[25.4,29.5]$ at 100, $+13.1$ $[10.8,15.3]$ at 200 and $+4.4$ $[2.0,6.8]$
at 400. The histogram about the reference center leads on all sites,
because it shares the reference's count rule; on the consensus sites the
fit beats it at every size from 50 on. A study with a frozen protocol on 348
sites in 88 U-373 cells gave gains over every detector: Hessian $+19.9$
$[17.2,22.4]$, Steger $+27.6$ $[24.4,30.7]$, line fit $+19.6$
$[16.7,22.4]$, histogram $+13.1$ $[10.3,16.1]$, skeleton $+27.0$
$[23.9,30.0]$. Its rate criterion measured error against a reference
that itself errs by several degrees and failed ($-0.18$); scored post hoc
as precision against each method's own full-half estimate, as in the
earlier studies, the slope is $-0.504$ $[-0.532,-0.455]$.

\section{Diagnostics}\label{app:diagnostics}
\paragraph{Optimizer.}
We restarted the fits of the cross-laboratory study at 400 and 800
localizations from 30 additional random starts each (480 fits). At the
reported count, restarts improved 15 of 120 fits at 400 and 17 of 120 at
800 by more than $0.5$ log-likelihood units (largest gains $6.9$ and
$11.0$), and the directions of improved fits moved by at most $1.22^\circ$
and $0.34^\circ$ (medians $0.18^\circ$ and $0.10^\circ$). At the true
count, three faint-arm junctions at 800 localizations had stored
four-arm fits far from the best found (gains of 104, 169 and 2546 units,
directions moving $82^\circ$ to $135^\circ$): part of the undercounting at
large $n$ is an optimization failure of the larger-count fit. Restarts
give lower bounds on the gap, not upper bounds.

We then tested a stronger start protocol that uses neither the truth nor
the stored optimum, on every cross-laboratory window at 400 and 800
localizations: the frozen starts, plus nested starts that add one arm to
the best fit with one arm fewer at each peak of a low-threshold angular
histogram, plus ten seeded random starts per count. It improves 22.8\% of
the 960 per-count fits by more than $0.5$ units and none gets worse; it
recovers all three stuck optima (gains of 104, 169 and 2550 units); and
the median window takes 3.3\,s instead of 2.2\,s at 400 localizations
and 4.6\,s instead of 2.9\,s at 800. The reported count changes in 1 of
240 windows, and exact recovery is unchanged (0.942 at 400 and 0.925 at
800, pooled). The BIC count alone gets worse, from 0.81 to 0.75 at 400
and from 0.75 to 0.68 at 800, because better-optimized larger counts win
more often under a misspecified blur and background; the agreement rule
absorbs this. Of the three stuck cases, the stronger starts fix the BIC
count in one, but there both detectors report two arms and the agreement
rule keeps two; in the other two the BIC prefers two arms even at the
better four-arm optimum. These faint arms are lost to the count rules, as
in Table~\ref{tab:faint}, not to the search.

A convex relaxation over
every ray component allowed by the fitted vertex box and width range
gives upper bounds on the attainable log-likelihood that exceed the
achieved values by 436 to 450 units in four 100-localization test cases;
these bounds are far from tight and do not certify global optimality.

\paragraph{Construction.}
Assembled arms are cut at the vertex, whereas a natural arm end is
blurred along its length. In paired synthetic data (skewed Y, $60^\circ$
X and T shapes; widths 25 and 40\,nm; 10\% background; 240 pairs per
size), cutting changed the fit's exact recovery by $-2.9$ points
$[-9.2,3.3]$ at 100 localizations and $-0.4$ $[-1.7,0.8]$ at 400, and the
Hessian detector's by $-8.3$ $[-15.4,-1.7]$ and $-5.0$ $[-9.6,-0.8]$. The
construction thus costs the detector 5 to 8 points and leaves the fit
unchanged; the gains over the detector in Table~\ref{tab:endpoints}
include this effect.

\paragraph{One or two Gaussians.}
At the true count and the same fitting budget, on 480 development
junctions, the two-Gaussian fit and the one-Gaussian fit differ in mean
arm error by $-0.04^\circ$ $[-0.25,0.14]$ at 100 localizations and
$+0.08^\circ$ $[-0.19,0.48]$ at 800, and the BIC prefers two Gaussians in
5 to 20\% of cases. In synthetic data with tails of 15 to 30\% weight at
two to four times the core width, the two-Gaussian fit never improved
exact recovery over one Gaussian (differences $-12$ to $+4$ points, every
interval containing zero), and improved the median error slightly only
with heavy wide tails at 400 localizations. On X-ray photons two
Gaussians helped in the first study (0.32 against 0.19 at 800 photons)
and not in the retest. The mixture is the theorem's class and a guard
against heavy tails; on microtubule localizations one Gaussian is as good.

\end{document}